\documentclass[11pt]{article}
\usepackage[utf8]{inputenc}
\usepackage{amsmath, amssymb, amsthm}
\usepackage{geometry}
\usepackage{xcolor}
\usepackage{booktabs}
\usepackage{hyperref}
\usepackage{natbib}
\usepackage{graphicx}
\usepackage{tikz}
\usetikzlibrary{positioning, arrows.meta}
\usepackage{float}

\newtheorem{theorem}{Theorem}
\theoremstyle{remark}
\newtheorem{remark}{Remark}

\hypersetup{
    colorlinks=true,
    linkcolor=blue,
    filecolor=magenta,      
    urlcolor=cyan,
    citecolor=purple,
}

\usepackage{fancyhdr}
\usepackage{enumitem}
\newcounter{boxlblcounter}  
\newenvironment{boxlabel}
  {\begin{list}
    {\arabic{boxlblcounter}}
    {\usecounter{boxlblcounter}
     \setlength{\labelwidth}{3em}
     \setlength{\labelsep}{0em}
     \setlength{\itemsep}{2pt}
     \setlength{\leftmargin}{1.5cm}
     \setlength{\rightmargin}{2cm}
     \setlength{\itemindent}{0em} 
     
    }
  }
{\end{list}}

\title{\textbf{Gaussian Mixture Attention: Linear-Time Sequence Mixing via Probabilistic Latent Routing}}
\author{
    Yongchao Huang, Hassan Raza
}
\author{Yongchao Huang\footnote{yongchao.huang@abdn.ac.uk} \and Hassan Raza\footnote{h.raza.24@abdn.ac.uk}}
\date{16/05/2026}

\begin{document}

\maketitle

\begin{abstract}
The dense token-to-token interaction pattern of standard dot-product attention remains a central bottleneck in scaling Transformer architectures to long contexts. We introduce \textbf{Gaussian Mixture Attention (GMA)}, a probabilistic attention-style sequence mixer that replaces explicit pairwise query--key comparison with routing through $K$ learned Gaussian mixture components. Queries and keys are mapped to posterior \textit{responsibility} vectors over a shared latent routing space; their overlap defines an implicit responsibility-space affinity, while values are written into and read from a $K$-slot latent memory. By exploiting the associativity of matrix multiplication, GMA avoids materializing the induced $N\times N$ affinity matrix and instead uses two responsibility matrices whose dominant activation storage scales as $\mathcal{O}(NK)$ rather than $\mathcal{O}(N^2)$ for fixed $K$. We formulate bidirectional and causal variants of GMA, provide an end-to-end differentiable parameterization of the Gaussian mixture components, and analyze its responsibility-modulated gradient structure, constrained non-negative low-rank affinity interpretation, and local routing stability. Empirically, GMA exhibits the intended fixed-$K$ linear memory scaling and is competitive with attention-style baselines on long-context classification, while causal GMA improves over tested linear/random-feature attention variants on WikiText-103 but remains behind optimized causal SDPA and Mamba in the current implementation. Analysis of learned responsibilities further shows broad component usage and moderate alignment with surface-form token categories, supporting GMA as a probabilistic, interpretable, fixed-$K$ linear-time attention-style alternative rather than a universal replacement for optimized softmax attention or state-space models.
\end{abstract}

\section{Introduction}
\label{sec:intro}

The \textit{Transformer} architecture has achieved strong performance across language, vision, and multimodal learning, driven largely by the representational power and parallelism of dot-product Multi-Head Attention (MHA) \citep{vaswani2017attention,dosovitskiy2020image,radford2021learning}. However, standard self-attention computes pairwise interactions between all $N$ tokens in a sequence. For one attention head, let $Q\in\mathbb{R}^{N\times d_k}$ denote the query matrix, $K_{\mathrm{att}}\in\mathbb{R}^{N\times d_k}$ denote the key matrix, and $V\in\mathbb{R}^{N\times d_v}$ denote the value matrix. Here $d_k$ is the query/key channel dimension used to compute dot-product scores, while $d_v$ is the value dimension of the vectors being aggregated. Scaled dot-product attention computes
\begin{equation}
\label{eq:dot_product_attention}
    O =
    \operatorname{softmax}
    \left(
    \frac{QK_{\mathrm{att}}^\top}{\sqrt{d_k}}
    \right)V ,
    \qquad
    O\in\mathbb{R}^{N\times d_v}.
\end{equation}
The intermediate score matrix $QK_{\mathrm{att}}^\top\in\mathbb{R}^{N\times N}$ contains token-to-token scores between all query and key positions. Although the final output has dimension $O\in\mathbb{R}^{N\times d_v}$, computing dense attention requires forming or implicitly representing interactions between all pairs of positions. In the standard explicit formulation, this leads to $\mathcal{O}(N^2)$ attention-score storage and $\mathcal{O}(N^2d_k+N^2d_v)$ arithmetic for the score and value multiplications. When $d_k$ and $d_v$ are treated as fixed, this gives the familiar \textit{quadratic scaling in sequence length}. This quadratic dependence makes long-context modelling expensive and motivates efficient alternatives for settings such as long-document processing, byte-level classification, high-resolution vision, genomics, and autoregressive language modelling \citep{beltagy2020longformer,zaheer2020bigbird,tay2020LRA,wang2020linformer, choromanski2021rethinking,katharopoulos2020transformers,avsec2021effective}.

A large body of work has attempted to reduce the quadratic attention bottleneck through sparse attention patterns, low-rank projections, kernel approximations, optimized exact-attention kernels, and recurrent or state-space sequence models \citep{tay2022efficient,beltagy2020longformer,zaheer2020bigbird, wang2020linformer,katharopoulos2020transformers,choromanski2021rethinking, dao2022flashattention,gu2021efficiently,gu2023mamba}. Sparse attention methods, such as \textit{Longformer} and \textit{BigBird}, reduce computation by restricting each token to attend to a subset of local, global, or structured positions \citep{beltagy2020longformer,zaheer2020bigbird}. Low-rank methods such as \textit{Linformer} compress the attention matrix through learned projections along the sequence dimension \citep{wang2020linformer}. Kernel-based methods replace the softmax attention kernel with feature-map constructions, including deterministic positive feature maps in \textit{Linear Transformer} and random-feature approximations in \textit{Performer} \citep{katharopoulos2020transformers,choromanski2021rethinking}. IO-aware implementations such as \textit{FlashAttention} retain exact softmax attention while reducing memory traffic through hardware-aware tiling \citep{dao2022flashattention}. More recently, structured state-space and selective state-space models, including \textit{S4} and \textit{Mamba}, have pursued linear or near-linear sequence modelling by replacing explicit attention with recurrent or state-space dynamics \citep{gu2021efficiently,gu2023mamba}. These approaches expose different trade-offs: they can be highly efficient, but their internal routing structure is often less directly probabilistic or less immediately interpretable as a token-to-token distribution than the original softmax attention weights.

In this work, we propose \textbf{Gaussian Mixture Attention (GMA)}, a probabilistic alternative to dot-product attention that reconceptualizes sequence mixing as \textit{latent responsibility-based routing}. Instead of explicitly computing all pairwise token-to-token similarities, GMA introduces $K$ learned Gaussian mixture components in a projected \textit{routing representation space}. Query and key representations are mapped to posterior responsibility vectors over these components. Key responsibilities write values into a latent memory $\tilde V\in\mathbb{R}^{K\times d_v}$ together with a component-wise normalizer $Z\in\mathbb{R}^K$, while query responsibilities read from this normalized latent memory to produce token-level outputs. Thus, GMA replaces the explicit $N\times N$ token-to-token attention matrix with two non- negative \textit{responsibility matrices} $\Gamma^Q,\Gamma^K\in\mathbb{R}^{N\times K}$. Although the induced affinity $\Gamma^Q(\Gamma^K)^\top$ is still an $N\times N$ matrix algebraically, GMA does not materialize it. Instead, it exploits the associativity of matrix multiplication by computing the key--value latent memory $(\Gamma^K)^\top V_X$ first and then multiplying by $\Gamma^Q$. For fixed $K$, this gives linear-in-$N$ dominant activation scaling while retaining a normalized attention-style routing interpretation.

A key motivation for GMA is that the responsibility matrices are not merely computational intermediates: they are analyzable probabilistic objects. The marginal usage of mixture components can diagnose whether the latent routing space is broadly used or collapsed to a small subset of components, while hard assignments derived from $z_i=\arg\max_k\gamma_{i,k}$ can be compared with token categories or other annotations. This provides an interpretability handle that is less direct in random-feature attention approximations or implicit recurrent/state-space hidden dynamics. Our later analysis shows that learned GMA responsibilities use most available components and exhibit moderate alignment with surface-form token categories, although the components should not be interpreted as clean semantic classes (Section~\ref{sec:gma_interpretability}).

Our contributions and findings are as follows:
\begin{enumerate}
    \item We introduce \textit{Gaussian Mixture Attention} (GMA), a normalized responsibility-based sequence mixer that replaces explicit token-to-token attention with routing through learned Gaussian mixture components.

    \item We derive both bidirectional and causal GMA. The causal variant uses prefix latent memories and prefix normalizers so that autoregressive predictions at position $i$ depend only on positions $j\leq i$, while preserving fixed-$K$ linear-in-sequence-length scaling.

    \item We analyze GMA's optimization and representational structure, including responsibility-modulated gradients, a constrained non-negative low-rank affinity interpretation, and local Lipschitz continuity under bounded inputs and variance lower bounds.

    \item We evaluate GMA in 4 empirical settings: controlled systems profiling (Table~\ref{tab:efficiency}), Long Range Arena (LRA) long-context classification \citep{tay2020LRA} (Table~\ref{tab:lra_accuracy}), WikiText-103 autoregressive language modelling \citep{merity2016pointer} (Table~\ref{tab:wikitext_results}), and latent-responsibility interpretability analysis (Table~\ref{tab:gma_interpretability}; Figures~\ref{fig:gma_gammaq_heatmap}--\ref{fig:gma_category_heatmap}). The results show that GMA exhibits the intended linear memory scaling and achieves competitive performance among attention-style baselines evaluated in our pipeline. On LRA, it gives the strongest average performance among those attention-style baselines (Table~\ref{tab:lra_accuracy}). On WikiText-103, causal GMA improves over Linear Transformer \citep{katharopoulos2020transformers} and Performer \citep{choromanski2021rethinking}, although optimized causal SDPA \citep{vaswani2017attention} and Mamba \citep{gu2023mamba} remain stronger (Table~\ref{tab:wikitext_results}). Finally, learned GMA responsibilities use most available components and show moderate alignment with surface-form token categories (Table~\ref{tab:gma_interpretability}; Figures~\ref{fig:gma_gammaq_heatmap}--\ref{fig:gma_category_heatmap}).

    \item We discuss future extensions of the GMA framework, including optimized and hybrid GMA implementations, cross-attention and multimodal routing, Bayesian and Dirichlet-process variants for adaptive component weighting, and probabilistic Mixture-of-Experts routing.
\end{enumerate}

\section{Related Work}
\label{sec:related_work}

\paragraph{Attention as learned compatibility.}
Attention mechanisms can be viewed broadly as methods for computing \textit{compatibility} between \textit{query representations} and \textit{context representations}, and then using the resulting weights to aggregate values. Early neural encoder--decoder models used learned alignment mechanisms to focus decoding on relevant source positions \citep{bahdanau2015neural,luong2015effective}. The \textit{Transformer} made \textit{scaled dot-product attention} the dominant formulation by computing token-to-token scores through $QK_{\mathrm{att}}^\top$ and applying row-wise softmax normalization \citep{vaswani2017attention}, as in Eq.~\eqref{eq:dot_product_attention}. This design is highly expressive and parallelizable: every token can form query--key scores against all other tokens using batched matrix multiplication, and the value aggregation can likewise be computed by dense linear algebra on modern accelerators. This parallel structure is a major reason for the success of \textit{Transformer} architectures, but it also requires forming or implicitly representing an $N\times N$ interaction pattern, giving the familiar quadratic scaling in sequence length.

From a broader design perspective, the dot product is only one possible compatibility function. Attention scores may also be based on learned additive scores, kernels, distances, divergences, sparsity patterns, or latent routing structures. \textit{GMA} follows this broader view by replacing direct pairwise query--key comparison with compatibility in a probabilistic responsibility space. A more explicit taxonomy of similarity, distance, divergence, and latent-space attention designs is given in Appendix~\ref{app:attention_similarity_distance}.

\paragraph{Sparse and structured attention.}
One major line of efficient-Transformer research reduces the number of query--key pairs that are compared. \textit{Sparse Transformer} uses structured sparse patterns to generate long sequences more efficiently than dense attention \citep{child2019generating}. \textit{Reformer} replaces dense dot-product attention with locality-sensitive hashing attention and also uses reversible residual layers to reduce activation storage \citep{kitaev2020reformer}. \textit{Longformer} combines a local sliding-window pattern with task-motivated global attention, giving linear scaling in sequence length and strong performance on long-document tasks \citep{beltagy2020longformer}. \textit{BigBird} combines local, random, and global attention patterns, obtaining linear sparse attention while preserving important theoretical properties of full attention, including universal approximation and Turing-completeness results under its sparse pattern \citep{zaheer2020bigbird}. These methods retain token-to-token attention but restrict the attention graph. Their efficiency therefore depends on a designed or sampled sparsity pattern. \textit{GMA} differs by keeping dense global information available through latent mixture components rather than choosing a sparse set of token pairs.

\paragraph{Low-rank and landmark approximations.}
A second line approximates the attention matrix by a lower-dimensional representation. \textit{Linformer} argues that self-attention can be approximated by a low-rank matrix and projects the sequence dimension to reduce the complexity of
attention to linear time and space \citep{wang2020linformer}. \textit{Nystr\"omformer} uses the Nystr\"om method to approximate self-attention through a set of landmark points, also targeting linear complexity with favorable long-sequence performance \citep{xiong2021nystromformer}. These approaches are related to \textit{GMA} in that they replace the full $N\times N$ interaction structure by a smaller intermediate representation. However, their intermediate factors are projection- or landmark-based approximations of attention. \textit{GMA} instead uses posterior responsibilities under learned Gaussian mixture components, producing a constrained non-negative low-rank affinity $\Gamma^Q(\Gamma^K)^\top$ with a probabilistic routing interpretation.

\paragraph{Kernelized and associative linear attention.}
\textit{Kernelized attention} methods exploit the associativity of matrix multiplication. \textit{Linear Transformer} rewrites attention using feature maps so that the key--value summary can be computed before multiplying by queries, reducing the cost from quadratic to linear in sequence length \citep{katharopoulos2020transformers}. \textit{Performer} approximates softmax attention using positive orthogonal random features in the FAVOR+ mechanism, giving linear space and time with theoretical guarantees for the random-feature approximation \citep{choromanski2021rethinking}. These methods are especially close to GMA at the computational level: all use an ordering of the form ``summarize keys and values first, then read with queries''. The difference lies in the feature map. Kernelized attention uses deterministic or random feature maps designed to approximate a kernel such as softmax. \textit{GMA} uses learned Gaussian-mixture posterior responsibilities as the feature map, so the intermediate representation is a probability vector over latent components. This gives \textit{GMA} an additional diagnostic object, i.e. the responsibility matrix, but also introduces extra constant factors from Gaussian log-density and responsibility computations.

\paragraph{IO-aware exact attention.}
Another important direction does not change the mathematical attention definition, but optimizes its implementation. \textit{FlashAttention} computes exact softmax attention using IO-aware tiling, reducing reads and writes between GPU high-bandwidth memory and on-chip SRAM \citep{dao2022flashattention}. Later \textit{FlashAttention} variants further improve hardware utilization through more aggressive scheduling and low-precision support \citep{shah2024flashattention3}. These methods are essential baselines because they show that quadratic attention can be much faster and more memory efficient in practice without changing the underlying attention weights. \textit{GMA} addresses a different question: it changes the attention-style operator itself to obtain fixed-$K$ linear activation scaling and interpretable latent routing. Consequently, the relevant comparison is not only asymptotic complexity but also wall-clock efficiency, kernel maturity, memory traffic, and model quality.

\paragraph{Latent bottleneck attention.}
Latent-array architectures reduce the cost of processing large inputs by routing information through a smaller set of latent variables. The \textit{Perceiver} uses asymmetric cross-attention to distill high-dimensional inputs into a latent bottleneck and then processes the latent array with Transformer blocks \citep{jaegle2021perceiver}. \textit{Perceiver} IO extends this idea to flexible structured outputs by querying the latent representation with output-specific queries \citep{jaegle2021perceiverio}. \textit{GMA} shares the high-level idea that a smaller latent substrate can mediate interactions among many input positions. However, the latent objects are different. \textit{Perceiver} uses learned latent vectors that attend to inputs through cross-attention. GMA uses learned Gaussian mixture components in routing space and maps each query/key token to a posterior responsibility vector over these components. The resulting latent memory is not a separate learned sequence, but a key-weighted aggregation of values through mixture responsibilities.

\paragraph{State-space and recurrent alternatives to attention.}
Structured state-space models provide a separate route to long-sequence modelling by replacing explicit attention with recurrent or convolutional sequence dynamics. \textit{S4} introduced structured state-space layers for efficient long-range sequence modelling \citep{gu2021efficiently}. \textit{Mamba} builds on this line by making state-space parameters input-dependent, allowing selective propagation or forgetting of information, and by using hardware-aware parallel algorithms for efficient training and inference \citep{gu2023mamba}. These models often provide strong throughput and strong performance, especially in long-sequence or discrete-token settings. They are not attention mechanisms in the usual token-to-token sense, and their internal routing structure is less directly exposed as an attention matrix. \textit{GMA} occupies a different position: it remains an attention-style sequence mixer with an induced token-to-token affinity, while introducing a probabilistic latent routing space that can be inspected through responsibility diagnostics.

\paragraph{Mixture models and probabilistic attention.}
Mixture models have a long history as probabilistic tools for soft assignment, density modelling, and responsibility-based learning \citep{dempster1977maximum,bishop1994mixture,mclachlan2000finite,bishop2006pattern}. Several recent works have explored connections between mixture models and attention. \textit{Transformer-MGK} (Transformer with a Mixture of Gaussian Keys) replaces redundant attention heads with a mixture of Gaussian keys at each head, allowing each head to focus on different parts of the input while reducing parameters and computation \citep{nguyen2021improving}. However, the mixture is local to the key representation inside an attention head, and the model still relies on query--key attention computations unless combined with separate linear-attention approximations. In machine translation, \textit{GMM-based cross-attention} has been used to model concentrated attention around central source positions, improving alignment quality and long-sentence translation behaviour \citep{zhang2021modeling}. These works show that Gaussian mixtures can be useful inside attention, but they do not use a shared global mixture responsibility space to replace the explicit $N\times N$ token-to-token attention matrix.

\textit{GMA} differs from these mixture-based approaches in three ways. First, it uses a shared global Gaussian mixture over routing vectors, rather than a local mixture attached to each position or attention head. Second, both queries and keys are mapped into responsibility vectors over the same mixture components, inducing the non-negative affinity $\widetilde A^{\mathrm{GMA}}=\Gamma^Q(\Gamma^K)^\top$. Third, the efficient implementation does not materialize this affinity; it uses the associative ordering $(\Gamma^K)^\top V_X$ followed by multiplication with $\Gamma^Q$. Thus, GMA is not merely a Gaussian modification of attention scores. It combines probabilistic responsibility-space similarity with latent-memory routing to
obtain fixed-$K$ linear activation scaling.

\paragraph{Mixture-of-Experts and probabilistic routing.}
\textit{GMA} is also related to the broader literature on conditional computation and Mixture-of-Experts (MoE). Classical MoE models combine multiple expert predictors through input-dependent gates \citep{jacobs1991adaptive}. Modern sparsely gated MoE models scale Transformer capacity by routing each token to a small number of feed-forward experts \citep{shazeer2017outrageously, lepikhin2020gshard,fedus2022switch}. Recent work also explores MoE-style routing inside the attention mechanism itself, for example by switching among attention heads or computing fewer attention matrices \citep{csordas2023switchhead}. These methods typically use learned softmax or noisy top-$k$ gates and are motivated by conditional computation. \textit{GMA}'s routing is different: the gating weights are posterior responsibilities under a learned Gaussian mixture model. This makes the routing probabilities density-based and interpretable as soft assignments to latent components. Although our present work uses GMA for sequence mixing rather than sparse expert selection, the same responsibility view suggests future probabilistic MoE extensions, as discussed in Section~\ref{sec:future_extensions}.

\paragraph{Positioning of GMA.}
Overall, prior work has reduced the quadratic attention bottleneck by sparsifying the attention graph, approximating the attention matrix, kernelizing the softmax, optimizing exact attention kernels, introducing latent bottlenecks, or replacing attention with state-space dynamics. \textit{GMA} contributes a complementary design point. It treats attention as routing through a learned probabilistic latent space: queries and keys are first converted into Gaussian-mixture responsibilities, values are written into a latent memory by key responsibilities, and queries read from that memory through normalized responsibility overlap. This gives \textit{GMA} a fixed-$K$ linear-time attention-style operator with an analyzable latent responsibility structure. Its purpose is not to replace all optimized attention or state-space models, but to expose a probabilistic and interpretable route to efficient sequence mixing.

\section{Methodology: Gaussian Mixture Attention}
\label{sec:GMA_method}

We present GMA primarily in the \textit{self-attention} setting, where queries, keys, and values are all projected from the same sequence $X$. For clarity, the main equations are written for a single attention head; a multi-head GMA layer applies the same computation independently across heads and then combines the resulting head outputs in the usual way. This formulation covers the bidirectional sequence-mixing experiments and the causal language-modelling experiments considered later in this paper. \textit{Cross-attention} follows analogously by forming queries from a query-side sequence $X$ and keys/values from a separate context sequence $Y$; the same responsibility-space routing mechanism can then be applied with $\Gamma_X^Q$ and $\Gamma_Y^K$. A detailed distinction between standard self-attention, standard cross-attention, and their GMA counterparts is given in Appendix~\ref{app:self_cross_gma}. We leave a systematic empirical study of GMA cross-attention to future work.

\subsection{Method Overview}
Standard self-attention relies on direct pairwise token-to-token comparisons, which gives the familiar quadratic interaction pattern in sequence length. GMA changes this computation by introducing a latent memory architecture based on $K$ learned Gaussian routing components. Instead of comparing each query with each key directly, GMA maps projected queries and keys to posterior responsibility vectors over a shared Gaussian mixture.

Intuitively, the projected value matrix $V_X$ stores token-level information. The key responsibility matrix $\Gamma^K$ acts as a soft writing mechanism, routing value vectors into $K$ latent memory slots. This produces a latent memory $\tilde V$ together with a component-wise normalizer $Z$. Query responsibilities $\Gamma^Q$ then act as a soft reading mechanism, determining how each output token reads from the normalized latent memory. Thus, queries do not explicitly compare against all keys in the efficient implementation; instead, query--key compatibility is mediated through shared responsibility vectors over the latent Gaussian components. A consolidated notation table for the methodology and analysis is provided in Appendix~\ref{app:gma_notation}.

\subsection{Probabilistic Latent Routing}
\label{sec:prob_latent_routing}

Instead of computing direct token-to-token similarities, we introduce $K$ learned latent centers, representing Gaussian routing components. We model these centers as the components of a Gaussian Mixture Model (GMM) \cite{huang2025GMA_sampling,huang2026GJE}. The $k$-th component is parameterized by a mean $\mu_k \in \mathbb{R}^{d_r}$ and a covariance matrix $\Sigma_k \in \mathbb{S}_{++}^{d_r}$, where $d_r$ is the routing dimension. For computational tractability, we use diagonal covariances, $\Sigma_k=\operatorname{diag}(\sigma_{k,1}^2,\ldots,\sigma_{k,d_r}^2)$. The  mixture prior of component $k$ is $\pi_k$, with $\sum_{k=1}^K \pi_k=1$.

For a routing vector $x_i\in\mathbb{R}^{d_r}$, which may be either a query projection or a key projection, the density under component $k$ is
\begin{equation}
\label{eq:gma_gaussian_density}
    \mathcal{N}(x_i \mid \mu_k,\Sigma_k)
    =
    \frac{1}{(2\pi)^{d_r/2}|\Sigma_k|^{1/2}}
    \exp\left(
    -\frac{1}{2}
    (x_i-\mu_k)^\top\Sigma_k^{-1}(x_i-\mu_k)
    \right).
\end{equation}
Let $z_i\in\{1,\ldots,K\}$ denote the latent component assignment for $x_i$. The responsibility of component $k$ for $x_i$ is the posterior probability
$p(z_i=k\mid x_i)$:
\begin{equation}
\label{eq:gma_responsibility}
    \gamma_{i,k}
    \equiv
    p(z_i=k\mid x_i)
    =
    \frac{
    p(z_i=k)p(x_i\mid z_i=k)
    }{
    p(x_i)
    }
    =
    \frac{
    \pi_k \mathcal{N}(x_i\mid\mu_k,\Sigma_k)
    }{
    \sum_{\ell=1}^K
    \pi_\ell \mathcal{N}(x_i\mid\mu_\ell,\Sigma_\ell)
    }.
\end{equation}
This operation maps a sequence of routing vectors into a responsibility matrix $\Gamma\in\mathbb{R}^{N\times K}$. Since $\sum_{k=1}^K\gamma_{i,k}=1$ for every token $i$, each row of $\Gamma$ is a probability distribution over the latent Gaussian components.

\subsection{Linear-Time Normalized Attention via GMA}
\label{sec:linear_gma}

Let $X\in\mathbb{R}^{N\times d_{\mathrm{model}}}$ denote the input sequence representations to a GMA layer, where $N$ is the sequence length and $d_{\mathrm{model}}$ is the model hidden dimension, i.e., the dimensionality of each token representation before it is projected into queries, keys, and values. As in standard attention, we first form query, key, and value projections:
\begin{equation}
\label{eq:gma_qkv_projection}
    Q_X = XW_Q \in \mathbb{R}^{N\times d_r},
    \qquad
    K_X = XW_K \in \mathbb{R}^{N\times d_r},
    \qquad
    V_X = XW_V \in \mathbb{R}^{N\times d_v},
\end{equation}
where
$W_Q,W_K\in\mathbb{R}^{d_{\mathrm{model}}\times d_r}$ and $W_V\in\mathbb{R}^{d_{\mathrm{model}}\times d_v}$. Here $d_r$ is the routing dimension used to compute Gaussian responsibilities, and $d_v$ is the value dimension. We reserve scalar $K$ for the number of Gaussian mixture components.

Given $Q_X$ and $K_X$, GMA computes two responsibility matrices:
\[
    \Gamma^Q \in \mathbb{R}^{N\times K},
    \qquad
    \Gamma^K \in \mathbb{R}^{N\times K}.
\]
The entry $\gamma^Q_{i,k}$ is the posterior responsibility of component $k$ for the query vector $q_i=(Q_X)_{i,:} \in \mathbb{R}^{d_r}$, while $\gamma^K_{j,k}$ is the posterior responsibility of component $k$ for the key vector $k_j=(K_X)_{j,:} \in \mathbb{R}^{d_r}$. These responsibilities are calculated using the Gaussian-mixture posterior in
Eq.~\eqref{eq:gma_responsibility}, applied separately to the query and key projections. The superscript $K$ in $\Gamma^K$ denotes key responsibilities, not the number of mixture components.

To construct an end-to-end $\mathcal{O}(NK)$ attention mechanism, GMA avoids forming the full $N\times N$ token-to-token attention matrix. Instead, it uses the key responsibilities to write the projected values $V_X$ into $K$ latent memory slots, and then uses the query responsibilities to read from those slots. To make the induced attention weights stochastic over the sequence dimension, we normalize the latent routing operation. The computation proceeds in two associative steps.

\begin{enumerate}
    \item \textbf{Latent aggregation.} The key responsibilities aggregate the value vectors into $K$ latent memory slots and compute the total key mass assigned to each component:
    \begin{equation}
    \label{eq:gma_latent_aggregation}
        \tilde V
        =
        (\Gamma^K)^\top V_X
        \in \mathbb{R}^{K\times d_v},
        \qquad
        Z
        =
        (\Gamma^K)^\top \mathbf{1}_N
        \in \mathbb{R}^{K},
    \end{equation}
    where $\mathbf{1}_N$ denotes the all-ones vector. Component-wise\footnote{From a linear algebra perspective, the $k$-th row of $\tilde V$, i.e. $\tilde V_k$, is a linear combination of all rows of $V_{X}$, with weights being the $k$-th column of $\Gamma^K$.},
    \begin{equation}
    \label{eq:gma_latent_aggregation_componentwise}
        \tilde V_k
        =
        \sum_{j=1}^N \gamma^K_{j,k} V_{X,j}
        \in \mathbb{R}^{d_v},
        \qquad
        Z_k
        =
        \sum_{j=1}^N \gamma^K_{j,k}.
    \end{equation}
    Thus, $\tilde V_k$ is the responsibility-weighted sum of all value vectors assigned to component $k$, and $Z_k$ is the total key-responsibility mass assigned to that component.

    \item \textbf{Normalized broadcasting.} The query responsibilities read from
    the latent memory and normalize by the corresponding routed key mass:
    \begin{equation}
    \label{eq:gma_normalized_broadcasting}
        O
        =
        \frac{\Gamma^Q \tilde V}{\Gamma^Q Z+\epsilon}
        =
        \boldsymbol{
        \frac{\Gamma^Q (\Gamma^K)^\top V_X}{\Gamma^Q Z+\epsilon}
        }
        =
        A^{\mathrm{GMA}} V_X
        \in \mathbb{R}^{N\times d_v},
    \end{equation}
    where we have defined the implicit normalized affinity $A^{\mathrm{GMA}}=\frac{\Gamma^Q (\Gamma^K)^\top}{\Gamma^Q Z+\epsilon}$, which is useful for interpretation and analysis, but not materialized in the efficient implementation. The denominator is broadcast across the value dimension, the division is applied \textit{row-wise}, and $\epsilon>0$ is used for numerical stability\footnote{Equivalently, one may write the row-wise normalization in
    matrix form. Define
    \[
        d
        =
        \Gamma^Q Z+\epsilon\mathbf{1}_N
        \in\mathbb{R}^{N}.
    \]
    Then
    \[
        O
        =
        \operatorname{diag}(d)^{-1}\Gamma^Q\tilde V
        =
        \operatorname{diag}(d)^{-1}
        \Gamma^Q(\Gamma^K)^\top V_X .
    \]
    The corresponding implicit normalized affinity matrix is
    \[
        A^{\mathrm{GMA}}
        =
        \operatorname{diag}(d)^{-1}
        \Gamma^Q(\Gamma^K)^\top .
    \]
    This matrix is useful for interpretation and analysis, but it is not
    materialized in the efficient implementation.}.
\end{enumerate}

Specifically, for token $i$, the output can be written as
\begin{equation}
\label{eq:gma_output_componentwise}
    O_i
    =
    \frac{
    \sum_{k=1}^K \gamma^Q_{i,k}\tilde V_k
    }{
    \sum_{k=1}^K \gamma^Q_{i,k}Z_k
    +\epsilon}
    =
    \frac{
    \sum_{k=1}^K \gamma^Q_{i,k}
    \sum_{j=1}^N \gamma^K_{j,k} V_{X,j}
    }{
    \sum_{k=1}^K \gamma^Q_{i,k}
    \sum_{j=1}^N \gamma^K_{j,k}
    +\epsilon
    }.
\end{equation}
Thus, GMA induces the normalized attention weights
\begin{equation}
\label{eq:gma_induced_attention}
    A^{\mathrm{GMA}}_{ij}
    =
    \frac{
    \sum_{k=1}^K \gamma^Q_{i,k}\gamma^K_{j,k}
    }{
    \sum_{\ell=1}^N
    \sum_{k=1}^K
    \gamma^Q_{i,k}\gamma^K_{\ell,k}
    +\epsilon
    },
\end{equation}
so that 
\begin{equation} \label{eq:gma_attention}
    O_i=\sum_{j=1}^N A^{\mathrm{GMA}}_{ij}V_{X,j}
\end{equation}
up to the numerical stabilizer. Thus, GMA replaces the explicit $N\times N$ attention matrix by two non-negative responsibility matrices $\Gamma^Q,\Gamma^K\in\mathbb{R}^{N\times K}$, while still admitting the implicit normalized affinity in Eq.~\eqref{eq:gma_induced_attention}.

\begin{remark}[\textit{Implicit affinity versus explicit attention}]
Equation~\eqref{eq:gma_normalized_broadcasting} can be read in two equivalent ways. Algebraically, substituting $\tilde V=(\Gamma^K)^\top V_X$ gives the implicit affinity $\Gamma^Q(\Gamma^K)^\top$, which is an $N\times N$ token-to-token matrix, analogous in shape to the dot-product score matrix $QK^\top$ in Eq.~\eqref{eq:dot_product_attention}. However, GMA \textit{does not} need to materialize this $N\times N$ matrix. Its efficient implementation uses the associative ordering
\[
    \tilde V
    =
    (\Gamma^K)^\top V_X
    \in \mathbb{R}^{K\times d_v} \quad \text{(Eq.\ref{eq:gma_latent_aggregation})}
    \quad\text{followed by}\quad
    \Gamma^Q \tilde V,
\]
together with the normalizer $\Gamma^Q Z+\epsilon$. Thus, GMA first \textit{compresses} the values into $K$ latent memory slots and then reads from those slots, rather than first forming all pairwise query--key interactions. This is the source of the fixed-$K$ linear scaling: the computation is organized around $N\times K$ responsibility matrices and a $K\times d_v$ latent memory, with $K$ chosen independently of the sequence length and typically much smaller than $N$.
\end{remark}

The whole GMA workflow is visually represented in Fig.\ref{fig:gma_flow}. This ordering reflects a different routing philosophy from standard dot-product attention in Eq.~\eqref{eq:dot_product_attention}. Standard attention first evaluates query--key similarity by forming the token-to-token interaction matrix $QK^\top$, normalizes this matrix into attention weights, and then uses those weights to collect information from the value matrix $V$. Informally, each query first ``knocks on'' all keys to decide where relevant content is located, and then retrieves the corresponding values.  
GMA reverses this order. It first uses the key projection $K_X$ to compute key responsibilities $\Gamma^K$, which write the value projection $V_X$ into a compact latent memory $\tilde V=(\Gamma^K)^\top V_X$. The query responsibilities $\Gamma^Q$ then read directly from this $K$-slot memory through normalized broadcasting. Thus, standard attention follows a ``query--key matching then value collection'' principle, whereas GMA follows a ``key-induced memory construction then query-based memory reading'' principle.

\begin{figure}[htbp]
\centering
\begin{tikzpicture}[
    x=1cm, y=1cm,
    every node/.style={font=\fontsize{7}{8}\selectfont},
    flowbox/.style={
        draw, rounded corners,
        align=center,
        fill=gray!10,
        text width=3.9cm,
        minimum height=1.35cm,
        inner sep=3pt
    },
    parambox/.style={
        draw, rounded corners,
        align=center,
        fill=green!12,
        text width=3.9cm,
        minimum height=1.35cm,
        inner sep=3pt
    },
    memorybox/.style={
        draw, rounded corners,
        align=center,
        fill=blue!12,
        text width=3.9cm,
        minimum height=1.35cm,
        inner sep=3pt,
        thick
    },
    outbox/.style={
        draw, rounded corners,
        align=center,
        fill=violet!12,
        text width=3.9cm,
        minimum height=1.35cm,
        inner sep=3pt
    },
    arrow/.style={->, >=Latex, thick}
]

\def\xL{-5.7}
\def\xC{0}
\def\xR{5.7}

\def\yTop{0}
\def\yProj{-2.1}
\def\yResp{-4.7}
\def\yMem{-7.2}
\def\yNorm{-8.6}
\def\yOut{-7.9}


\node[flowbox] (X) at (\xC,\yTop)
{\textbf{Input Sequence} $(X)$\\
$\mathbb{R}^{N\times d_{\mathrm{model}}}$};

\node[flowbox] (QX) at (\xL,\yProj)
{\textbf{Query Projection} $(Q_X)$\\
$Q_X = XW_Q$\\
$\mathbb{R}^{N\times d_r}$};

\node[flowbox] (GammaQ) at (\xL,\yResp)
{\textbf{Read Responsibilities} $(\Gamma^Q)$\\
via Eq.~\eqref{eq:gma_responsibility}\\
$\mathbb{R}^{N\times K}$};

\node[outbox] (O) at (\xL,\yOut)
{\textbf{Output} $(O)$\\
$\displaystyle O=\frac{\Gamma^Q \tilde V}{\Gamma^Q Z+\epsilon}$\\
$\mathbb{R}^{N\times d_v}$};

\node[flowbox] (KX) at (\xC,\yProj)
{\textbf{Key Projection} $(K_X)$\\
$K_X = XW_K$\\
$\mathbb{R}^{N\times d_r}$};

\node[parambox] (GMM) at (\xC,\yResp)
{\textbf{Shared GMM Parameters}\\
$\{\pi_k,\mu_k,\Sigma_k\}_{k=1}^K$};

\node[flowbox] (VX) at (\xR,\yProj)
{\textbf{Value Projection} $(V_X)$\\
$V_X = XW_V$\\
$\mathbb{R}^{N\times d_v}$};

\node[flowbox] (GammaK) at (\xR,\yResp)
{\textbf{Write Responsibilities} $(\Gamma^K)$\\
via Eq.~\eqref{eq:gma_responsibility}\\
$\mathbb{R}^{N\times K}$};

\node[memorybox] (Vtilde) at (\xR,\yMem)
{\textbf{Latent Memory} $(\tilde V)$\\
$\tilde V = (\Gamma^K)^\top V_X$\\
$\mathbb{R}^{K\times d_v}$};

\node[flowbox] (Z) at (\xR,\yNorm)
{\textbf{Normalizer} $(Z)$\\
$Z = (\Gamma^K)^\top \mathbf{1}_N$\\
$\mathbb{R}^{K}$};

\coordinate (outerR1) at ([xshift=1.45cm]VX.east);
\coordinate (outerR2) at ([xshift=1.45cm]GammaK.east);


\draw[arrow] (X.south west) -- (QX.north);
\draw[arrow] (X.south) -- (KX.north);
\draw[arrow] (X.south east) -- (VX.north);

\draw[arrow] (QX.south) -- (GammaQ.north);
\draw[arrow] (KX.south east) -- (GammaK.north west);

\draw[arrow] (GMM.west) -- (GammaQ.east);
\draw[arrow] (GMM.east) -- (GammaK.west);

\draw[arrow]
    (VX.east) -- (outerR1)
    |- ([yshift=0.0cm]Vtilde.east);

\draw[arrow]
    (GammaK.south) -- (Vtilde.north)
    node[midway,left] {write};

\draw[arrow]
    (GammaK.east) -- ([xshift=-0.5cm]outerR2)
    |- ([yshift=0.0cm]Z.east);

\draw[arrow]
    (GammaQ.south) -- (O.north)
    node[midway,left] {read};

\draw[arrow]
    (Vtilde.west) -| ([xshift=1.05cm,yshift=0.0cm]O.east)
    -- ([yshift=0cm]O.east);

\draw[arrow]
    (Z.west) -| ([xshift=1.05cm,yshift=0cm]O.east)
    -- ([yshift=0cm]O.east);

\end{tikzpicture}
\caption{End-to-end workflow of Gaussian Mixture Attention (GMA). The input sequence $X$ is first projected into query, key, and value representations $Q_X$, $K_X$, and $V_X$. A shared Gaussian Mixture Model with parameters $\{\pi_k,\mu_k,\Sigma_k\}_{k=1}^K$ maps $Q_X$ and $K_X$ to query and key responsibility matrices, $\Gamma^Q$ and $\Gamma^K$, using the posterior responsibility formula in Eq.~\eqref{eq:gma_responsibility}. The key responsibilities write the projected values $V_X$ into a latent memory $\tilde V$ and a component-wise normalizer $Z$. The query responsibilities then read from this normalized latent memory to produce the output sequence $O$. This replaces explicit $N\times N$ token-to-token attention with routing through $K$ latent Gaussian components.}
\label{fig:gma_flow}
\end{figure}
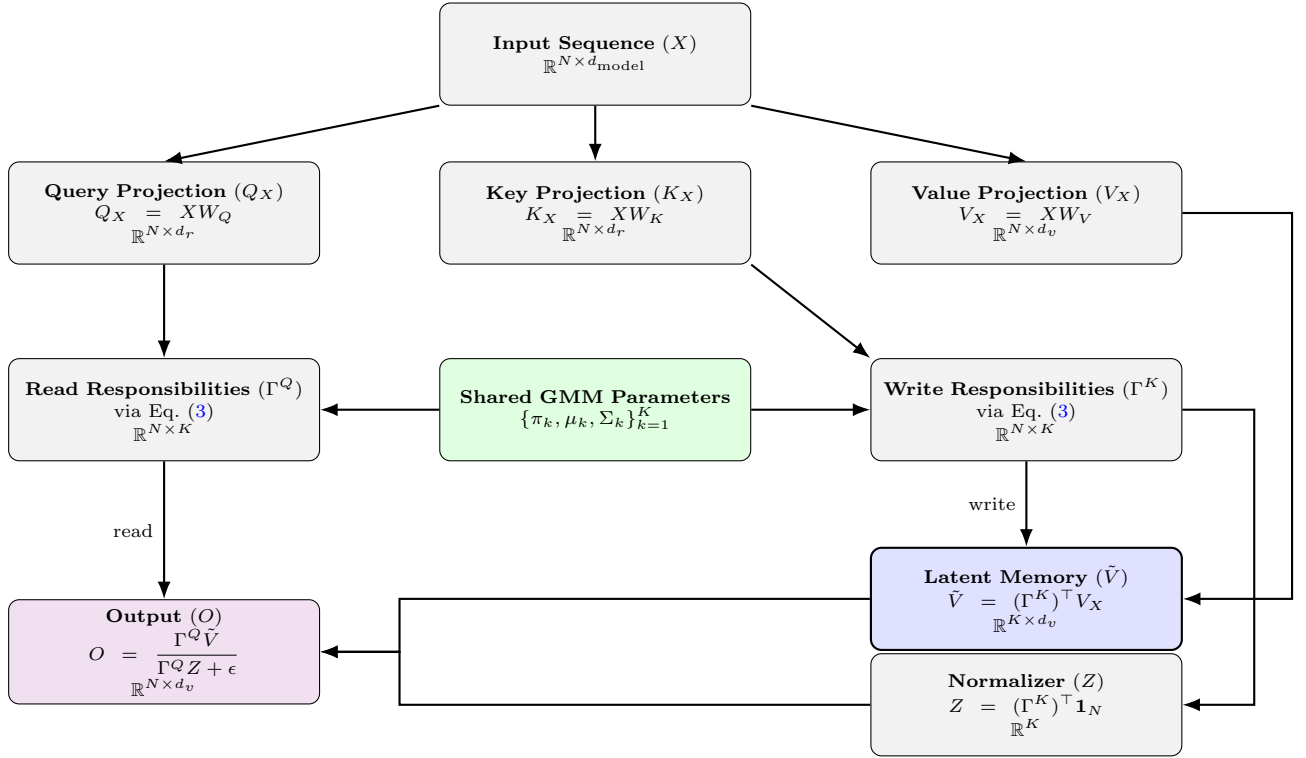

\subsection*{\textit{Understanding the GMA Workflow}}
\label{sec:understanding_gma_workflow}

The GMA computation can be summarized as \emph{responsibility-space affinity plus latent-memory routing}. Unlike standard dot-product attention, GMA does not directly score a query and key by the dot product $q_i^\top k_j$. Instead, it first maps both query and key representations into posterior responsibility vectors over the same $K$ learned Gaussian mixture components.

For a routing vector $x$, the responsibility of component $k$ is (Eq.\ref{eq:gma_responsibility})
\begin{equation*}
\label{eq:gma_workflow_responsibility}
    \gamma_k(x)
    =
    p(z=k\mid x)
    =
    \frac{
    \pi_k\mathcal{N}(x\mid \mu_k,\Sigma_k)
    }{
    \sum_{\ell=1}^K
    \pi_\ell\mathcal{N}(x\mid \mu_\ell,\Sigma_\ell)
    }.
\end{equation*}
For queries and keys, this gives
\[
    \gamma^Q_{i,k}=p(z=k\mid q_i),
    \qquad
    \gamma^K_{j,k}=p(z=k\mid k_j).
\]
Thus, the query vector $q_i$ is represented by the responsibility vector
\[
    \gamma^Q_i
    =
    (\gamma^Q_{i,1},\ldots,\gamma^Q_{i,K}),
\]
and the key vector $k_j$ is represented by
\[
    \gamma^K_j
    =
    (\gamma^K_{j,1},\ldots,\gamma^K_{j,K}).
\]
GMA compares query and key positions through the \textit{overlap} of these responsibility vectors:
\[
    \widetilde A^{\mathrm{GMA}}_{ij}
    =
    \sum_{k=1}^K
    \gamma^Q_{i,k}\gamma^K_{j,k}
    =
    \langle \gamma^Q_i,\gamma^K_j\rangle,
\]
which defines the un-normalized attention weight matrix (as compared to its normalised version Eq.\ref{eq:gma_induced_attention})
\begin{equation}
\label{eq:gma_induced_attention_unnormalised}
    \widetilde A^{\mathrm{GMA}}
    =
    \Gamma^Q(\Gamma^K)^\top.
\end{equation}
Therefore, two tokens have high GMA affinity if their query and key representations assign high posterior probability mass to similar Gaussian components.

The second part of the workflow is latent-memory routing, which avoids materializing the full pairwise affinity matrix in the efficient implementation. Rather than explicitly materializing the full $N\times N$ affinity matrix $\widetilde A^{\mathrm{GMA}}$, the key responsibilities first write values into a $K$-slot latent memory (Eq.\ref{eq:gma_latent_aggregation}):
\[
    \tilde V
    =
    (\Gamma^K)^\top V_X
    \in \mathbb{R}^{K\times d_v},
    \qquad
    Z
    =
    (\Gamma^K)^\top \mathbf{1}_N
    \in \mathbb{R}^{K}.
\]
Component-wise (Eq.\ref{eq:gma_latent_aggregation_componentwise}),
\[
    \tilde V_k
    =
    \sum_{j=1}^N
    \gamma^K_{j,k}V_{X,j},
    \qquad
    Z_k
    =
    \sum_{j=1}^N
    \gamma^K_{j,k}.
\]
Thus, $\tilde V_k$ is the \textit{memory slot} associated with Gaussian component $k$: it stores a responsibility-weighted sum of the value vectors whose keys are assigned to that component.

Finally, the query responsibilities read from this key-induced latent memory (Eq.\ref{eq:gma_normalized_broadcasting}):
\[
    O
    =
    \frac{
    \Gamma^Q\tilde V
    }{
    \Gamma^QZ+\epsilon
    }.
\]
For a single output token $i$, this becomes (Eq.\ref{eq:gma_output_componentwise})
\[
    O_i
    =
    \frac{
    \sum_{k=1}^K
    \gamma^Q_{i,k}\tilde V_k
    }{
    \sum_{k=1}^K
    \gamma^Q_{i,k}Z_k+\epsilon
    }
    =
    \frac{
    \sum_{k=1}^K
    \gamma^Q_{i,k}
    \sum_{j=1}^N
    \gamma^K_{j,k}V_{X,j}
    }{
    \sum_{k=1}^K
    \gamma^Q_{i,k}
    \sum_{j=1}^N
    \gamma^K_{j,k}
    +\epsilon
    }.
\]
Equivalently, GMA induces \textit{normalized} attention-style weights (Eq.\ref{eq:gma_induced_attention})
\[
    A^{\mathrm{GMA}}_{ij}
    =
    \frac{
    \sum_{k=1}^K
    \gamma^Q_{i,k}\gamma^K_{j,k}
    }{
    \sum_{\ell=1}^N
    \sum_{k=1}^K
    \gamma^Q_{i,k}\gamma^K_{\ell,k}
    +\epsilon
    },
\]
so that the output for query position $i$ is (Eq.\ref{eq:gma_attention})
\[
    O_i
    =
    \sum_{j=1}^N
    A^{\mathrm{GMA}}_{ij}V_{X,j}.
\]

In short, GMA first compares query and key positions in \textit{responsibility space}, through
$
    \widetilde A^{\mathrm{GMA}}_{ij}
    =
    \langle \gamma^Q_i,\gamma^K_j\rangle,
$
but it avoids forming all such pairwise affinities explicitly. Instead, it computes
$
    \tilde V=(\Gamma^K)^\top V_X
$
first, and then reads this key-induced memory using
$
    O
    =
    \frac{\Gamma^Q\tilde V}{\Gamma^QZ+\epsilon}.
$
So GMA is best understood as \textit{responsibility-space affinity followed by latent-memory routing}, rather than as dot-product matching against memory components. Appendix~\ref{app:understand_GMA} gives a more dimension-explicit walkthrough of the same computation.

\subsection{Memory and computational costs} \label{subsec:memory_and_compute_costs}
For one sequence and one attention head, assume the \textit{diagonal} covariance parameterization above and an associative implementation that does not materialize $A^{\mathrm{GMA}}$. Excluding the standard linear projections used to form $Q_X,K_X,V_X$, the main GMA routing operations have the following costs. Computing Gaussian-mixture responsibilities for both query and key projections costs approximately
\[
    2NKd_r + 2NK,
\]
where the first term accounts for evaluating diagonal Gaussian log-densities over $K$ components for both $Q_X$ and $K_X$, and the second term accounts for the row-wise normalization over components. The write step $\tilde V=(\Gamma^K)^\top V_X$ costs $NKd_v$, and computing $Z=(\Gamma^K)^\top\mathbf{1}_N$ costs $NK$. The read step $\Gamma^Q\tilde V$ costs $NKd_v$, computing the denominator $\Gamma^Q Z+\epsilon$ costs $NK$, and the final row-wise normalization of the output costs $Nd_v$. Thus, the total routing cost is
\begin{equation}
\label{eq:gma_total_cost}
    \mathcal{C}_{\mathrm{GMA}}
    \approx
    2NKd_r
    +
    2NKd_v
    +
    4NK
    +
    Nd_v .
\end{equation}
Equivalently,
\[
    \mathcal{C}_{\mathrm{GMA}}
    =
    \mathcal{O}(NKd_r+NKd_v).
\]
Including the projection matrices adds the usual cost
\[
    \mathcal{C}_{\mathrm{proj}}
    =
    \mathcal{O}\!\left(
    Nd_{\mathrm{model}}(2d_r+d_v)
    \right),
\]
for forming $Q_X=XW_Q$, $K_X=XW_K$, and $V_X=XW_V$. 
Hence, for fixed $d_r,d_v$, number of heads, and fixed number of mixture components $K$, GMA scales linearly with sequence length $N$ while avoiding explicit $\mathcal{O}(N^2)$ attention storage\footnote{The matrix-multiplication associativity underlying this cost saving is discussed in Appendix~\ref{app:matrix_association_rule}.}.

The attention-specific activation \textit{storage} is similarly linear in $N$. The two responsibility matrices require $2NK$ entries, the latent memory and normalizer require $Kd_v+K$ entries, and the output requires $Nd_v$ entries. Thus, aside from standard projection activations, the routing-specific storage is
\[
    2NK + Kd_v + K + Nd_v
    =
    \mathcal{O}(NK+Kd_v+Nd_v),
\]
which reduces to $\mathcal{O}(NK)$ when focusing on the dominant responsibility storage for fixed $d_v$ and $K$.

\subsection{Causal GMA for Autoregressive Modelling}
\label{sec:causal_gma}

The formulation above is bidirectional: each output position may aggregate values from all key positions. For autoregressive language modelling, however, the output at position $i$ must depend only on prefix positions $j\leq i$. Therefore, the non-causal aggregation in Eq.~\eqref{eq:gma_latent_aggregation} is replaced by prefix versions of the latent memory and normalizer. For each position $i$ and component $k$, define
\begin{equation}
\label{eq:causal_gma_prefix_stats} \tag{\ref{eq:gma_latent_aggregation_componentwise}b}
    \tilde V^{(i)}_k
    =
    \sum_{j\leq i}
    \gamma^K_{j,k}V_{X,j}
    \in\mathbb{R}^{d_v},
    \qquad
    Z^{(i)}_k
    =
    \sum_{j\leq i}
    \gamma^K_{j,k}
    \in\mathbb{R}.
\end{equation}
Here $\tilde V^{(i)}_k$ and $Z^{(i)}_k$ are the causal prefix counterparts of $\tilde V_k$ and $Z_k$ in Eq.~\eqref{eq:gma_latent_aggregation_componentwise}. In other words, $\tilde V^{(i)}_k$ stores the value information assigned to component $k$ using only tokens up to position $i$, while $Z^{(i)}_k$ stores the corresponding prefix key-responsibility mass.

The causal GMA output for position $i$ is then\footnote{Unlike the bidirectional case in Eq.~\eqref{eq:gma_normalized_broadcasting}, there is no single full-sequence matrix expression of the form $O=\Gamma^Q\tilde V/(\Gamma^Q Z+\epsilon)$ unless one introduces an additional prefix-memory tensor. The reason is that the causal memory $\tilde V^{(i)}_k$ and normalizer $Z^{(i)}_k$ depend on the output position $i$: each query position reads from a different prefix-restricted memory built only from tokens $j\leq i$. Therefore, we wrote the causal formulation row-wise as in Eq.~\eqref{eq:causal_gma_output_componentwise}.}
\begin{equation}
\label{eq:causal_gma_output_componentwise} \tag{\ref{eq:gma_output_componentwise}b}
    O_i
    =
    \frac{
    \sum_{k=1}^K \gamma^Q_{i,k}\tilde V^{(i)}_k
    }{
    \sum_{k=1}^K \gamma^Q_{i,k}Z^{(i)}_k
    +\epsilon
    }.
\end{equation}
Thus, query position $i$ reads only from a prefix-restricted latent memory, rather than from the full-sequence memory used in the bidirectional case.

Equivalently, causal GMA induces attention weights supported only on the prefix:
\begin{equation}
\label{eq:causal_gma_induced_attention} \tag{\ref{eq:gma_induced_attention}b}
    A^{\mathrm{cGMA}}_{ij}
    =
    \mathbf{1}\{j\leq i\}
    \frac{
    \sum_{k=1}^K \gamma^Q_{i,k}\gamma^K_{j,k}
    }{
    \sum_{\ell\leq i}
    \sum_{k=1}^K
    \gamma^Q_{i,k}\gamma^K_{\ell,k}
    +\epsilon
    }.
\end{equation}
Therefore,
\begin{equation}
\label{eq:causal_gma_attention} \tag{\ref{eq:gma_attention}b}
    O_i
    =
    \sum_{j\leq i}
    A^{\mathrm{cGMA}}_{ij}V_{X,j}
\end{equation}
up to the numerical stabilizer.
In implementation, the prefix memories $\tilde V^{(i)}_k$ and prefix normalizers $Z^{(i)}_k$ are computed by cumulative sums along the sequence dimension, so the causal variant preserves the same fixed-$K$ linear scaling in sequence length while enforcing autoregressive causality.

\subsection{End-to-End Parameter Learning}
\label{sec:end_to_end_parameter_learning}

Traditional Gaussian mixture models (GMMs) are typically fit to static datasets by maximum-likelihood estimation using the Expectation-Maximization (EM) algorithm \cite{dempster1977maximum,mclachlan2000finite,bishop2006pattern,huang2025GMA_sampling}. However, deploying an iterative EM loop inside each attention layer during every forward pass would be computationally prohibitive and would disrupt the continuous gradient flow required for end-to-end neural network training. Instead, we treat the GMM parameters in each GMA layer, i.e. the means $\mu_k$, diagonal covariances $\Sigma_k$, and mixture priors $\pi_k$, as fully differentiable learnable parameters optimized jointly with the rest of the model by standard backpropagation.

This treatment is related in spirit to \textit{Mixture Density Networks} (MDNs) \cite{bishop1994mixture,huang2026GJE}, which combine neural networks with mixture density models and train mixture parameters by gradient-based optimization. The role of the mixture model, however, is different: MDNs parameterize an output conditional density, whereas GMA uses shared mixture components to define latent responsibility-based routing inside an attention layer. Thus, the GMA components are not fitted by standalone maximum-likelihood density estimation on a fixed dataset; rather, they are task-adapted latent routing components shaped by the predictive objective.

For one GMA layer, the trainable routing parameters are
\begin{equation} \label{eq:GMA_trainable_params}
    \Theta_{\mathrm{GMA}}
    =
    \{W_Q,W_K,W_V,\mu,\omega,\alpha\},
\end{equation}
where $W_Q,W_K,W_V$ are the query, key, and value projection matrices, $\mu$ contains the Gaussian component means, $\omega$ parameterizes the diagonal covariance entries, and $\alpha$ parameterizes the mixture-prior logits. These parameters are updated by minimizing the downstream task loss $\mathcal{L}_{\mathrm{task}}$, such as cross-entropy for language modelling or classification.

To ensure that the GMM parameters remain valid during gradient-based training, we use the following reparameterizations.

\begin{itemize}[label=-]
    \item \textbf{Means ($\mu_k$).} The latent routing centers are represented by unconstrained learnable parameters $\mu_k\in\mathbb{R}^{d_r}$, updated directly by the optimizer.

    \item \textbf{Covariances ($\Sigma_k$).} To guarantee that each diagonal covariance matrix remains strictly positive-definite, we maintain unconstrained learnable parameters $\omega\in\mathbb{R}^{K\times d_r}$ and set
    \begin{equation}
    \label{eq:gma_covariance_reparameterization}
        \sigma_{k,j}^2
        =
        \log(1+\exp(\omega_{k,j}))+\epsilon_\sigma,
        \qquad
        j=1,\ldots,d_r,
    \end{equation}
    where $\epsilon_\sigma>0$ is a small numerical constant. The covariance of component $k$ is then $\Sigma_k=\operatorname{diag}(\sigma_{k,1}^2,\ldots,\sigma_{k,d_r}^2)$.

    \item \textbf{Mixture priors ($\pi_k$).} To ensure that the mixture priors lie on the probability simplex, we maintain unconstrained logits $\alpha\in\mathbb{R}^{K}$ and apply a softmax transformation:
    \begin{equation}
    \label{eq:gma_prior_reparameterization}
        \pi_k
        =
        \frac{\exp(\alpha_k)}
        {\sum_{\ell=1}^K \exp(\alpha_\ell)},
        \qquad
        k=1,\ldots,K.
    \end{equation}
\end{itemize}

The training process is therefore fully differentiable\footnote{For details, Appendix~\ref{app:gradient_derivations} provides the corresponding gradient derivations for the reparameterized means, diagonal covariances, and mixture-prior logits, showing explicitly how gradients pass through the Gaussian responsibility computation during backpropagation.}. During the forward pass, the input sequence $X$ is projected into $Q_X$, $K_X$, and $V_X$ via Eq.~\eqref{eq:gma_qkv_projection}; the reparameterized GMM parameters define the query and key responsibility matrices $\Gamma^Q$ and $\Gamma^K$ through the Gaussian responsibility computation in Eq.~\eqref{eq:gma_responsibility}; and the output $O$ is computed by the normalized latent-memory operation in Eq.~\eqref{eq:gma_normalized_broadcasting}. During the backward pass, gradients of $\mathcal{L}_{\mathrm{task}}$ propagate through the Gaussian responsibility computation in Eq.~\eqref{eq:gma_responsibility}, the softmax mixture-prior parameterization in Eq.~\eqref{eq:gma_prior_reparameterization}, the softplus covariance parameterization in Eq.~\eqref{eq:gma_covariance_reparameterization}, and the latent aggregation and broadcasting operations in Eqs.~\eqref{eq:gma_latent_aggregation} and \eqref{eq:gma_normalized_broadcasting}. Consequently, no EM loop or auxiliary clustering loss is required inside the attention layer.
The resulting components are interpreted as task-adapted latent routing components rather than as classical GMM clusters obtained by standalone maximum-likelihood estimation.

\section{Theoretical Analysis}
\label{sec:theory}

To understand the learning dynamics and representational capacity of GMA, we analyze its gradient flow, its induced non-negative low-rank affinity structure, and the local stability of the Gaussian responsibility map. Throughout this section, $x_i\in\mathbb{R}^{d_r}$ denotes a generic routing vector, which may be either a query routing vector $q_i=(Q_X)_{i,:}$ or a key routing vector $k_i=(K_X)_{i,:}$ after the projection Eq.\eqref{eq:gma_qkv_projection}.

\subsection{Gradient Flow and Optimization Stability}
\label{sec:gradient_flow}

A primary concern when replacing dot-product attention with Gaussian-mixture routing is the stability of the responsibility gradients, since Gaussian log-densities contain Mahalanobis-distance terms\footnote{A related pitfall is the ``Mahalanobis Trace Trap'' \cite{huang2026GJE}: if the same mini-batch is used both to estimate an empirical covariance matrix and to evaluate Mahalanobis distances against its inverse, the quadratic term can algebraically collapse to a trace constant and lose its gradient signal. The same issue can arise component-wise in a GMM if each $\Sigma_k$ is recomputed as the responsibility-weighted empirical covariance of the same batch being scored. 
GMA avoids this failure mode because $\mu_k$, $\Sigma_k$, and $\pi_k$ are learned layer parameters, reparameterized as in Eqs.~\eqref{eq:gma_covariance_reparameterization} and \eqref{eq:gma_prior_reparameterization}, rather than empirical covariance estimates recomputed inside the forward pass.} and exponentials. We show that the softmax-normalized responsibility map yields a bounded and responsibility-modulated gradient structure.

For a routing vector $x_i\in\mathbb{R}^{d_r}$ and component $k$, define the pre-normalized log-density score
\begin{equation}
\label{eq:gma_log_density_score}
    s_{i,k}
    =
    \log \pi_k
    -
    \frac{d_r}{2}\log(2\pi)
    -
    \frac{1}{2}\log|\Sigma_k|
    -
    \frac{1}{2}
    (x_i-\mu_k)^\top\Sigma_k^{-1}(x_i-\mu_k).
\end{equation}
Then the responsibility in Eq.~\eqref{eq:gma_responsibility} can equivalently be written as a softmax over these scores:
\begin{equation}
\label{eq:gma_responsibility_score_softmax} \tag{\ref{eq:gma_responsibility}b}
    \gamma_{i,k}
    =
    \frac{\exp(s_{i,k})}
    {\sum_{\ell=1}^K \exp(s_{i,\ell})}.
\end{equation}

Let $\mathcal{L}$ be the downstream loss and let $g_{i,j}=\partial \mathcal{L}/\partial \gamma_{i,j}$ denote the upstream gradient arriving at the responsibility vector for token $i$. By the softmax Jacobian\footnote{For the binary sigmoid, the derivative is $\sigma(x)(1-\sigma(x))$. The softmax generalizes this: its Jacobian has diagonal terms $\gamma_k(1-\gamma_k)$ and off-diagonal terms $-\gamma_j\gamma_k$, which account for competition between components.},
\begin{equation}
\label{eq:gma_softmax_jacobian}
    \frac{\partial \gamma_{i,j}}{\partial s_{i,k}}
    =
    \gamma_{i,j}
    \left(
    \mathbf{1}\{j=k\}
    -
    \gamma_{i,k}
    \right).
\end{equation}
Therefore, the gradient entering the score $s_{i,k}$ is\footnote{Using
$\partial\gamma_{i,j}/\partial s_{i,k}
=\gamma_{i,j}(\mathbf{1}\{j=k\}-\gamma_{i,k})$, we have
$\sum_j g_{i,j}\gamma_{i,j}(\mathbf{1}\{j=k\}-\gamma_{i,k})
= g_{i,k}\gamma_{i,k}-\gamma_{i,k}\sum_j\gamma_{i,j}g_{i,j}
=\gamma_{i,k}\bigl(g_{i,k}-\sum_j\gamma_{i,j}g_{i,j}\bigr)$.}
\begin{equation}
\label{eq:gma_score_gradient_general}
    \frac{\partial \mathcal{L}}{\partial s_{i,k}}
    =
    \sum_{j=1}^K
    \frac{\partial \mathcal{L}}{\partial \gamma_{i,j}}
    \frac{\partial \gamma_{i,j}}{\partial s_{i,k}}
    =
    \gamma_{i,k}
    \left(
    g_{i,k}
    -
    \sum_{j=1}^K
    \gamma_{i,j}g_{i,j}
    \right).
\end{equation}
Thus, the gradient passed into each component score is a responsibility-weighted and mean-centered upstream gradient. For any learnable scalar parameter $\theta$ that enters the score, such as a component mean $\mu_{k,m}$, a covariance parameter $\omega_{k,m}$ (Eq.~\eqref{eq:gma_covariance_reparameterization}), a mixture-prior logit $\alpha_k$ (Eq.~\eqref{eq:gma_prior_reparameterization}), or an entry of the projection matrices $W_Q$ and $W_K$ through the routing vectors $q_i$ and $k_i$ (Eq.~\eqref{eq:gma_qkv_projection}), the chain rule gives
\begin{equation}
\label{eq:gma_parameter_gradient_general_a}
    \frac{\partial \mathcal{L}}{\partial \theta}
    =
    \sum_i
    \sum_{k=1}^K
    \frac{\partial \mathcal{L}}{\partial s_{i,k}}
    \frac{\partial s_{i,k}}{\partial \theta},
\end{equation}
where $\partial \mathcal{L}/\partial s_{i,k}$ is given by Eq.~\eqref{eq:gma_score_gradient_general}. Equivalently, expanding the softmax Jacobian explicitly gives
\begin{equation} \label{eq:gma_parameter_gradient_general_b} \tag{\ref{eq:gma_parameter_gradient_general_a}b}
    \frac{\partial \mathcal{L}}{\partial \theta}
    =
    \sum_i
    \sum_{j=1}^K
    \sum_{k=1}^K
    \frac{\partial \mathcal{L}}{\partial \gamma_{i,j}}
    \frac{\partial \gamma_{i,j}}{\partial s_{i,k}}
    \frac{\partial s_{i,k}}{\partial \theta}.
\end{equation}
This expression is the appropriate backpropagation form for the responsibility-based GMA layer.

The familiar factor $\gamma_{i,k}(1-\gamma_{i,k})$ appears, from Eq.\eqref{eq:gma_softmax_jacobian} with $j=k$, as the direct sensitivity of a component's own responsibility to its own score:
\begin{equation}
\label{eq:gma_self_responsibility_derivative} \tag{\ref{eq:gma_softmax_jacobian}b}
    \frac{\partial \gamma_{i,k}}{\partial s_{i,k}}
    =
    \gamma_{i,k}(1-\gamma_{i,k}).
\end{equation}
For example, since
\begin{equation}
\label{eq:gma_score_mean_derivative}
    \nabla_{\mu_k}s_{i,k}
    =
    \Sigma_k^{-1}(x_i-\mu_k),
\end{equation}
the direct derivative of $\gamma_{i,k}$ with respect to the component mean is
\begin{equation}
\label{eq:gma_mean_responsibility_derivative}
    \nabla_{\mu_k}\gamma_{i,k}
    =
    \gamma_{i,k}(1-\gamma_{i,k})
    \Sigma_k^{-1}(x_i-\mu_k).
\end{equation}
For $j\neq k$, the cross-component derivative is
\begin{equation}
\label{eq:gma_cross_mean_responsibility_derivative}
    \nabla_{\mu_k}\gamma_{i,j}
    =
    -
    \gamma_{i,j}\gamma_{i,k}
    \Sigma_k^{-1}(x_i-\mu_k).
\end{equation}

These identities show that Gaussian routing gradients are naturally modulated by the uncertainty of the responsibility assignment. The direct self-sensitivity in Eq.~\eqref{eq:gma_self_responsibility_derivative} is largest when the component assignment is uncertain and vanishes as $\gamma_{i,k}$ approaches $0$ or $1$. In the full loss gradient, Eq.~\eqref{eq:gma_score_gradient_general} shows that the update is further centered by the responsibility-weighted average upstream gradient. This gives a structured gradient-modulation mechanism, although it does not by itself rule out component under-use, component collapse, or poor conditioning; these properties must still be monitored empirically through component-usage diagnostics.

\subsection{Non-Negative Low-Rank Affinity Interpretation}
\label{sec:nonnegative_low_rank_affinity}

We have the implicit unnormalized GMA affinity matrix (Eq.\ref{eq:gma_induced_attention_unnormalised})
\[
\label{eq:gma_unnormalized_affinity} 
    \widetilde A^{\mathrm{GMA}}
    =
    \Gamma^Q(\Gamma^K)^\top,
\]
where $\Gamma^Q,\Gamma^K\in\mathbb{R}^{N\times K}$ are non-negative responsibility matrices whose rows lie on the probability simplex. Therefore, $\widetilde A^{\mathrm{GMA}}$ is non-negative and satisfies
\begin{equation}
\label{eq:gma_affinity_rank_bound}
    \operatorname{rank}(\widetilde A^{\mathrm{GMA}})
    \leq
    K.
\end{equation}

The normalized affinity used by GMA, i.e. Eq.~\eqref{eq:gma_induced_attention}, is obtained by row-wise normalization:
\begin{equation}
\label{eq:gma_normalized_affinity_matrix} \tag{\ref{eq:gma_induced_attention}b}
    A^{\mathrm{GMA}}_{ij}
    =
    \frac{
    \widetilde A^{\mathrm{GMA}}_{ij}
    }{
    \sum_{\ell=1}^N \widetilde A^{\mathrm{GMA}}_{i\ell}
    +\epsilon
    }.
\end{equation}
Equivalently, if $d_i=\sum_{\ell=1}^N\widetilde A^{\mathrm{GMA}}_{i\ell}
+\epsilon$, then
\[
    A^{\mathrm{GMA}}
    =
    \operatorname{diag}(d)^{-1}\widetilde A^{\mathrm{GMA}}.
\]
Since $\operatorname{diag}(d)^{-1}$ is a diagonal left multiplication with
positive entries, the normalized affinity also has rank at most $K$.

Thus, GMA induces a constrained non-negative low-rank affinity structure. Here ``affinity'' refers to the induced token-to-token relatedness $\widetilde A^{\mathrm{GMA}}_{ij} =\sum_{k=1}^K\gamma^Q_{i,k}\gamma^K_{j,k}$, while ``constrained'' refers to the fact that the factor matrices $\Gamma^Q$ and $\Gamma^K$ are not arbitrary non-negative matrices: each row is a posterior probability distribution over latent Gaussian components. This is related to the intuition behind non-negative matrix factorization (NMF), in which non-negativity encourages additive rather than subtractive combinations \cite{lee1999learning}. However, GMA is not an NMF algorithm in the classical optimization sense: it does not solve a reconstruction or approximation problem for a given matrix. Instead, it learns two probabilistic responsibility factors as part of a task-trained neural sequence model.

This distinction is important. The product $\Gamma^Q(\Gamma^K)^\top$ gives a rank-at-most-$K$ non-negative affinity, but the effective rank may be smaller if some components are unused or if different responsibility columns become dependent. Consequently, the interpretation is a \emph{constrained non-negative low-rank affinity} interpretation rather than a claim that GMA always produces an exact rank-$K$ decomposition. Similarly, component specialization and parts-based structure are possible consequences of the learned responsibility factors, but they are not guaranteed by the factorization alone. They should therefore be evaluated empirically using the responsibility matrices, for example through component-usage statistics, responsibility heatmaps, top-token analysis, or alignment with token-category labels.

\subsection{Local Lipschitz Continuity and Routing Stability}
\label{sec:lipschitz_stability}

Standard dot-product attention can be sensitive to large embedding norms, which motivates the scaling factor $1/\sqrt{d_k}$ in Eq.~\eqref{eq:dot_product_attention}. In GMA, routing is instead determined by Gaussian responsibilities. We now show that, for fixed learned GMM parameters, under variance lower bounds and bounded routing vectors, the responsibility map is locally Lipschitz.

\begin{theorem}[Local Lipschitz Continuity of GMA Responsibilities]
\label{thm:lipschitz}
Assume that the diagonal variances satisfy $\sigma^2_{k,m}\geq \epsilon_\sigma>0$ for all components $k$ and coordinates $m=1,\ldots,d_r$, and that routing vectors and component means lie in a bounded set such that
\[
    \|x-\mu_k\|_2\leq R
    \qquad
    \text{for all } k.
\]
Then, for fixed GMM parameters, the responsibility map $x\mapsto \gamma(x)\in\Delta^{K-1}$ is locally Lipschitz continuous. In particular, it has a finite Lipschitz constant depending on $R$, $\epsilon_\sigma$, and $K$.
\end{theorem}

\begin{proof}
For each component, the score from Eq.~\eqref{eq:gma_log_density_score} can be written as
\[
    s_k(x)
    =
    \log \pi_k
    -
    \frac{d_r}{2}\log(2\pi)
    -
    \frac{1}{2}\log|\Sigma_k|
    -
    \frac{1}{2}
    (x-\mu_k)^\top\Sigma_k^{-1}(x-\mu_k).
\]
Its gradient with respect to $x$ is\footnote{This quantity is the score of the component log-density with respect to the input variable, i.e. $\nabla_x \log p_k(x)$. In classical statistics, the term ``score'' often refers to the gradient of the log-likelihood with respect to model parameters; in score-based modelling convention, the gradient is taken with respect to the data or representation variable \cite{huang2022score}.}
\[
    \nabla_x s_k(x)
    =
    -
    \Sigma_k^{-1}(x-\mu_k).
\]
Since $\Sigma_k$ is diagonal\footnote{Here we use the diagonal covariance parameterization from Eq.~\eqref{eq:gma_covariance_reparameterization}, namely $\Sigma_k=\operatorname{diag} \sigma^2_{k,1},\ldots,\sigma^2_{k,d_r})$. The notation $\|\Sigma_k^{-1}\|_2$ denotes the spectral/operator norm of the inverse covariance, not the determinant. For a diagonal positive-definite matrix, this norm is the largest diagonal entry of $\Sigma_k^{-1}$.} and $\sigma^2_{k,m}\geq \epsilon_\sigma$ for all $m=1,\ldots,d_r$, we have
\[
    \|\Sigma_k^{-1}\|_2
    =
    \max_{1\leq m\leq d_r}
    \frac{1}{\sigma^2_{k,m}}
    \leq
    \frac{1}{\epsilon_\sigma}.
\]
Using the boundedness assumption $\|x-\mu_k\|_2\leq R$ and the operator-norm inequality, we obtain
\[
    \|\nabla_x s_k(x)\|_2
    =
    \|-\Sigma_k^{-1}(x-\mu_k)\|_2
    \leq
    \|\Sigma_k^{-1}\|_2\|x-\mu_k\|_2
    \leq
    \frac{1}{\epsilon_\sigma}R
    =
    \frac{R}{\epsilon_\sigma}.
\]
Recall that a differentiable map $f$ is Lipschitz on a region $\mathcal{X}$ if there exists a finite constant $L<\infty$ such that
\[
    \|f(x)-f(x')\|_2
    \leq
    L\|x-x'\|_2,
    \qquad
    x,x'\in\mathcal{X}.
\]
For differentiable functions, a sufficient condition is that the Jacobian norm is uniformly bounded on $\mathcal{X}$:
\[
    \sup_{x\in\mathcal{X}}\|J_f(x)\|_2 < \infty.
\]
Applying this to the score map $s(x)=(s_1(x),\ldots,s_K(x))$, its Jacobian has rows $\nabla_x s_k(x)^\top$. Since each row satisfies $\|\nabla_x s_k(x)\|_2\leq R/\epsilon_\sigma$, the Frobenius norm gives
\[
    \|J_s(x)\|_2
    \leq
    \|J_s(x)\|_F
    =
    \left(
    \sum_{k=1}^K
    \|\nabla_x s_k(x)\|_2^2
    \right)^{1/2}
    \leq
    \frac{\sqrt{K}R}{\epsilon_\sigma}.
\]
Thus, the score map has a finite local Lipschitz constant, with one simple upper bound proportional to $\sqrt{K}R/\epsilon_\sigma$.

The softmax map from scores to responsibilities has a uniformly bounded Jacobian on $\mathbb{R}^K$, and is therefore globally Lipschitz. Hence the composition
\[
    x
    \mapsto
    s(x)
    \mapsto
    \operatorname{softmax}(s(x))
    =
    \gamma(x)
\]
is locally Lipschitz. The resulting Lipschitz constant is finite and depends on $R$, $\epsilon_\sigma$, and $K$.
\end{proof}

This result shows that, under the stated boundedness and variance-lower-bound assumptions, small perturbations of a routing vector cannot cause unbounded changes in its responsibility vector. The statement should be interpreted as a local routing-stability guarantee, not as a global robustness guarantee for the entire Transformer architecture.

For fixed model parameters and fixed sequence length $N$, the normalized GMA output is also locally stable on bounded domains. The latent aggregation $\tilde V=(\Gamma^K)^\top V_X$ and the normalizer $Z=(\Gamma^K)^\top\mathbf{1}_N$ are finite sums of products of locally Lipschitz responsibility functions and bounded value vectors, while the broadcasting step divides by $\Gamma^QZ+\epsilon$. Since $\epsilon>0$ and all responsibilities are non-negative, the denominator is bounded below by $\epsilon$. Therefore, on bounded sets of routing vectors and values, the full GMA mapping defined by Eq.~\eqref{eq:gma_normalized_broadcasting} is a composition of locally Lipschitz operations and is itself locally Lipschitz. The corresponding Lipschitz constant may depend on $N$, $K$, the parameter bounds, the variance lower bound, and the value-vector bounds.

\section[Experiments]{Experiments\protect\footnote{Part of the experimental code was developed with kind assistance from ChatGPT-5.5 \citep{openai2026chatgpt55}, which the authors gratefully acknowledge.}} \label{sec:experiments}

To empirically validate GMA, we evaluate four aspects: (1) controlled systems profiling of memory and throughput, (2) long-context classification accuracy on LRA, (3) autoregressive language modelling on WikiText-103, and (4) the interpretable latent structures induced by learned responsibility matrices. All experiments in this paper evaluate GMA in \textit{self-attention}-style settings rather than cross-attention settings. Specifically, the systems profiling and LRA experiments use bidirectional sequence mixing, the WikiText-103 experiment uses causal self-attention-style sequence mixing, and the interpretability analysis examines the learned query responsibilities of the trained causal GMA model.

\subsection*{\textit{Experimental Setup and Baselines}}
We compare GMA against standard dot-product Multi-Head Attention (MHA) \citep{vaswani2017attention} and a suite of efficient sequence-modelling baselines chosen to represent distinct architectural trade-offs:

\begin{boxlabel}
\item \textbf{Low-Rank \& Static Projections:} \textbf{\textit{Linformer}} \citep{wang2020linformer} projects the sequence dimension $N$ into a smaller fixed dimension $k$. This baseline tests a static low-rank compression approach against GMA's input-dependent responsibility routing.

\item \textbf{Kernel Approximations:} \textbf{\textit{Performer}} \citep{choromanski2021rethinking} uses random orthogonal features, and \textbf{\textit{Linear Transformer}} \citep{katharopoulos2020transformers} uses an ELU-based (exponential linear unit) feature map. These baselines compare GMA against fixed or random feature-map approaches to linear attention.

\item \textbf{Modern State-Space Models:} \textbf{\textit{Mamba}} \citep{gu2023mamba} is a strong selective state-space model with linear scaling in sequence length. We include it to situate GMA against contemporary sequence-modelling alternatives that do not use attention-style token-to-token routing.
\end{boxlabel}

Across all experiments, GMA uses the multi-head normalized formulation described in Section~\ref{sec:GMA_method}, with mixture priors $\pi_k$ initialized uniformly. The number of Gaussian components $K$ is chosen according to the purpose of each experiment. In the controlled systems profiling experiment, we report a component ablation over $K\in\{64,128,256,512\}$ to characterize the efficiency--capacity trade-off. In the LRA accuracy experiments, we use $K=128$ as a moderate setting that keeps responsibility-memory cost manageable while retaining sufficient routing capacity. In the WikiText-103 language-modelling experiments, we evaluate both $K=128$ and $K=256$ causal GMA variants to measure how increasing the number of components affects perplexity and throughput. All models are implemented in \textit{PyTorch} \citep{pytorch2019} and trained using the \textit{AdamW} optimizer \citep{loshchilov2019AdamW}.

\subsection{Controlled Systems Profiling}
\label{sec:controlled_systems_profiling}

To isolate the computational behavior of each sequence-mixing mechanism, we first conduct\footnote{Unless otherwise stated, all mechanisms are evaluated on a single NVIDIA A100-SXM4-40GB GPU using PyTorch 2.10.0 with CUDA 12.8, bfloat16 mixed precision, batch size 8, 12 heads, and model hidden dimension $d_{\mathrm{model}}=768$.} controlled systems profiling using synthetic sequences. This experiment does not measure task accuracy; instead, it empirically tests the memory and throughput behavior predicted by the GMA cost analysis in Section~\ref{subsec:memory_and_compute_costs}. Specifically, we measure incremental peak GPU memory allocation and forward--backward throughput for a \textit{single attention/mixing block} under a shared input shape. This separates the computational scaling claims from the downstream representation-learning claims.

\begin{table}[H]
    \centering
    \caption{Single-block systems profiling of computational efficiency. Memory and throughput are empirical measurements for a forward-backward pass. Parameter counts, shown in parentheses, are reported in millions for each sequence-mixing layer. \underline{Underlined} values denote the best value in each metric, i.e. lowest memory or highest throughput.}
    \label{tab:efficiency}
    \resizebox{\textwidth}{!}{
    \scriptsize
    \begin{tabular}{lcccccc}
        \toprule
        \textbf{Model} &
        \textbf{$N=1$K Mem. (Params)} &
        \textbf{$N=1$K Speed} &
        \textbf{$N=2$K Mem. (Params)} &
        \textbf{$N=2$K Speed} &
        \textbf{$N=4$K Mem. (Params)} &
        \textbf{$N=4$K Speed} \\
        \midrule
        Standard MHA (Eager)
        & 1.47 GB (2.4M) & 1.02M tok/s
        & 5.81 GB (2.4M) & 610K tok/s
        & 23.07 GB (2.4M) & 186K tok/s \\

        Standard MHA (SDPA)
        & \underline{0.15 GB} (2.4M) & \underline{3.33M tok/s}
        & \underline{0.25 GB} (2.4M) & 3.34M tok/s
        & \underline{0.50 GB} (2.4M) & 2.38M tok/s \\
        \midrule

        Linear Transformer
        & 0.18 GB (2.4M) & 2.88M tok/s
        & 0.36 GB (2.4M) & \underline{3.76M tok/s}
        & 0.73 GB (2.4M) & \underline{3.96M tok/s} \\

        Linformer ($k=256$)
        & 0.41 GB (2.9M) & 2.17M tok/s
        & 0.82 GB (3.4M) & 2.56M tok/s
        & 1.64 GB (4.4M) & 2.68M tok/s \\

        Mamba
        & 0.27 GB (3.8M) & 2.06M tok/s
        & 0.55 GB (3.8M) & 1.99M tok/s
        & 1.11 GB (3.8M) & 2.07M tok/s \\

        Performer ($M=256$)
        & 0.58 GB (2.4M) & 1.30M tok/s
        & 1.15 GB (2.4M) & 1.37M tok/s
        & 2.30 GB (2.4M) & 1.39M tok/s \\
        \midrule

        \textbf{GMA (Ours, $K=64$)}
        & 0.35 GB (2.5M) & 688K tok/s
        & 0.71 GB (2.5M) & 722K tok/s
        & 1.37 GB (2.5M) & 742K tok/s \\

        \textbf{GMA (Ours, $K=128$)}
        & 0.50 GB (2.6M) & 566K tok/s
        & 1.04 GB (2.6M) & 592K tok/s
        & 2.02 GB (2.6M) & 607K tok/s \\

        \textbf{GMA (Ours, $K=256$)}
        & 0.85 GB (2.8M) & 416K tok/s
        & 1.70 GB (2.8M) & 430K tok/s
        & 3.39 GB (2.8M) & 437K tok/s \\

        \textbf{GMA (Ours, $K=512$)}
        & 1.59 GB (3.2M) & 270K tok/s
        & 3.17 GB (3.2M) & 277K tok/s
        & 6.32 GB (3.2M) & 280K tok/s \\
        \bottomrule
    \end{tabular}
    }
\vspace{2mm}
\begin{minipage}{0.98\textwidth}
\scriptsize
\textit{Note:} ``Standard MHA (Eager)'' denotes a conventional explicit scaled dot-product MHA implementation that materializes the attention score/probability tensors in the standard eager PyTorch execution path. ``Standard MHA (SDPA)'' denotes PyTorch's optimized scaled-dot-product-attention backend, which uses fused memory-efficient kernels when available. Both compute standard dot-product attention; they differ in implementation backend rather than in the mathematical attention definition.
\end{minipage}
\end{table}

As shown in Table~\ref{tab:efficiency}, eager standard MHA exhibits the expected quadratic memory growth, reaching 23.07 GB at $N=4000$, together with a marked throughput drop. In contrast, GMA shows approximately linear empirical memory scaling in $N$ for fixed $K$. For example, with $K=256$, memory increases from 0.85 GB at $N=1000$ to 1.70 GB at $N=2000$ and 3.39 GB at $N=4000$, while throughput remains nearly constant. This behavior is consistent with the fixed-$K$ activation scaling analyzed in Section~\ref{subsec:memory_and_compute_costs}.

The $K$-ablation exposes the expected efficiency--capacity trade-off. GMA's parameter count grows with $K$ but is independent of $N$, whereas its activation memory scales through the responsibility tensors $\Gamma^Q,\Gamma^K\in\mathbb{R}^{B\times H\times N\times K}$, where $B$, $H$, $N$, and $K$ denote batch size, number of heads, sequence length, and number of Gaussian mixture components, respectively. Thus, for fixed $K$, memory scales linearly with sequence length, while for fixed $N$, both memory and arithmetic cost scale linearly with the number of components. At $N=4000$, memory rises from 1.37 GB for $K=64$ to 6.32 GB for $K=512$, while the parameter count increases from 2.5M to 3.2M. By comparison, Linformer's static projection matrices introduce an explicit sequence-length-dependent parameter cost, increasing from 2.9M parameters at $N=1000$ to 4.4M at $N=4000$.

Overall, the profiling results empirically support the intended linear-in-$N$ scaling of GMA, but also show that this high-level PyTorch implementation,  currently without any optimisation, has larger constant factors than optimized SDPA, Linear Transformer, Mamba, and Performer baselines. This overhead comes from computing Gaussian log densities, Mahalanobis terms, and responsibility normalizers. The dominant responsibility activation storage is $\mathcal{O}(BHNK)$. Summing over heads, the arithmetic cost is approximately $\mathcal{O}(BNKD)$ when the per-head routing and value dimensions sum to the model hidden dimension $D$. Therefore, these results should be read as evidence for GMA's empirical scaling behavior, not as evidence that the present implementation is the fastest raw sequence mixer. The downstream experiments test whether this probabilistic routing cost yields competitive accuracy and interpretable responsibility structure.

\subsection{Long-Context Task Accuracy: LRA}
\label{sec:lra_accuracy_experiment}

Having isolated the single-block computational behavior, we next evaluate\footnote{All models are evaluated under the same PyTorch pipeline, using 4 sequence-mixing layers, model hidden dimension $d_{\mathrm{model}}=256$, eight heads, the AdamW optimizer \citep{loshchilov2019AdamW}, bfloat16 mixed precision, and a warmup-cosine learning-rate schedule.} representational capability on the Long Range Arena (LRA) benchmark \citep{tay2020LRA}. We focus on ListOps, a hierarchical reasoning task with sequence length $N=2\text{K}$, and byte-level IMDb/Text classification with sequence length $N=4\text{K}$.  We train each model for 20 epochs over 3 random seeds and report test accuracy at the epoch with the highest validation accuracy. For GMA, we use $K=128$ components in this experiment, matching the efficiency--capacity trade-off identified in the controlled systems profiling experiment.

\begin{table}[H]
    \centering
    \scriptsize
    \setlength{\tabcolsep}{4pt}
    \renewcommand{\arraystretch}{0.92}
    \caption{Task accuracy on the LRA ListOps and byte-level IMDb/Text tasks. Results are \textit{test accuracies} at the best validation epoch, reported as mean $\pm$ standard deviation over 3 random seeds. The parameter column reports millions of parameters for ListOps/Text, respectively; differences arise from task-specific components and, for some baselines, sequence-length dependent modules. \underline{Underlined} values denote the best overall result in each metric, while boldface highlights the GMA row.}
    \label{tab:lra_accuracy}
    \resizebox{0.8\textwidth}{!}{
    \begin{tabular}{lcccc}
        \toprule
        \textbf{Model} &
        \textbf{Params (L/T)} &
        \textbf{ListOps (2K)} &
        \textbf{Text (4K)} &
        \textbf{Avg. Acc.} \\
        \midrule
        Standard MHA (SDPA)
        & 3.74M / 4.25M
        & $41.33 \pm 0.49$
        & $64.03 \pm 0.89$
        & $52.68 \pm 0.48$ \\

        Linformer ($k=256$)
        & 7.84M / 12.44M
        & $39.15 \pm 0.13$
        & $58.82 \pm 1.60$
        & $48.99 \pm 0.86$ \\

        Linear Transformer
        & 3.74M / 4.25M
        & $40.03 \pm 0.70$
        & $64.88 \pm 0.13$
        & $52.46 \pm 0.29$ \\

        Performer ($M=256$)
        & 3.74M / 4.25M
        & $40.65 \pm 0.85$
        & $63.38 \pm 0.73$
        & $52.01 \pm 0.79$ \\

        Mamba
        & 4.44M / 4.95M
        & $\underline{42.28 \pm 1.75}$
        & $\underline{82.12 \pm 0.40}$
        & $\underline{62.20 \pm 0.68}$ \\
        \midrule
        \textbf{GMA (Ours, $K=128$)}
        & \textbf{4.01M / 4.52M}
        & $\mathbf{41.17 \pm 0.23}$
        & $\mathbf{64.97 \pm 0.21}$
        & $\mathbf{53.07 \pm 0.12}$ \\
        \bottomrule
    \end{tabular}
    }
\end{table}

Table~\ref{tab:lra_accuracy} shows that GMA is competitive with standard and efficient attention baselines under the same PyTorch LRA pipeline. On ListOps, GMA reaches $41.17\%$, closely matching SDPA attention ($41.33\%$) and outperforming Linformer, Linear Transformer, and Performer. On Text, GMA obtains $64.97\%$, the strongest result among the attention-style baselines, slightly above Linear Transformer ($64.88\%$) and SDPA attention ($64.03\%$). Averaged over the two tasks, GMA achieves $53.07\%$, giving the best average performance among the attention-based baselines evaluated here.

Mamba performs best overall, driven mainly by its much stronger Text result ($82.12\%$), indicating that selective state-space models provide a different and highly effective inductive bias for byte-level classification. Nevertheless, GMA outperforms Linformer, Linear Transformer, and Performer on both ListOps and Text in this experiment, and gives the strongest average accuracy among the attention-style baselines. These results support the hypothesis that learned probabilistic routing can be competitive with static or randomized efficient-attention mechanisms under a shared training pipeline, while retaining the fixed-$K$, linear-in-$N$ scaling and analyzable responsibility structure introduced in Sections~\ref{sec:linear_gma} and \ref{subsec:memory_and_compute_costs}.

\subsection{Language Modelling on WikiText-103}
\label{sec:wikitext_experiment}

Efficient attention mechanisms improve the scalability of Transformers on long sequences, but many do so by introducing sparsity patterns, low-rank projections, kernel feature maps, or other approximations to full softmax attention \citep{tay2022efficient,katharopoulos2020transformers,choromanski2021rethinking}. Prior work has observed that such approximations can involve a model-quality trade-off relative to exact softmax attention, especially in settings requiring dense token-level prediction \citep{dao2022flashattention,yang2024gated}. We therefore evaluate autoregressive language modelling on WikiText-103 \citep{merity2016pointer}.

Unlike the bidirectional LRA setting, this experiment requires strictly causal sequence mixing: the representation at position $i$ may depend only on tokens at positions $j\leq i$, and is then used for next-token prediction. For GMA, we use the causal prefix-sum formulation described in Section~\ref{sec:causal_gma}, which accumulates key-weighted latent memories and normalizers only over prefix positions. This prevents future-token leakage while preserving linear-in-$N$ complexity for fixed $K$.

All models are trained\footnote{We use a fixed-budget decoder-only setup with context length 1024, model hidden dimension $d_{\mathrm{model}}=512$, 6 layers, 8 heads, bfloat16 mixed precision, AdamW optimization, and an effective batch size of 16,384 tokens. Each model is trained for 20,000 optimization steps over 3 random seeds on an NVIDIA A100-SXM4-40GB GPU.} from scratch using a GPT-2 BPE tokenizer, with no pretrained model weights. We report validation perplexity at the best validation checkpoint\footnote{Since the best validation checkpoint occurred at the final training step for all runs, these results should be interpreted as a fixed-budget accuracy-efficiency comparison rather than as full-convergence performance.} and training throughput in tokens per second. All models passed a prefix-causality test, with maximum prefix error equal to zero.

\begin{table}[H]
    \centering
    \caption{Autoregressive language modelling on WikiText-103 under a fixed training budget. We report validation perplexity at the best validation checkpoint; lower is better. Throughput is measured in training tokens per second; higher is better. \underline{Underlined} values denote the best overall result in each metric, while boldface highlights the stronger GMA variant.}
    \label{tab:wikitext_results}
    \footnotesize
    \setlength{\tabcolsep}{5pt}
    \renewcommand{\arraystretch}{0.95}
    \begin{tabular}{lccc}
        \toprule
        \textbf{Model Architecture} &
        \textbf{Params} &
        \textbf{Val. PPL $\downarrow$} &
        \textbf{Tok/s $\uparrow$} \\
        \midrule
        Standard MHA (causal SDPA)
        & 45.17M
        & $\underline{35.40 \pm 0.06}$
        & $\underline{110{,}778 \pm 809}$ \\

        Linear Transformer
        & 45.17M
        & $44.81 \pm 0.11$
        & $41{,}693 \pm 47$ \\

        Performer
        & 45.17M
        & $47.49 \pm 0.23$
        & $14{,}559 \pm 6$ \\

        Mamba
        & 49.04M
        & $35.44 \pm 0.21$
        & $74{,}532 \pm 564$ \\
        \midrule

        \textbf{GMA (Ours, $K=128$)}
        & 45.96M
        & $42.62 \pm 0.38$
        & $22{,}912 \pm 23$ \\

        \textbf{GMA (Ours, $K=256$)}
        & 46.76M
        & $\mathbf{41.72 \pm 0.21}$
        & $13{,}056 \pm 21$ \\
        \bottomrule
    \end{tabular}
\end{table}

Table~\ref{tab:wikitext_results} shows that causal GMA improves perplexity over the tested linear-attention baselines on dense autoregressive prediction. GMA with $K=256$ obtains a validation perplexity of $41.72$, outperforming Linear Transformer ($44.81$) and Performer ($47.49$) under the same fixed training budget. This suggests that learned Gaussian responsibilities can preserve token-level predictive information more effectively than the fixed ELU feature map of Linear Transformer or the random feature approximation used by Performer in this setup.

However, causal GMA does not close the gap to optimized causal softmax attention or Mamba in this fixed-budget implementation. Standard MHA with causal SDPA achieves the best perplexity ($35.40$) and the highest throughput, while Mamba reaches a similar perplexity ($35.44$) with lower but still strong throughput. This does not contradict the linear-in-$N$ complexity analysis of GMA. The cost estimate in Eq.~\eqref{eq:gma_total_cost} describes asymptotic scaling for fixed $K$, whereas wall-clock throughput also depends on constant factors, kernel fusion, memory access patterns, and implementation maturity. Moreover, at the context length used here ($N=1024$), the difference between $NK$ and $N^2$ is still modest for $K=128$ or $K=256$ before constant factors are taken into account. The scaling advantage of GMA is expected to become more relevant at longer sequence lengths when $K \ll N$, although realizing this advantage in wall-clock time would require optimized kernels. In the present high-level, non-optimized PyTorch implementation, GMA incurs larger constant factors from Gaussian log-density computation, Mahalanobis terms, causal prefix accumulation, and responsibility normalization. Increasing the number of components from $K=128$ to $K=256$ improves perplexity from $42.62$ to $41.72$, but reduces throughput from $22.9$K to $13.1$K tokens/s---an expected capacity-efficiency trade-off.

Overall, the WikiText-103 experiment supports GMA as a probabilistic and interpretable linear-time attention-style alternative that improves over the tested linear and random-feature attention variants, while remaining less efficient and less accurate than highly optimized causal SDPA and Mamba under the current implementation.

\subsection{Interpretability: Latent Responsibility Structure}
\label{sec:gma_interpretability}

Motivated by the non-negative factorization structure of the GMA responsibility matrices, we analyze the learned query responsibilities on WikiText-103 validation sequences. We use the previously trained causal GMA model with $K=128$ and extract the last-layer query responsibility tensor $\Gamma^Q\in\mathbb{R}^{S\times H\times L\times K}$, where $S=64$ is the number of validation sequences, $H=8$ is the number of attention heads, $L=1024$ is the sequence length, and $K=128$ is the number of Gaussian components. For each token occurrence at sequence $s$ and position $\ell$, we average the responsibility vectors over heads,
\begin{equation}
    \bar{\gamma}^Q_{s,\ell,k}
    =
    \frac{1}{H}\sum_{h=1}^H \gamma^Q_{s,h,\ell,k},
    \qquad k=1,\ldots,K.
\end{equation}
This produces one head-averaged responsibility vector $\bar{\gamma}^Q_{s,\ell,:}\in\Delta^{K-1}$ for each token occurrence. We then flatten the sequence-position pair $(s,\ell)$ into a token index $t=1,\ldots,T$, where $T=SL=65{,}536$, and write the resulting vector as $\bar{\gamma}^Q_t$. 
Thus, the averaging is performed across heads for the same token occurrence, not across different sequence positions or across repeated token types. Because this diagnostic averages responsibilities across heads, the resulting component indices should be interpreted as aggregate, head-averaged routing channels rather than isolated per-head components.\footnote{If each head has its own independently learned Gaussian components, then component $k$ in one head is not necessarily aligned with component $k$ in another head. The head-averaged analysis is therefore a coarse diagnostic of aggregate routing structure; per-head component specialization is left for more detailed future analysis.}

We study both soft and hard responsibility structure. The marginal component usage is defined as
\begin{equation}
    p_k = \frac{1}{T}\sum_{t=1}^T \bar{\gamma}^Q_{t,k}.
\end{equation}
We use the normalized usage entropy
\begin{equation}
    H_{\mathrm{usage}}
    =
    \frac{-\sum_{k=1}^K p_k \log p_k}{\log K}
\end{equation}
to measure whether the latent routing channels are broadly used or collapse to a small subset. A value close to 1 indicates broad component usage, while a value close to 0 indicates concentration on a few components.

To measure routing sharpness, we compute the normalized mean token entropy
\begin{equation}
    H_{\mathrm{token}}
    =
    \frac{1}{T}
    \sum_{t=1}^T
    \frac{-\sum_{k=1}^K \bar{\gamma}^Q_{t,k}\log \bar{\gamma}^Q_{t,k}}{\log K}.
\end{equation}
Lower values indicate sharper, more confident routing. We also report the mean maximum responsibility
\begin{equation}
    R_{\max}
    =
    \frac{1}{T}\sum_{t=1}^T \max_k \bar{\gamma}^Q_{t,k},
\end{equation}
which gives a complementary measure of assignment sharpness.

For hard-assignment diagnostics, we define
\begin{equation}
    z_t = \arg\max_k \bar{\gamma}^Q_{t,k}.
\end{equation}
Let $c_t$ denote the surface-form category of token $t$, such as punctuation, numeric, capitalized, lower-case alphabetic, subword alphabetic, newline/space, or function word. For component $k$, let
\begin{equation}
    n_{k,c}=\sum_{t=1}^T \mathbf{1}\{z_t=k,c_t=c\},
    \qquad
    n_k=\sum_c n_{k,c}.
\end{equation}
For components with $n_k>0$, the component purity is
\begin{equation}
    \operatorname{Purity}(k)
    =
    \max_c \frac{n_{k,c}}{n_k},
\end{equation}
and the weighted category purity is
\begin{equation}
    \operatorname{Purity}_{\mathrm{w}}
    =
    \sum_{k:n_k>0}
    \frac{n_k}{T}\operatorname{Purity}(k).
\end{equation}
This measures how strongly hard GMA assignments align with surface-form token categories.

Finally, we compute the mutual information between hard component assignments $Z$ and token categories $C$:
\begin{equation}
    I(Z;C)
    =
    \sum_{k,c} p(k,c)
    \log
    \frac{p(k,c)}{p(k)p(c)},
\end{equation}
where $p(k,c)=n_{k,c}/T$, $p(k)=\sum_c p(k,c)$, and $p(c)=\sum_k p(k,c)$. We report normalized mutual information as
\begin{equation}
    \operatorname{NMI}(Z,C)
    =
    \frac{I(Z;C)}{\min\{H(Z),H(C)\}},
\end{equation}
where $H(Z)$ and $H(C)$ are the marginal entropies of component assignments and token categories, respectively. Higher NMI indicates stronger alignment between latent routing and token surface-form categories.

\begin{table}[H]
\centering
\caption{Interpretability statistics for last-layer, head-averaged GMA query responsibilities on WikiText-103 validation sequences. Usage entropy and token entropy are normalized by $\log K$. Purity, mutual information (MI), and normalized mutual information (NMI) are computed from hard assignments $z_t=\arg\max_k \bar{\gamma}^Q_{t,k}$ and surface-form token categories.}
\label{tab:gma_interpretability}
\scriptsize
\setlength{\tabcolsep}{2.5pt}
\renewcommand{\arraystretch}{0.95}
\begin{tabular}{lcccccccc}
\toprule
Model &
$K$ &
Active &
Usage Ent. &
Token Ent. &
Mean Max Resp. &
Purity &
MI &
NMI \\
\midrule
GMA &
128 &
104/128 &
0.933 &
0.790 &
0.103 &
0.483 &
0.427 &
0.244 \\
\bottomrule
\end{tabular}
\end{table}

Table~\ref{tab:gma_interpretability} shows that the learned responsibilities use most of the available latent routing channels. In total, 104 out of 128 components receive at least one hard assignment, and the normalized usage entropy is 0.933, indicating that the soft usage distribution is broad rather than collapsed. At the same time, the normalized per-token entropy is 0.790 and the mean maximum responsibility is 0.103, showing that head-averaged routing remains soft rather than nearly one-hot. This softness is expected because responsibilities are averaged across heads.

The hard component assignments nevertheless carry measurable surface-form information. The weighted category purity is $0.483$, substantially above a permutation baseline of $0.293\pm 0.001$ that preserves the component and category marginals while randomizing their association\footnote{The weighted category purity is computed as $\operatorname{Purity}_{\mathrm{w}} =\sum_{k:n_k>0}(n_k/T)\max_c n_{k,c}/n_k$, giving $0.483$ for the observed hard assignments. The permutation baseline is obtained by randomly permuting the surface-form category labels $\{c_t\}_{t=1}^T$ relative to the hard component assignments $\{z_t\}_{t=1}^T$, recomputing $\operatorname{Purity}_{\mathrm{w}}$ for each permutation, and reporting the mean and standard deviation across permutations. This preserves the empirical component counts and category counts while destroying their association. The global majority-category baseline is $\max_c n_c/T=18{,}757/65{,}536\approx 0.286$, corresponding to lower-case alphabetic tokens.}. For reference, the largest global token category, lower-case alphabetic tokens, accounts for $0.286$ of the analyzed token set. The mutual information between component assignments and token categories is 0.427 nats, corresponding to a normalized mutual information of 0.244. Several high-usage components show preferences for recognizable token roles, including function words, punctuation, lower-case alphabetic tokens, numeric tokens, capitalized tokens, and subword pieces.

\begin{figure}[H]
    \centering
    \includegraphics[width=0.98\textwidth]{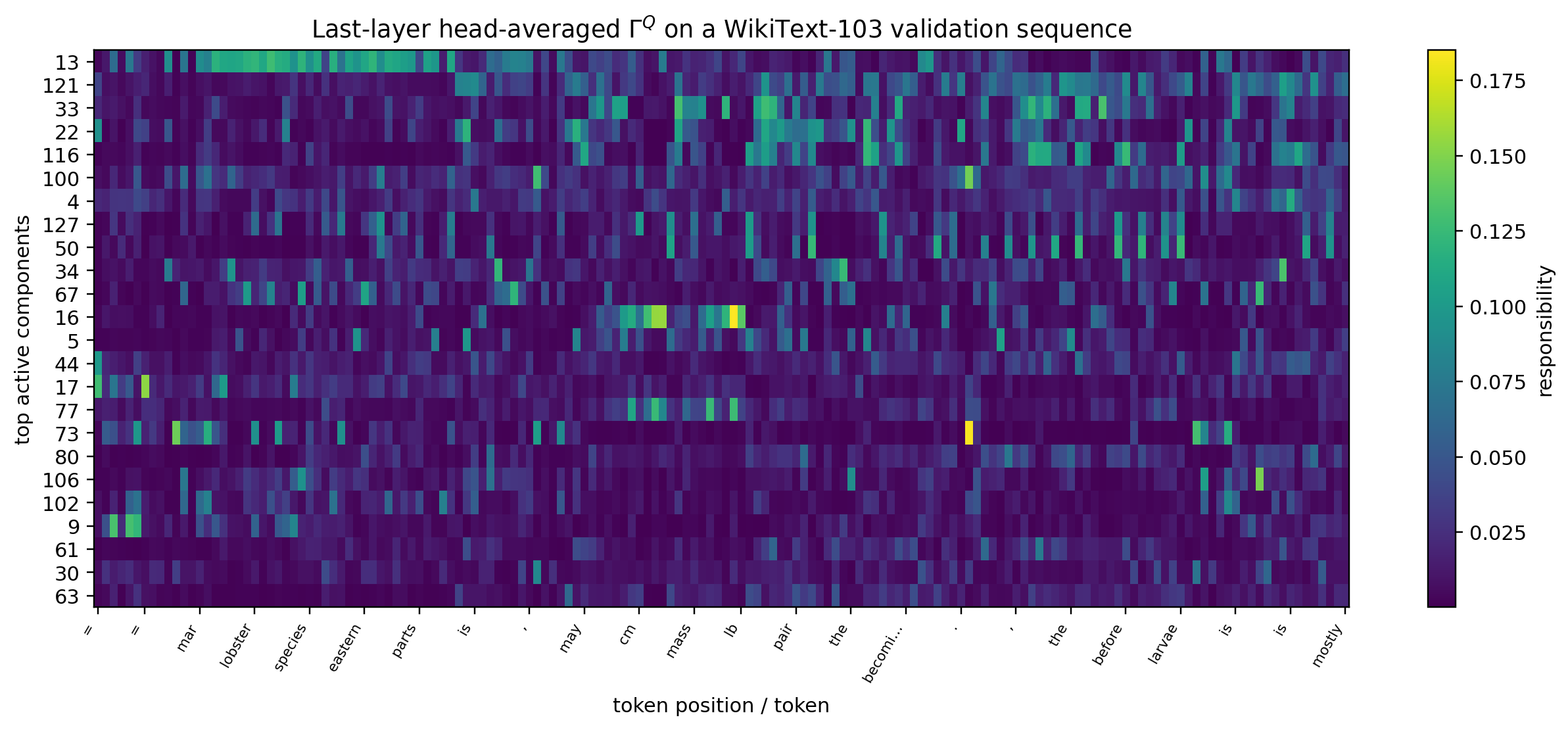}
    \caption{Last-layer, head-averaged GMA query responsibilities on a WikiText-103 validation sequence. Rows show the most active latent routing channels in the sequence and columns show token positions. Concentrated bands indicate that different token positions route through distinct latent components.}
    \label{fig:gma_gammaq_heatmap}
\end{figure}

Figure~\ref{fig:gma_gammaq_heatmap} visualizes the learned responsibility structure at the token level. The heatmap shows that different token positions activate different subsets of latent routing channels, rather than routing uniformly through all components. This provides qualitative evidence that the GMA responsibility matrix is an analyzable latent routing object.

\begin{figure}[H]
    \centering
    \includegraphics[width=0.82\textwidth]{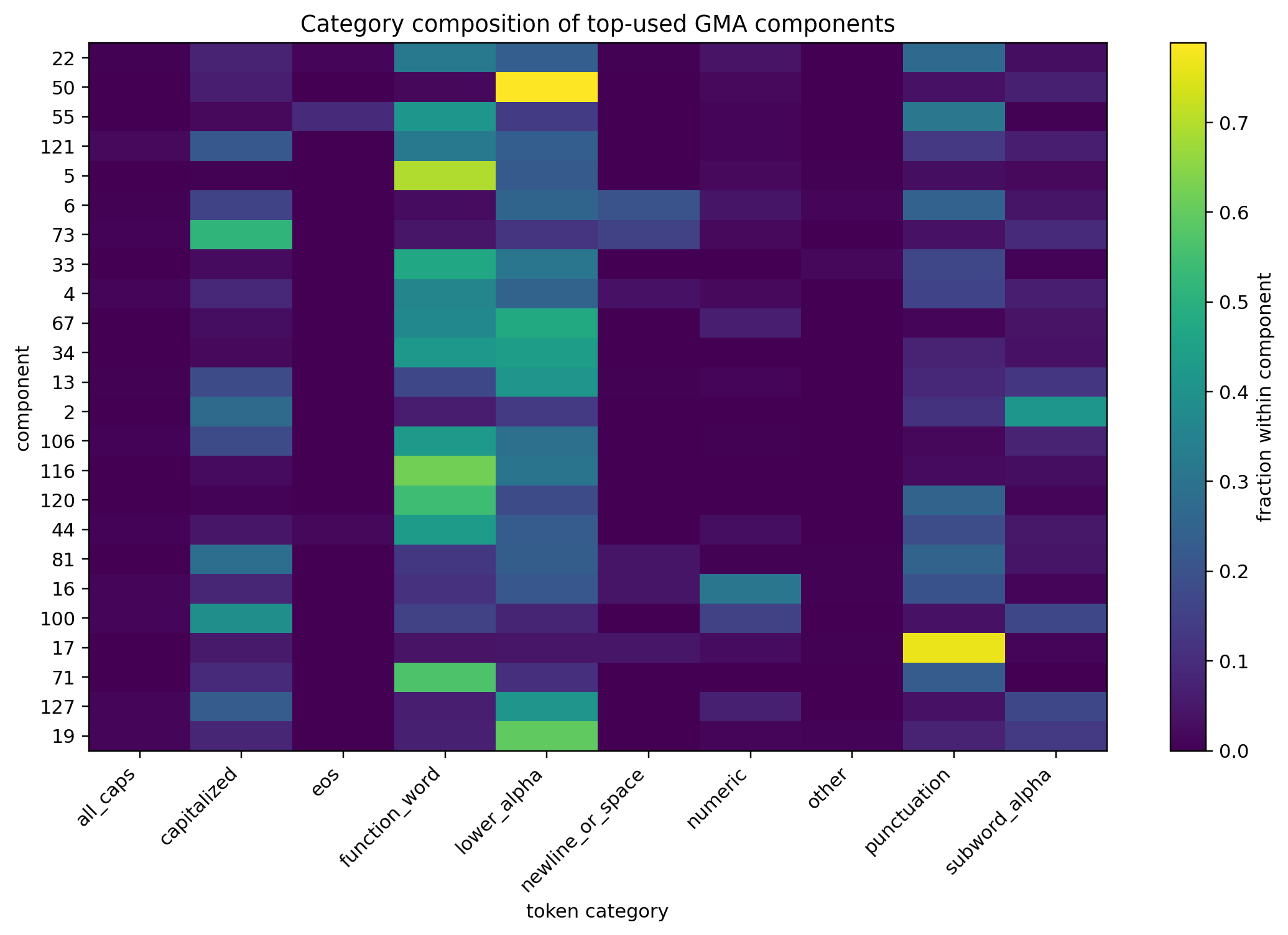}
    \caption{Category composition of the most-used GMA components under hard assignments $z_t=\arg\max_k \bar{\gamma}^Q_{t,k}$. Rows correspond to latent routing channels and columns correspond to surface-form token categories. Several components specialize toward categories such as function words, punctuation, capitalized tokens, lower-case alphabetic tokens, and subword fragments.}
    \label{fig:gma_category_heatmap}
\end{figure}

Figure~\ref{fig:gma_category_heatmap} further shows that high-usage components are not purely random with respect to surface-form categories: several components have visible preferences for categories such as function words, punctuation, capitalized tokens, lower-case alphabetic tokens, and subword fragments. This provides evidence that the learned responsibility matrices are analyzable latent routing objects. Additional diagnostics, including the token-category distribution, representative component examples, permutation baselines, component-usage histograms, token-level assignment strips, and PCA projections, are provided in Appendix~\ref{app:gma_interpretability}.

However, the components should not be interpreted as clean semantic classes. Many components are mixed, and the categories used here are surface-form categories rather than full syntactic or semantic annotations\footnote{Additional diagnostic plots in Appendix~\ref{app:gma_interpretability}, including component-usage histograms, purity bar plots, token-level assignment strips, and PCA projections of responsibility vectors, show partial structure with substantial overlap rather than fully separated clusters.}. We therefore conclude that GMA learns a non-collapsed and analyzable latent responsibility structure with moderate alignment to token surface roles, rather than a set of fully disentangled semantic concepts.

\section{Discussion}
\label{sec:discussion}

We position GMA as a new attention-style mechanism, not as a universally superior replacement for existing attention or state-space models. Its central contribution is to change both the \emph{geometry} and the \emph{computational ordering} of attention. At its most abstract level, an attention score measures the \textit{affinity} or \textit{compatibility} between two representations. This affinity need not be a dot product: it may be defined by cosine similarity, learned additive scores, kernels, distances, divergences, or other task-adapted similarity functions. From this perspective, GMA can be seen as a responsibility-space attention mechanism. 
Rather than comparing projected queries and keys directly in the original routing representation space, GMA first maps them into posterior responsibility vectors over a shared set of learned Gaussian mixture components, and then defines their unnormalized affinity by the overlap
(Eq.~\eqref{eq:gma_induced_attention_unnormalised})
\[
    \widetilde A^{\mathrm{GMA}}_{ij}
    =
    \langle \gamma^Q_i,\gamma^K_j\rangle
    =
    \sum_{k=1}^K \gamma^Q_{i,k}\gamma^K_{j,k}.
\]
The expressiveness of this affinity depends not only on the number of mixture components $K$, but also on the routing dimension, covariance parameterization, projection matrices, number of heads, depth, and training objective. We discuss this broader view of attention design in Appendix~\ref{app:attention_similarity_distance}.

This responsibility-space affinity alone would not necessarily reduce the computational cost: explicitly materializing $\widetilde A^{\mathrm{GMA}}=\Gamma^Q(\Gamma^K)^\top$ would still produce an $N\times N$ matrix. The efficiency of GMA comes from using the associativity of matrix multiplication to change the order of computation. Instead of computing
\[
    \bigl(\Gamma^Q(\Gamma^K)^\top\bigr)V_X,
\]
GMA first writes values into a $K$-slot latent memory,
\[
    \tilde V=(\Gamma^K)^\top V_X,
\]
and then reads from this memory using query responsibilities,
\[
    O
    =
    \frac{\Gamma^Q\tilde V}{\Gamma^QZ+\epsilon}.
\]
Thus, GMA is best understood as \emph{responsibility-space affinity plus latent-memory routing}: the induced token-to-token affinity exists algebraically, but it is not materialized in the efficient implementation. This is the mechanism that gives GMA its fixed-$K$, linear-in-$N$ activation scaling while retaining probabilistic responsibilities that can be inspected and analyzed.

\paragraph{Scaling behavior and implementation constants.}
The controlled systems profiling experiment confirms this intended scaling behavior. For fixed $K$, the routing-specific responsibility tensors scale as $\mathcal{O}(BHNK)$ rather than $\mathcal{O}(BHN^2)$, and the observed memory growth is approximately linear in sequence length. This supports the cost analysis in Section~\ref{subsec:memory_and_compute_costs}. At the same time, the profiling results also show that asymptotic scaling does not directly imply superior wall-clock efficiency at the sequence lengths tested here. The current high-level PyTorch implementation has larger constant factors than optimized SDPA, Linear Transformer, Performer, and Mamba implementations. This is not surprising: GMA explicitly evaluates Gaussian log densities, Mahalanobis terms, posterior responsibilities, and normalizers, whereas several baselines benefit from simpler arithmetic, fused kernels, or highly optimized recurrent or state-space operators. Thus, the empirical message is twofold: GMA has the intended fixed-$K$ linear memory structure, but realizing its full speed potential will likely require specialized kernels and implementation-level optimization.

\paragraph{Task accuracy and inductive bias.}
The LRA results suggest that the additional cost of probabilistic routing can come with useful representational capacity. On ListOps and byte-level Text classification, GMA is competitive with standard SDPA attention and gives the strongest average performance among the attention-style baselines evaluated in our PyTorch pipeline. This indicates that the $N\times K$ latent bottleneck does not merely compress away useful information; learned responsibilities can preserve task-relevant sequence structure while avoiding an explicit $N\times N$ attention matrix. However, the comparison with Mamba is also important. Mamba is substantially stronger on the byte-level Text task, suggesting that selective state-space dynamics provide a different and highly effective inductive bias for this setting. Therefore, GMA offers a competitive probabilistic attention-style alternative with a distinct routing mechanism and diagnostic structure.

\paragraph{Autoregressive modelling.}
The WikiText-103 experiment gives a more demanding test of causal GMA. Causal GMA improves over the Linear Transformer and Performer baselines in validation perplexity under the same training budget, which suggests that learned Gaussian responsibilities can preserve more predictive information than the tested fixed or randomized feature-map approximations in this setup. Nevertheless, GMA remains behind optimized causal SDPA and Mamba in both perplexity and throughput. Increasing the number of mixture components from $K=128$ to $K=256$ improves perplexity but reduces throughput, exposing the expected capacity--efficiency trade-off. These results clarify the present role of GMA in language modelling: its immediate strength is not raw speed or state-of-the-art perplexity, but the combination of linear scaling, probabilistic routing, and analyzable latent responsibilities. Longer contexts, optimized kernels, and hybrid local--global designs are natural settings in which this trade-off should be tested more thoroughly.

\paragraph{Interpretability and responsibility diagnostics.}
The interpretability experiment directly examines one of the main advantages of a responsibility-based mechanism. The last-layer, head-averaged query responsibilities on WikiText-103 use most of the available latent routing channels: 104 out of 128 components receive at least one hard assignment, and the normalized usage entropy is 0.933. This indicates broad component usage rather than collapse to a small subset of mixture components. The hard assignments also carry measurable surface-form information, with weighted category purity 0.483 and normalized mutual information 0.244 relative to the token categories used in the analysis. These numbers should be interpreted carefully. The categories are surface-form categories derived from GPT-2 BPE tokens, not full syntactic or semantic annotations, and many components remain mixed. Moreover, because the diagnostic averages responsibilities over heads, the resulting components should be read as head-averaged routing channels rather than isolated per-head experts. Even with these qualifications, the analysis shows that GMA exposes a latent routing structure that can be quantified, visualized, and compared against external token annotations. This kind of diagnostic access is less direct in many random-feature, low-rank, or state-space alternatives.

\paragraph{Limitations.}
Several limitations remain. First, the present implementation is not optimized at the kernel level, so its wall-clock throughput underestimates what may be possible with fused Gaussian-responsibility computation and more efficient prefix accumulation. Second, all experiments in this paper use finite-$K$ GMA; the number of components is selected manually rather than adapted automatically to the data or task. Third, the experiments focus on self-attention-style sequence mixing. Cross-attention follows naturally from the same formulation, as discussed in Appendix~\ref{app:self_cross_gma}, but it is not empirically evaluated here. Fourth, the interpretability analysis is intentionally limited: it uses surface-form token categories and head-averaged responsibilities, so it does not establish that GMA components correspond to clean semantic, syntactic, or discourse-level concepts. Finally, although GMA induces a constrained non-negative low-rank affinity structure, this does not guarantee component specialization or disentanglement; these properties must be measured
empirically.

\subsection*{\textit{GMA Extensions: Adaptive Mixtures and Probabilistic Expert Routing}}
\label{sec:future_extensions}

The experiments in this paper evaluate GMA as a finite-$K$ linear-time routing mechanism with learned mixture priors, means, and diagonal covariances. The results suggest several directions for future work: implementation-level optimization, adaptive mixture capacity, hybrid local--global attention, cross-attention and multimodal applications, and probabilistic expert routing.

\paragraph{Optimized and hybrid GMA.}
A first practical direction is to develop fused kernels for Gaussian log-density evaluation, responsibility normalization, and causal prefix accumulation. The current implementation is useful for validating the method and analyzing its behavior, but it does not exploit the same level of hardware specialization as SDPA or state-space implementations. A second direction is to combine exact local attention with long-range GMA routing. Such a hybrid architecture could use standard attention to capture short-range syntax, local byte-level patterns, or fine-grained visual details, while using mixture routing to summarize and retrieve long-range information through a fixed number of latent memory channels. This may be particularly relevant for long-context language modelling, document understanding, genomics, and multimodal settings where local detail and global context are both important.

\paragraph{Bayesian and nonparametric GMA.}
In the finite GMA formulation used in this paper, the number of latent Gaussian components $K$ is fixed, and the mixture prior $\pi=(\pi_1,\ldots,\pi_K)$ is learned through unconstrained logits followed by a softmax, as in Eq.~\eqref{eq:gma_prior_reparameterization}. A Bayesian extension would place a \textit{Dirichlet} prior over the mixture weights,
\begin{equation}
\label{eq:bayesian_gma_dirichlet_prior}
    \pi
    \sim
    \operatorname{Dirichlet}
    \left(
    \frac{\alpha_0}{K},\ldots,\frac{\alpha_0}{K}
    \right),
\end{equation}
where $\alpha_0>0$ controls how concentrated or diffuse the component usage is. Small values of $\alpha_0$ encourage sparse component usage, while larger values encourage more uniform use of the available mixture components.

For a generic routing vector $x_i$, such as a query vector $q_i=(Q_X)_{i,:}\in\mathbb{R}^{d_r}$ or a key vector $k_i=(K_X)_{i,:}\in\mathbb{R}^{d_r}$, the responsibility retains the same form (Eq.\eqref{eq:gma_responsibility}),
\begin{equation}
\label{eq:bayesian_gma_responsibility}
    \gamma_{i,k}
    =
    \frac{
    \pi_k \mathcal{N}(x_i\mid \mu_k,\Sigma_k)
    }{
    \sum_{\ell=1}^{K}
    \pi_\ell \mathcal{N}(x_i\mid \mu_\ell,\Sigma_\ell)
    },
\end{equation}
but $\pi$ is treated as a random variable or regularized latent parameter rather than only as a point estimate. 
A finite Bayesian variant could be implemented by learning a variational Dirichlet posterior $q(\pi)=\operatorname{Dirichlet}(a)$ and adding the KL penalty
\begin{equation}
\label{eq:bayesian_gma_dirichlet_kl}
    \operatorname{KL}
    \left[
    q(\pi)
    \,\|\,
    \operatorname{Dirichlet}
    \left(
    \frac{\alpha_0}{K},\ldots,\frac{\alpha_0}{K}
    \right)
    \right]
\end{equation}
to the task loss $\mathcal{L}_{\mathrm{task}}$. The forward pass could use samples from $q(\pi)$ to propagate uncertainty over mixture weights, or use the posterior mean $\mathbb{E}_{q}[\pi_k]=a_k/\sum_{\ell}a_\ell$ as a deterministic approximation. This would provide a route toward representing uncertainty over mixture weights, rather than only learning point-valued priors.

A further extension is to replace the finite Dirichlet prior with a \textit{Dirichlet process} prior, yielding a nonparametric or truncated infinite-mixture variant of GMA \citep{ferguson1973bayesian,sethuraman1994constructive, rasmussen2000infinite,blei2006variational}. In stick-breaking form,
\begin{equation}
\label{eq:dp_gma_stick_breaking}
    v_k \sim \operatorname{Beta}(1,\alpha),
    \qquad
    \pi_k = v_k\prod_{\ell<k}(1-v_\ell),
\end{equation}
where $\alpha$ is the concentration parameter. 
In practice, such a model would be implemented using a finite truncation at a maximum size $K_{\max}$. The nominal routing cost would then scale as $\mathcal{O}(NK_{\max})$. The effective number of active components could be smaller if many learned stick weights become negligible, although actual wall-clock savings would require pruning, sparse evaluation, or another mechanism that skips negligible components. This would allow GMA to adapt its effective routing capacity to the data and task, rather than relying only on manual selection of $K$.

\paragraph{Cross-attention and multimodal routing.}
Although the experiments in this work focus on self-attention-style sequence mixing, the same responsibility-space mechanism extends naturally to cross-attention. Queries can be formed from a query-side sequence $X$, while keys and values are formed from a separate context sequence $Y$, producing $\Gamma_X^Q$ and $\Gamma_Y^K$ as described in Appendix~\ref{app:self_cross_gma}. This suggests possible applications in encoder--decoder architectures, retrieval-augmented models, vision--language models, and multimodal systems in which the query and context streams have different statistical structure. In such settings, the mixture components may serve as shared latent routing channels between modalities or between source and target representations. A systematic study of GMA cross-attention is therefore a natural next step.

\paragraph{Probabilistic expert routing.}
A second direction is to use Gaussian responsibilities as probabilistic gates over attention heads, sequence-mixing modules, or expert networks. Let $r_i\in\mathbb{R}^{d_g}$ denote a gating representation of token $i$, and suppose that each expert or head $h=1,\ldots,H$ is associated with a Gaussian routing component with parameters $(\rho_h,\nu_h,\Lambda_h)$, where $\rho_h\geq 0$, $\sum_{h=1}^H\rho_h=1$, $\nu_h\in\mathbb{R}^{d_g}$, and $\Lambda_h\in\mathbb{S}^{d_g}_{++}$. A GMA-style gate could define
\begin{equation}
\label{eq:gma_moe_responsibility}
    \gamma_{i,h}^{\mathrm{MoE}}
    =
    \frac{
    \rho_h\mathcal{N}(r_i\mid\nu_h,\Lambda_h)
    }{
    \sum_{\ell=1}^{H}
    \rho_\ell\mathcal{N}(r_i\mid\nu_\ell,\Lambda_\ell)
    }.
\end{equation}
Here $\gamma_{i,h}^{\mathrm{MoE}}$ is the posterior responsibility of expert $h$ for token $i$ under the learned Gaussian gating model.

If all experts are evaluated, the output can be written as the dense mixture
\begin{equation}
\label{eq:gma_moe_dense_output}
    y_i
    =
    \sum_{h=1}^{H}
    \gamma_{i,h}^{\mathrm{MoE}}
    \operatorname{Expert}_h(X)_i,
\end{equation}
where $\operatorname{Expert}_h(X)_i$ denotes the output of expert, head, or sequence-mixing module $h$ at token position $i$. This dense form is fully differentiable, but it is not computationally sparse as all experts are executed. To obtain conditional-computation savings, one would need an additional sparse selection rule, such as top-$r$ routing, stochastic sampling, or a differentiable relaxation.

For example, let $\mathcal{S}_i\subset\{1,\ldots,H\}$ denote the selected expert set for token $i$, obtained by top-$r$ selection or sampling from $\gamma_i^{\mathrm{MoE}}$. A sparse GMA-MoE output can then be written as
\begin{equation}
\label{eq:gma_moe_sparse_output}
    y_i
    =
    \sum_{h\in\mathcal{S}_i}
    \bar{\gamma}_{i,h}^{\mathrm{MoE}}
    \operatorname{Expert}_h(X)_i,
    \qquad
    \bar{\gamma}_{i,h}^{\mathrm{MoE}}
    =
    \frac{
    \gamma_{i,h}^{\mathrm{MoE}}
    }{
    \sum_{\ell\in\mathcal{S}_i}
    \gamma_{i,\ell}^{\mathrm{MoE}}
    }.
\end{equation}
The discrete expert-selection step could be implemented using an $\operatorname{argmax}$ or top-$r$ operator, stochastic sampling, the Gumbel--Softmax or Concrete relaxation \citep{jang2017categorical,maddison2017concrete}, or a straight-through variant \citep{bengio2013estimating,jang2017categorical,huang2026NBSR}. This formulation is related to conventional MoE routing \citep{shazeer2017outrageously,lepikhin2020gshard,fedus2022switch}, but differs in its probabilistic interpretation. Standard MoE gates are often implemented as learned softmax or noisy top-$k$ classifiers over experts. A GMA-style gate instead defines expert probabilities as posterior responsibilities under a learned latent-density model. This could provide a principled way to combine conditional computation with uncertainty-aware and interpretable expert selection, while preserving the possibility of sparse execution through top-$r$ or sampled routing.

Taken together, these directions suggest that GMA is best viewed not only as a single efficient-attention module, but also as an instance of a broader design principle: attention mechanisms can be designed by choosing the space in which compatibility is measured, the normalization used to turn compatibility into routing weights, and the computational ordering used to retrieve values. The present paper studies one concrete realization of that principle using finite Gaussian mixtures and responsibility-space routing. In its current form, GMA is therefore \emph{\textit{not} all you need}: it is a probabilistic, interpretable, linear-time attention-style alternative whose value is clearest when latent routing structure, analyzability, and fixed-$K$ scaling matter alongside task accuracy and implementation efficiency.

\section{Conclusion}
\label{sec:conclusion}

We introduced \textit{Gaussian Mixture Attention} (GMA), a probabilistic attention-style sequence mixer that rethinks attention as \textit{routing through a learned latent responsibility space}. Instead of comparing every query directly with every key, GMA maps queries and keys to posterior responsibility vectors over $K$ learned Gaussian mixture components. Their overlap defines an implicit responsibility-space \textit{affinity}, and an efficient implementation avoids materializing the corresponding $N\times N$ matrix by first writing values into a $K$-slot latent memory and then reading from that memory with query responsibilities. This combination of responsibility-space affinity and associative latent-memory routing gives GMA fixed-$K$ linear-in-$N$ activation scaling while preserving a normalized attention-style interpretation.

Methodologically, GMA opens a complementary route for attention design: rather than modifying only sparsity patterns, low-rank projections, kernel feature maps, or implementation kernels, it changes the space in which \textit{compatibility} is computed. The resulting responsibility matrices are not merely computational intermediates; they are probabilistic routing objects that can be analyzed, visualized, and compared with token-level annotations. We developed bidirectional and causal variants of GMA, gave an end-to-end differentiable parameterization of the Gaussian mixture components, and analyzed its gradient structure, constrained non-negative low-rank affinity interpretation, and local routing stability.

Empirically, the results support a balanced view of GMA. Controlled systems profiling confirms the intended approximately \textit{linear memory growth with sequence length} for fixed $K$, but also shows that the present high-level PyTorch implementation has larger constant factors than optimized SDPA, Linear Transformer, Performer, and Mamba implementations. On LRA ListOps and byte-level Text classification tasks, GMA is competitive with standard attention and achieves the strongest average performance among the attention-style baselines evaluated in our pipeline. On WikiText-103, causal GMA improves over Linear Transformer and Performer in validation perplexity, but remains behind optimized causal SDPA and Mamba in both perplexity and throughput. The interpretability analysis further shows broad, non-collapsed component usage and moderate alignment between hard responsibility assignments and surface-form token categories.

These findings position GMA \textit{not} as a universal replacement for optimized softmax attention or state-space sequence models, but as a probabilistic, interpretable, fixed-$K$ linear-time attention-style alternative. Its advantages are clearest when fixed-$K$ scaling, normalized latent routing, and access to analyzable responsibility matrices are valuable. Its present limitations are also clear: the implementation is not yet kernel-optimized, the number of components is manually selected, the experiments focus on self-attention-style settings, and the current interpretability analysis does not establish clean semantic or syntactic disentanglement. Future work may pursue fused and hardware-aware GMA kernels, hybrid local--global architectures, adaptive Bayesian or nonparametric component selection, cross-attention and multimodal extensions, probabilistic expert-routing variants, and richer linguistic or task-specific analyses of learned mixture components. In this sense, GMA is \textit{not} ``all you need'', but it opens a complementary route for designing attention mechanisms through latent probabilistic routing spaces.

\section*{Author Contributions}

\textit{Dr. Yongchao Huang} initiated the idea of Gaussian Mixture Attention, formulated the methodological and theoretical framework, developed the mathematical derivations, co-designed and ran the experimental codebase, and drafted the manuscript. \textit{Hassan Raza} contributed to the empirical evaluation by co-designing and running the Long Range Arena (LRA) and WikiText-103 experiments. Both authors reviewed and approved the final manuscript. Despite careful review, some errors or inaccuracies may remain; readers are therefore advised to exercise due caution when consulting the material.

Code availability: \href{https://github.com/YongchaoHuang/Gaussian_mixture_attention}{GitHub}

\bibliographystyle{plain}
\bibliography{reference}

\appendix

\section{Notation for Gaussian Mixture Attention} \label{app:gma_notation} 

Table~\ref{tab:gma_notation} summarizes the notation used in the GMA methodology, theoretical analysis, and experiments.

\begin{table}[p]
\centering
\caption{Notation used in the GMA methodology and analysis.}
\label{tab:gma_notation}
\tiny
\setlength{\tabcolsep}{4pt}
\renewcommand{\arraystretch}{1.08}
\begin{tabular}{@{}lp{0.72\linewidth}@{}}
\toprule
\textbf{Symbol} & \textbf{Meaning} \\
\midrule

\multicolumn{2}{@{}l}{\textit{Sequence, projection, and standard-attention notation}} \\
\addlinespace[2pt]

$B$ & Batch size, used in systems profiling and activation-scaling discussions. \\

$H$ & Number of attention heads, used when discussing multi-head activation scaling. \\

$N$ & Sequence length, i.e., the number of token positions. \\

$d_{\mathrm{model}}$ & Model hidden dimension of each input token representation before query, key, and value projection. \\

$d_k$ & Query/key channel dimension in standard scaled dot-product attention. \\

$d_r$ & GMA routing dimension used to compute Gaussian-mixture responsibilities. \\

$d_v$ & Value dimension of the projected value vectors. \\

$X\in\mathbb{R}^{N\times d_{\mathrm{model}}}$ 
& Input sequence representations to a GMA layer in the self-attention setting. \\

$Q\in\mathbb{R}^{N\times d_k}$ 
& Query matrix in standard scaled dot-product attention. \\

$K_{\mathrm{att}}\in\mathbb{R}^{N\times d_k}$ 
& Key matrix in standard scaled dot-product attention; the subscript distinguishes it from the scalar $K$ used for the number of GMA mixture components. \\

$V\in\mathbb{R}^{N\times d_v}$ 
& Value matrix in standard scaled dot-product attention. \\

$W_Q,W_K\in\mathbb{R}^{d_{\mathrm{model}}\times d_r}$ 
& GMA query and key projection matrices. \\

$W_V\in\mathbb{R}^{d_{\mathrm{model}}\times d_v}$ 
& GMA value projection matrix. \\

$Q_X=XW_Q\in\mathbb{R}^{N\times d_r}$ 
& Query projection of the input sequence in GMA. \\

$K_X=XW_K\in\mathbb{R}^{N\times d_r}$ 
& Key projection of the input sequence in GMA. \\

$V_X=XW_V\in\mathbb{R}^{N\times d_v}$ 
& Value projection of the input sequence in GMA. \\

$q_i=(Q_X)_{i,:}$ 
& Query routing vector at token position $i$. \\

$k_i=(K_X)_{i,:}$ 
& Key routing vector at token position $i$. \\

$x_i\in\mathbb{R}^{d_r}$ 
& Generic routing vector, used when a formula applies to either $q_i$ or $k_i$. \\

\midrule
\multicolumn{2}{@{}l}{\textit{Gaussian-mixture routing notation}} \\
\addlinespace[2pt]

$K$ & Number of latent Gaussian mixture components in GMA. \\

$\mu_k\in\mathbb{R}^{d_r}$ 
& Mean vector of the $k$-th Gaussian routing component. \\

$\Sigma_k\in\mathbb{S}_{++}^{d_r}$ 
& Covariance matrix of the $k$-th Gaussian component. In the implemented GMA layer, $\Sigma_k$ is diagonal. \\

$\Sigma_k=\operatorname{diag}(\sigma^2_{k,1},\ldots,\sigma^2_{k,d_r})$ 
& Diagonal covariance parameterization used for computational tractability. \\

$\sigma^2_{k,m}$ 
& Diagonal variance of component $k$ in routing coordinate $m$. \\

$\pi_k$ 
& Mixture prior probability of the $k$-th Gaussian component, with $\sum_{k=1}^K\pi_k=1$. \\

$z_i\in\{1,\ldots,K\}$ 
& Latent component-assignment variable for routing vector $x_i$. \\

$\gamma_{i,k}$ 
& Responsibility of component $k$ for generic routing vector $x_i$, i.e. $p(z_i=k\mid x_i)$. \\

$\gamma^Q_{i,k}$ 
& Responsibility of component $k$ for query routing vector $q_i$. \\

$\gamma^K_{i,k}$ 
& Responsibility of component $k$ for key routing vector $k_i$. \\

$\Gamma\in\mathbb{R}^{N\times K}$ 
& Generic responsibility matrix for a sequence of routing vectors. \\

$\Gamma^Q\in\mathbb{R}^{N\times K}$ 
& Query responsibility matrix computed from $Q_X$. \\

$\Gamma^K\in\mathbb{R}^{N\times K}$ 
& Key responsibility matrix computed from $K_X$; the superscript $K$ denotes ``key'', not the number of mixture components. \\

$\omega\in\mathbb{R}^{K\times d_r}$ 
& Unconstrained covariance parameters used in the softplus reparameterization of diagonal variances. \\

$\alpha\in\mathbb{R}^{K}$ 
& Unconstrained logits used to parameterize the mixture priors $\pi_k$ by a softmax. \\

$\Theta_{\mathrm{GMA}}$ 
& Trainable parameters of a GMA layer, e.g. $\{W_Q,W_K,W_V,\mu,\omega,\alpha\}$. \\

$\epsilon_\sigma>0$ 
& Numerical stabilizer used to lower-bound diagonal covariance entries. \\

\midrule
\multicolumn{2}{@{}l}{\textit{Latent-memory routing and induced attention}} \\
\addlinespace[2pt]

$\mathbf{1}_N$ 
& All-ones vector in $\mathbb{R}^{N}$. \\

$\tilde V\in\mathbb{R}^{K\times d_v}$ 
& Latent memory obtained by aggregating $V_X$ with key responsibilities: $\tilde V=(\Gamma^K)^\top V_X$. \\

$\tilde V_k\in\mathbb{R}^{d_v}$ 
& $k$-th latent memory slot, i.e. the $k$-th row of $\tilde V$. \\

$Z\in\mathbb{R}^{K}$ 
& Component-wise key-responsibility normalizer: $Z=(\Gamma^K)^\top\mathbf{1}_N$. \\

$Z_k$ 
& Total key-responsibility mass assigned to component $k$. \\

$O\in\mathbb{R}^{N\times d_v}$ 
& Output sequence representation produced by GMA. \\

$\widetilde A^{\mathrm{GMA}}\in\mathbb{R}^{N\times N}$ 
& Implicit unnormalized GMA affinity matrix, $\widetilde A^{\mathrm{GMA}}=\Gamma^Q(\Gamma^K)^\top$; used for analysis but not materialized in the efficient implementation. \\

$A^{\mathrm{GMA}}\in\mathbb{R}^{N\times N}$ 
& Implicit normalized GMA affinity matrix induced by row-wise normalization of $\widetilde A^{\mathrm{GMA}}$; not materialized in the efficient implementation. \\

$\epsilon>0$ 
& Numerical stabilizer used in normalized latent-memory routing. \\

\midrule
\multicolumn{2}{@{}l}{\textit{Causal GMA notation}} \\
\addlinespace[2pt]

$\tilde V^{(i)}_k\in\mathbb{R}^{d_v}$ 
& Prefix-restricted latent memory slot for component $k$ at autoregressive position $i$, using only positions $j\leq i$. \\

$Z^{(i)}_k$ 
& Prefix-restricted key-responsibility normalizer for component $k$ at autoregressive position $i$. \\

$A^{\mathrm{cGMA}}\in\mathbb{R}^{N\times N}$ 
& Causal GMA affinity matrix supported only on prefix positions $j\leq i$. \\

\midrule
\multicolumn{2}{@{}l}{\textit{Gradient and stability analysis}} \\
\addlinespace[2pt]

$\mathcal{L}$ 
& Downstream task loss used in gradient-flow analysis. \\

$s_{i,k}$ 
& Pre-normalized log-density score of routing vector $x_i$ under Gaussian component $k$. \\

$g_{i,j}=\partial\mathcal{L}/\partial\gamma_{i,j}$ 
& Upstream gradient arriving at responsibility $\gamma_{i,j}$. \\

$\theta$ 
& Generic learnable scalar parameter entering the score $s_{i,k}$, e.g. a component of $\mu$, $\omega$, $\alpha$, $W_Q$, or $W_K$. \\

$\mathbb{S}_{++}^{d_r}$ 
& Cone of $d_r\times d_r$ symmetric positive-definite matrices. \\

$R$ 
& Radius used in the local Lipschitz analysis to bound $\|x-\mu_k\|_2$. \\

$L_\gamma$ 
& A local Lipschitz constant for the responsibility map $x\mapsto\gamma(x)$. \\

\midrule
\multicolumn{2}{@{}l}{\textit{Computational-cost notation}} \\
\addlinespace[2pt]

$\mathcal{C}_{\mathrm{GMA}}$ 
& Approximate routing cost of GMA, excluding the standard linear projections. \\

$\mathcal{C}_{\mathrm{proj}}$ 
& Cost of forming the query, key, and value projections $Q_X$, $K_X$, and $V_X$. \\

\bottomrule
\end{tabular}
\end{table}

\section{Motivation: the Matrix Multiplication Association Rule}
\label{app:matrix_association_rule}

The linear-time implementation of GMA relies on a standard principle from matrix algebra and matrix-chain multiplication. Matrix multiplication is \textit{associative}: whenever the dimensions are compatible,
\begin{equation}
    (AB)C = A(BC).
\end{equation}
Thus, different parenthesizations give the same final matrix. However, they can have very different computational costs. This is the classical matrix-chain multiplication observation: although all valid parenthesizations of a matrix product are algebraically equivalent, the choice of parenthesization can have a dramatic impact on the number of scalar multiplications required \cite{cormen2009introduction}. The associative law itself is a basic linear algebra identity \cite{strang2020}.

To see this concretely, let
\[
    A\in\mathbb{R}^{m\times n},
    \qquad
    B\in\mathbb{R}^{n\times p},
    \qquad
    C\in\mathbb{R}^{p\times q}.
\]
Computing $(AB)C$ first forms $AB\in\mathbb{R}^{m\times p}$, so the approximate multiplication cost is \cite{stephen2018VMLS}
\begin{equation}
    \mathcal{C}_{(AB)C}
    =
    mnp + mpq.
\end{equation}
In contrast, computing $A(BC)$ first forms $BC\in\mathbb{R}^{n\times q}$, giving cost
\begin{equation}
    \mathcal{C}_{A(BC)}
    =
    npq + mnq.
\end{equation}
These two costs can differ substantially, even though the final result is identical.
For example, take
\[
    A\in\mathbb{R}^{1000\times 10},
    \qquad
    B\in\mathbb{R}^{10\times 1000},
    \qquad
    C\in\mathbb{R}^{1000\times 10}.
\]
Then computing $(AB)C$ first forms the large intermediate matrix $AB\in\mathbb{R}^{1000\times 1000}$. The cost is
\[
    1000\cdot 10\cdot 1000
    +
    1000\cdot 1000\cdot 10
    =
    2\times 10^7.
\]
By contrast, computing $A(BC)$ first forms the much smaller intermediate matrix $BC\in\mathbb{R}^{10\times 10}$. The cost is
\[
    10\cdot 1000\cdot 10
    +
    1000\cdot 10\cdot 10
    =
    2\times 10^5.
\]
Thus, the two parenthesizations produce the same output, but the second is cheaper by a factor of $100$.

This principle directly explains the efficient implementation of GMA. Ignoring the row-wise normalizer for the moment, the numerator of the GMA output in Eq.~\eqref{eq:gma_normalized_broadcasting} has the algebraic form
\[
    \Gamma^Q(\Gamma^K)^\top V_X.
\]
If we identify
\[
    A=\Gamma^Q\in\mathbb{R}^{N\times K},
    \qquad
    B=(\Gamma^K)^\top\in\mathbb{R}^{K\times N},
    \qquad
    C=V_X\in\mathbb{R}^{N\times d_v},
\]
then the explicit-affinity computation corresponds to
\begin{equation}
    \bigl(\Gamma^Q(\Gamma^K)^\top\bigr)V_X.
\end{equation}
This first forms the implicit token-to-token affinity matrix
\[
    \Gamma^Q(\Gamma^K)^\top\in\mathbb{R}^{N\times N},
\]
which costs $\mathcal{O}(N^2K)$ arithmetic and requires $\mathcal{O}(N^2)$ intermediate storage before multiplying by $V_X$.

GMA instead uses the associative parenthesization
\begin{equation}
    \Gamma^Q\bigl((\Gamma^K)^\top V_X\bigr).
\end{equation}
The inner product
\begin{equation}
    \tilde V=(\Gamma^K)^\top V_X
    \in\mathbb{R}^{K\times d_v}
\end{equation}
constructs a compact latent memory with only $K$ slots. The final multiplication $\Gamma^Q\tilde V$ then reads from this memory. The cost of these two matrix multiplications is
\begin{equation}
    NKd_v + NKd_v
    =
    2NKd_v,
\end{equation}
and the intermediate memory has size only $K\times d_v$. Therefore, GMA obtains the same implicit affinity interpretation without materializing the $N\times N$ affinity matrix.

The full normalized GMA computation also includes the denominator
\begin{equation}
    Z=(\Gamma^K)^\top\mathbf{1}_N,
    \qquad
    \Gamma^Q Z+\epsilon,
\end{equation}
which adds only $\mathcal{O}(NK)$ arithmetic. In addition, computing the two responsibility matrices $\Gamma^Q$ and $\Gamma^K$ under diagonal Gaussian components costs $\mathcal{O}(2NKd_r)$, up to lower-order normalization terms. Thus, the overall routing computation remains linear in $N$ for fixed $K$, $d_r$, and $d_v$, while retaining an implicit normalized attention matrix.

\section[Understanding GMA in Detail]{Understanding GMA in Detail\protect\footnote{In this section, the author would like to express sincere gratitude to those whose textbooks, lectures, and scholarly discussions have shaped his understanding of numerical linear algebra, including Gilbert Strang, Stephen Boyd, Steven Roman, Nick Higham, and many others.}}
\label{app:understand_GMA}

We have briefly touched on the intuition behind GMA as responsibility-space affinity followed by latent-memory routing at the end of the core methodology in Section~\ref{sec:linear_gma}. This appendix gives a more dimension-explicit walkthrough of the same computation. The goal is to clarify how key responsibilities write values into a latent memory, how query responsibilities read from this memory, how the implicit token-to-token affinity arises, and why the computation avoids materializing an explicit $N\times N$ attention matrix.

\subsection{Objects and Dimensions}

Consider a sequence of length $N$. Let $V_X\in\mathbb{R}^{N\times d_v}$ denote the value matrix, where $V_{X,j}\in\mathbb{R}^{d_v}$ is the value vector at position $j$. GMA computes two responsibility matrices:
\[
    \Gamma^K \in \mathbb{R}^{N\times K},
    \qquad
    \Gamma^Q \in \mathbb{R}^{N\times K}.
\]
The row $\gamma^K_j=(\gamma^K_{j,1},\ldots,\gamma^K_{j,K})$ gives the posterior responsibilities of the $K$ latent Gaussian components for key position $j$. Similarly, $\gamma^Q_i=(\gamma^Q_{i,1},\ldots,\gamma^Q_{i,K})$ gives the responsibilities of the same latent components for query position $i$. Thus, $\gamma^K_{j,k}$ measures how much key position $j$ belongs to latent component $k$, while $\gamma^Q_{i,k}$ measures how much query position $i$ reads from component $k$.

Each row of a responsibility matrix lies on the probability simplex:
\[
    \sum_{k=1}^K \gamma^K_{j,k}=1,
    \qquad
    \sum_{k=1}^K \gamma^Q_{i,k}=1.
\]
Therefore, GMA routes tokens through normalized probability vectors over latent components rather than through unconstrained similarity scores.

\subsection{The Write Step: Keys Write Values into Latent Memory}

The key responsibilities first aggregate the value vectors into $K$ latent memory slots:
\[
    \tilde V = (\Gamma^K)^\top V_X \in \mathbb{R}^{K\times d_v},
    \qquad
    Z = (\Gamma^K)^\top \mathbf{1}_N \in \mathbb{R}^{K},
\]
where $\mathbf{1}_N$ is the all-ones vector. Component-wise, this means
\[
    \tilde V_k
    =
    \sum_{j=1}^N \gamma^K_{j,k} V_{X,j}
    \in \mathbb{R}^{d_v},
    \qquad
    Z_k
    =
    \sum_{j=1}^N \gamma^K_{j,k}.
\]
Thus, $\tilde V_k$ is the responsibility-weighted sum of all value vectors assigned to latent component $k$, and $Z_k$ records the total key-responsibility mass written into that component.

An intuitive normalized memory slot can be written as
\[
    M_k
    =
    \frac{\tilde V_k}{Z_k+\epsilon}
    =
    \frac{
    \sum_{j=1}^N \gamma^K_{j,k} V_{X,j}
    }{
    \sum_{j=1}^N \gamma^K_{j,k}+\epsilon
    }
    \in \mathbb{R}^{d_v}.
\]
This shows that $M_k$ is the average value stored in latent component $k$, up to the numerical stabilizer $\epsilon$. In practice, the implementation does not need to explicitly materialize $M_k$; it is enough to keep the unnormalized memory $\tilde V_k$ and the normalizer $Z_k$.

\subsection{The Read Step: Queries Read from Latent Memory}

After the write step, query responsibilities read from the latent memory. In matrix form, the GMA output is
\[
    O =
    \frac{\Gamma^Q \tilde V}{\Gamma^Q Z+\epsilon}
    \in \mathbb{R}^{N\times d_v},
\]
where the denominator is broadcast across the value dimension and the division is applied row-wise. For query position $i$, this is
\[
    O_i
    =
    \frac{
    \sum_{k=1}^K \gamma^Q_{i,k}\tilde V_k
    }{
    \sum_{k=1}^K \gamma^Q_{i,k}Z_k+\epsilon
    }.
\]
Substituting the write step into this expression gives
\[
    O_i
    =
    \frac{
    \sum_{k=1}^K
    \gamma^Q_{i,k}
    \sum_{j=1}^N
    \gamma^K_{j,k}V_{X,j}
    }{
    \sum_{k=1}^K
    \gamma^Q_{i,k}
    \sum_{j=1}^N
    \gamma^K_{j,k}
    +\epsilon
    }.
\]
This equation makes the full workflow explicit: key responsibilities distribute values into latent components, and query responsibilities select a normalized mixture of those latent components.

Ignoring the numerical stabilizer for clarity, we can also express the read operation as attention over normalized memory slots. If $Z_k>0$, define $M_k=\tilde V_k/Z_k$ and
\[
    \alpha_{i,k}
    =
    \frac{
    \gamma^Q_{i,k}Z_k
    }{
    \sum_{\ell=1}^K \gamma^Q_{i,\ell}Z_\ell
    }.
\]
Then $\sum_k \alpha_{i,k}=1$, and
\[
    O_i
    =
    \sum_{k=1}^K \alpha_{i,k} M_k.
\]
This form is useful for intuition. The query does not simply average memory slots using $\gamma^Q_{i,k}$ alone. Instead, the effective read weight $\alpha_{i,k}$ also depends on $Z_k$, the amount of key mass stored in component $k$. For fixed query responsibility $\gamma^Q_{i,k}$, a component with larger stored key mass receives a larger effective read weight than an almost-empty component.

\subsection{Implicit Token-to-Token Attention}

Although GMA does not explicitly construct an $N\times N$ attention matrix, it still induces an attention-style weighting over value tokens. Combining the write and read equations gives
\[
    O_i
    =
    \sum_{j=1}^N A^{\mathrm{GMA}}_{ij} V_{X,j},
\]
where
\[
    A^{\mathrm{GMA}}_{ij}
    =
    \frac{
    \sum_{k=1}^K \gamma^Q_{i,k}\gamma^K_{j,k}
    }{
    \sum_{\ell=1}^N
    \sum_{k=1}^K
    \gamma^Q_{i,k}\gamma^K_{\ell,k}
    +\epsilon
    }.
\]
Thus, two tokens interact strongly when their query and key responsibilities overlap in the latent mixture space. In standard attention, token $i$ attends to token $j$ through a direct dot product $q_i^\top k_j$. In GMA, token $i$ attends to token $j$ through their shared responsibility mass $\sum_k \gamma^Q_{i,k}\gamma^K_{j,k}$.

When $\epsilon=0$, the induced weights are row-normalized:
\[
    \sum_{j=1}^N A^{\mathrm{GMA}}_{ij}=1.
\]
With $\epsilon>0$, the row sum is slightly below one, since the stabilizer adds extra positive mass to the denominator. The stabilizer is included only for numerical safety.

\subsection{Workflow Intuition}

The full GMA computation can be understood as a three-stage routing process:

\begin{itemize}
    \item First, the Gaussian mixture maps each token representation to a soft assignment over $K$ latent components. This produces $\Gamma^K$ for keys and $\Gamma^Q$ for queries. These matrices are probabilistic: each row is a distribution over latent components.
    \item Second, key responsibilities write the sequence values into a compact latent memory. Instead of storing $N$ separate value vectors for direct pairwise comparison, GMA stores $K$ responsibility-weighted memory slots. Each slot summarizes the values assigned to one latent component, together with a normalizer recording how much total key mass was assigned there.
    \item Third, query responsibilities read from this latent memory. A query token selects a mixture of memory slots according to its own responsibilities and the amount of key mass stored in those slots. The output is therefore a normalized mixture of component-level summaries.
\end{itemize}

In short, the workflow can be summarised as:
\[
    \text{keys write values into } K \text{ latent memory slots,}
\]
and
\[
    \text{queries read mixtures of those slots to produce token outputs.}
\]

\subsection{Why This Avoids the Quadratic Bottleneck}

Standard dot-product attention forms or implicitly represents all pairwise token-to-token scores, producing an intermediate $N\times N$ attention matrix. GMA instead uses the two $N\times K$ responsibility matrices $\Gamma^K$ and $\Gamma^Q$, together with the latent memory $\tilde V\in\mathbb{R}^{K\times d_v}$ and normalizer $Z\in\mathbb{R}^K$. For fixed $K$, the dominant activation storage is therefore linear in $N$.

The main costs are:
\[
    \text{responsibility computation: } \mathcal{O}(NKd_r),
\]
where $d_r$ is the routing representation dimension, and
\[
    \text{latent write/read operations: } \mathcal{O}(NKd_v).
\]
Thus, for fixed $K$, GMA scales linearly with sequence length while retaining a
normalized attention-style interpretation.

\subsection{Causal GMA}

For autoregressive language modelling, position $i$ must not read information from future positions $j>i$. Causal GMA replaces the global write statistics by prefix-restricted statistics:
\[
    \tilde V^{(i)}_k
    =
    \sum_{j\leq i} \gamma^K_{j,k}V_{X,j},
    \qquad
    Z^{(i)}_k
    =
    \sum_{j\leq i} \gamma^K_{j,k}.
\]
The output at position $i$ is then
\[
    O_i
    =
    \frac{
    \sum_{k=1}^K \gamma^Q_{i,k}\tilde V^{(i)}_k
    }{
    \sum_{k=1}^K \gamma^Q_{i,k}Z^{(i)}_k
    +\epsilon
    }.
\]
This is the same write-read idea, but the memory available to position $i$ is restricted to its prefix. In implementation, the prefix memories are computed by cumulative sums along the sequence dimension. Therefore, causal GMA preserves the same fixed-$K$ linear scaling while enforcing autoregressive causality.

\section{Detailed Gradient Derivations for Reparameterized GMM}
\label{app:gradient_derivations}

In Section~\ref{sec:end_to_end_parameter_learning}, we defined the trainable parameters of a GMA layer as $\Theta_{\mathrm{GMA}}=\{W_Q,W_K,W_V,\mu,\omega,\alpha\}$ in Eq.~\eqref{eq:GMA_trainable_params}. Among these, $W_Q,W_K,W_V$ are the standard query, key, and value projection matrices (Eq.~\eqref{eq:gma_qkv_projection}), while $\mu,\omega,\alpha$ parameterize the Gaussian mixture routing mechanism: $\mu$ contains the component means, $\omega$ parameterizes the diagonal covariance entries (Eq.~\eqref{eq:gma_covariance_reparameterization}), and $\alpha$ parameterizes the mixture-prior logits (Eq.~\eqref{eq:gma_prior_reparameterization}). This appendix derives the gradients for the reparameterized GMM routing parameters $\mu,\omega,\alpha$. The projection matrices $W_Q$ and $W_K$ receive gradients through the routing vectors $q_i=(Q_X)_{i,:}$ and $k_i=(K_X)_{i,:}$, while $W_V$ receives gradients through the latent aggregation and output operations.

Throughout, $x_i\in\mathbb{R}^{d_r}$ denotes a generic routing vector, which may be either a query routing vector $q_i$ or a key routing vector $k_i$. We use $\theta$ to denote a generic scalar GMM routing parameter, i.e. $\theta\in\{\mu_{k,m},\omega_{k,m},\alpha_m\}$ for appropriate component and coordinate indices. For component $k$, the pre-normalized log-density score under the \textit{diagonal} covariance assumption is a coordinate-wise version of Eq.~\eqref{eq:gma_log_density_score}:
\begin{equation}
\label{eq:app_gma_log_density_diag}
    s_{i,k}
    =
    \log \pi_k
    -
    \frac{d_r}{2}\log(2\pi)
    -
    \frac{1}{2}
    \sum_{m=1}^{d_r}
    \log(\sigma_{k,m}^2)
    -
    \frac{1}{2}
    \sum_{m=1}^{d_r}
    \frac{(x_{i,m}-\mu_{k,m})^2}{\sigma_{k,m}^2}.
\end{equation}
The responsibility is the softmax-normalized score (Eq.\eqref{eq:gma_responsibility_score_softmax})
\begin{equation}
\label{eq:app_gma_responsibility_softmax}
    \gamma_{i,k}
    =
    \frac{\exp(s_{i,k})}
    {\sum_{\ell=1}^{K}\exp(s_{i,\ell})}.
\end{equation}
Let
\[
    g_{i,j}
    =
    \frac{\partial \mathcal{L}}{\partial \gamma_{i,j}}
\]
denote the upstream gradient arriving at the responsibility vector for token $i$. By the softmax Jacobian (Eq.\eqref{eq:gma_softmax_jacobian}),
\[
    \frac{\partial \gamma_{i,j}}{\partial s_{i,k}}
    =
    \gamma_{i,j}
    \left(
    \mathbf{1}\{j=k\}
    -
    \gamma_{i,k}
    \right).
\]
Therefore, the gradient entering the score $s_{i,k}$ is (Eq.\eqref{eq:gma_score_gradient_general})
\[
    \delta_{i,k}
    \equiv
    \frac{\partial \mathcal{L}}{\partial s_{i,k}}
    =
    \sum_{j=1}^{K}
    \frac{\partial \mathcal{L}}{\partial \gamma_{i,j}}
    \frac{\partial \gamma_{i,j}}{\partial s_{i,k}}
    =
    \gamma_{i,k}
    \left(
    g_{i,k}
    -
    \sum_{j=1}^{K}
    \gamma_{i,j}g_{i,j}
    \right).
\]
This is the quantity that multiplies each parameter derivative $\partial s_{i,k}/\partial\theta$ in backpropagation. The familiar factor $\gamma_{i,k}(1-\gamma_{i,k})$ appears as the self-sensitivity $\partial\gamma_{i,k}/\partial s_{i,k}$ (Eq.\eqref{eq:gma_self_responsibility_derivative}), while the full gradient also includes cross-component softmax terms (Eq.\eqref{eq:gma_cross_mean_responsibility_derivative}).

For any learnable scalar parameter $\theta$ that enters the scores, such as a component of $\mu$, $\omega$, or $\alpha$, or a projection parameter entering the routing vectors through $W_Q$ or $W_K$, the chain rule gives
\begin{equation}
\label{eq:gma_parameter_gradient_general_a_app} \tag{cc.Eq.\ref{eq:gma_parameter_gradient_general_a}}
    \frac{\partial \mathcal{L}}{\partial \theta}
    =
    \sum_i
    \sum_{k=1}^K
    \frac{\partial \mathcal{L}}{\partial s_{i,k}}
    \frac{\partial s_{i,k}}{\partial \theta},
\end{equation}
where $\partial \mathcal{L}/\partial s_{i,k}$ is given by Eq.~\eqref{eq:gma_score_gradient_general}. Equivalently, expanding the softmax Jacobian explicitly gives
\begin{equation} \label{eq:gma_parameter_gradient_general_b_app} \tag{cc.Eq.\ref{eq:gma_parameter_gradient_general_b}}
    \frac{\partial \mathcal{L}}{\partial \theta}
    =
    \sum_i
    \sum_{j=1}^K
    \sum_{k=1}^K
    \frac{\partial \mathcal{L}}{\partial \gamma_{i,j}}
    \frac{\partial \gamma_{i,j}}{\partial s_{i,k}}
    \frac{\partial s_{i,k}}{\partial \theta}.
\end{equation}
This expression is the appropriate backpropagation form for the responsibility-based GMA layer.
We now derive the score derivatives for the reparameterized GMM parameters.

\subsection{Derivation for Latent Means}
\label{app:gradient_means}

The means are unconstrained parameters $\mu_k\in\mathbb{R}^{d_r}$. Differentiating Eq.~\eqref{eq:app_gma_log_density_diag} with respect to the $m$-th coordinate of $\mu_k$ gives
\begin{equation}
\label{eq:app_mean_score_derivative_coord}
    \frac{\partial s_{i,k}}{\partial \mu_{k,m}}
    =
    \frac{x_{i,m}-\mu_{k,m}}{\sigma_{k,m}^2}.
\end{equation}
In vector form,
\begin{equation}
\label{eq:app_mean_score_derivative_vec}
    \nabla_{\mu_k}s_{i,k}
    =
    \Sigma_k^{-1}(x_i-\mu_k).
\end{equation}
Thus, the full loss gradient with respect to the mean of component $k$ is
\begin{equation}
\label{eq:app_mean_loss_gradient}
    \frac{\partial \mathcal{L}}{\partial \mu_k}
    =
    \sum_i
    \delta_{i,k}
    \Sigma_k^{-1}(x_i-\mu_k).
\end{equation}
Hence, the mean update is driven by precision-scaled residuals, weighted by the score-level gradient $\delta_{i,k}$ induced by the responsibility softmax.

\subsection{Derivation for Reparameterized Diagonal Covariances}
\label{app:gradient_covariances}

To ensure positive diagonal covariance entries, GMA maintains unconstrained parameters $\omega\in\mathbb{R}^{K\times d_r}$ and sets (Eq.\eqref{eq:gma_covariance_reparameterization})
\begin{equation}
\label{eq:app_covariance_reparameterization}
    \sigma_{k,m}^2
    =
    \operatorname{softplus}(\omega_{k,m})
    +
    \epsilon_\sigma
    =
    \log(1+\exp(\omega_{k,m}))
    +
    \epsilon_\sigma ,
    \qquad
    m=1,\ldots,d_r.
\end{equation}
By the chain rule,
\begin{equation}
\label{eq:app_cov_chain_rule}
    \frac{\partial s_{i,k}}{\partial \omega_{k,m}}
    =
    \frac{\partial s_{i,k}}{\partial \sigma_{k,m}^2}
    \frac{\partial \sigma_{k,m}^2}{\partial \omega_{k,m}}.
\end{equation}
First, differentiating Eq.~\eqref{eq:app_gma_log_density_diag} with respect to the variance $\sigma_{k,m}^2$ gives
\begin{equation}
\label{eq:app_variance_score_derivative}
    \frac{\partial s_{i,k}}{\partial \sigma_{k,m}^2}
    =
    -\frac{1}{2\sigma_{k,m}^2}
    +
    \frac{(x_{i,m}-\mu_{k,m})^2}{2(\sigma_{k,m}^2)^2}
    =
    \frac{1}{2\sigma_{k,m}^2}
    \left(
    \frac{(x_{i,m}-\mu_{k,m})^2}{\sigma_{k,m}^2}
    -
    1
    \right).
\end{equation}
Second, the derivative of the softplus reparameterization is
\begin{equation}
\label{eq:app_softplus_derivative}
    \frac{\partial \sigma_{k,m}^2}{\partial \omega_{k,m}}
    =
    \frac{\exp(\omega_{k,m})}{1+\exp(\omega_{k,m})}
    =
    \sigma_{\mathrm{sig}}(\omega_{k,m}),
\end{equation}
where $\sigma_{\mathrm{sig}}(\cdot)$ denotes the logistic \textit{sigmoid function}. Combining Eqs.~\eqref{eq:app_variance_score_derivative} and \eqref{eq:app_softplus_derivative}, we obtain
\begin{equation}
\label{eq:app_omega_score_derivative}
    \frac{\partial s_{i,k}}{\partial \omega_{k,m}}
    =
    \frac{1}{2\sigma_{k,m}^2}
    \left(
    \frac{(x_{i,m}-\mu_{k,m})^2}{\sigma_{k,m}^2}
    -
    1
    \right)
    \sigma_{\mathrm{sig}}(\omega_{k,m}).
\end{equation}
Therefore, the full loss gradient for the covariance parameter is
\begin{equation}
\label{eq:app_omega_loss_gradient}
    \frac{\partial \mathcal{L}}{\partial \omega_{k,m}}
    =
    \sum_i
    \delta_{i,k}
    \frac{1}{2\sigma_{k,m}^2}
    \left(
    \frac{(x_{i,m}-\mu_{k,m})^2}{\sigma_{k,m}^2}
    -
    1
    \right)
    \sigma_{\mathrm{sig}}(\omega_{k,m}).
\end{equation}
The sigmoid factor attenuates the gradient passed through the softplus reparameterization, while the additive lower bound $\epsilon_\sigma$ prevents the diagonal variances from becoming singular. This improves numerical stability, although it does not by itself eliminate the need for appropriate initialization, learning-rate control, or other optimization safeguards.

\subsection{Derivation for Reparameterized Mixture Priors}
\label{app:gradient_priors}

To satisfy the simplex constraint, the mixture priors are parameterized by unconstrained logits $\alpha\in\mathbb{R}^K$ (Eq.\eqref{eq:gma_prior_reparameterization}):
\begin{equation}
\label{eq:app_prior_softmax} \tag{cc.Eq.\ref{eq:gma_prior_reparameterization}}
    \pi_k
    =
    \frac{\exp(\alpha_k)}
    {\sum_{\ell=1}^{K}\exp(\alpha_\ell)}.
\end{equation}
Because $s_{i,k}$ contains the term $\log\pi_k$ (Eq.\eqref{eq:app_gma_log_density_diag}), we need the derivative of $\log\pi_k$ with respect to an arbitrary logit $\alpha_m$. Using the softmax Jacobian,
\begin{equation}
\label{eq:app_prior_softmax_jacobian}
    \frac{\partial \pi_k}{\partial \alpha_m}
    =
    \pi_k(\mathbf{1}\{k=m\}-\pi_m).
\end{equation}
Therefore,
\begin{equation}
\label{eq:app_log_prior_derivative}
    \frac{\partial s_{i,k}}{\partial \alpha_m}
    =
    \frac{\partial \log\pi_k}{\partial \alpha_m}
    =
    \frac{1}{\pi_k}
    \frac{\partial \pi_k}{\partial \alpha_m}
    =
    \mathbf{1}\{k=m\}
    -
    \pi_m.
\end{equation}
The full loss gradient with respect to prior logit $\alpha_m$ is then
\begin{equation}
\label{eq:app_alpha_loss_gradient}
    \frac{\partial \mathcal{L}}{\partial \alpha_m}
    =
    \sum_i
    \sum_{k=1}^{K}
    \delta_{i,k}
    \left(
    \mathbf{1}\{k=m\}
    -
    \pi_m
    \right).
\end{equation}
Since the score-level gradients $\delta_{i,k}$ arise from a softmax Jacobian, they satisfy $\sum_{k=1}^K\delta_{i,k}=0$ for each token $i$. Hence Eq.~\eqref{eq:app_alpha_loss_gradient} can also be written as
\begin{equation}
\label{eq:app_alpha_loss_gradient_simplified}
    \frac{\partial \mathcal{L}}{\partial \alpha_m}
    =
    \sum_i
    \delta_{i,m}.
\end{equation}
At the score level, increasing $\alpha_m$ increases the log-prior contribution of component $m$ by $1-\pi_m$ and decreases the log-prior contribution of every other component by $\pi_m$. The actual parameter update is then determined by the downstream loss through the score-level gradients $\delta_{i,k}$.

\section{Self-Attention, Cross-Attention, and GMA}
\label{app:self_cross_gma}

This appendix clarifies the distinction between \textit{self-attention} and \textit{cross-attention}, and shows how the GMA routing mechanism extends naturally to both settings. To avoid notational ambiguity, we write $K_{\mathrm{att}}$ for the standard-attention key matrix, while $K$ denotes the number of Gaussian mixture components in GMA.

\subsection{Standard Self-Attention}

In self-attention, the queries, keys, and values are all projected from the same input sequence. Let
\[
    X\in\mathbb{R}^{N\times d_{\mathrm{model}}},
\]
where $N$ is the sequence length. Standard self-attention forms
\[
    Q=XW_Q,
    \qquad
    K_{\mathrm{att}}=XW_K,
    \qquad
    V=XW_V,
\]
with
\[
    Q,K_{\mathrm{att}}\in\mathbb{R}^{N\times d_k},
    \qquad
    V\in\mathbb{R}^{N\times d_v}.
\]
The attention matrix and output are
\[
    A
    =
    \operatorname{softmax}
    \left(
    \frac{QK_{\mathrm{att}}^\top}{\sqrt{d_k}}
    \right),
    \qquad
    O=AV.
\]
Elementwise,
\[
    A_{ij}
    =
    \frac{
    \exp(q_i^\top k_j/\sqrt{d_k})
    }{
    \sum_{\ell=1}^{N}
    \exp(q_i^\top k_\ell/\sqrt{d_k})
    },
    \qquad
    O_i=\sum_{j=1}^{N} A_{ij}v_j.
\]
Thus, token $i$ attends to tokens $j$ within the same sequence.

For causal self-attention, as used in autoregressive language modelling, future positions are masked:
\[
    A_{ij}=0
    \qquad
    \text{for } j>i.
\]
Equivalently,
\[
    A_{ij}
    =
    \frac{
    \mathbf{1}\{j\leq i\}
    \exp(q_i^\top k_j/\sqrt{d_k})
    }{
    \sum_{\ell\leq i}
    \exp(q_i^\top k_\ell/\sqrt{d_k})
    }.
\]
The representation at position $i$ may depend only on tokens at positions $j\leq i$, and is then used for next-token prediction.

\subsection{Standard Cross-Attention}

In cross-attention, the queries come from one sequence, while keys and values come from another sequence. Let
\[
    X\in\mathbb{R}^{N_q\times d_{\mathrm{model}}}
\]
denote a query-side sequence, and let
\[
    Y\in\mathbb{R}^{N_k\times d_{\mathrm{model}}}
\]
denote a context or source sequence. Cross-attention forms
\[
    Q_X=XW_Q,
    \qquad
    K_Y=YW_K,
    \qquad
    V_Y=YW_V,
\]
where
\[
    Q_X\in\mathbb{R}^{N_q\times d_k},
    \qquad
    K_Y\in\mathbb{R}^{N_k\times d_k},
    \qquad
    V_Y\in\mathbb{R}^{N_k\times d_v}.
\]
The attention matrix is rectangular:
\[
    A
    =
    \operatorname{softmax}
    \left(
    \frac{Q_XK_Y^\top}{\sqrt{d_k}}
    \right)
    \in\mathbb{R}^{N_q\times N_k},
\]
and the output is
\[
    O_X=AV_Y\in\mathbb{R}^{N_q\times d_v}.
\]
Elementwise,
\[
    A_{ij}
    =
    \frac{
    \exp(q_{X,i}^\top k_{Y,j}/\sqrt{d_k})
    }{
    \sum_{\ell=1}^{N_k}
    \exp(q_{X,i}^\top k_{Y,\ell}/\sqrt{d_k})
    },
    \qquad
    O_{X,i}=\sum_{j=1}^{N_k} A_{ij}v_{Y,j}.
\]
Thus, query token $i$ in sequence $X$ attends to context token $j$ in sequence $Y$. Self-attention is the special case in which the query-side and context-side sequences coincide.

\subsection{GMA Self-Attention}

The GMA mechanism developed in Section~\ref{sec:GMA_method} is presented primarily in the \textit{self-attention} setting. Given
\[
    X\in\mathbb{R}^{N\times d_{\mathrm{model}}},
\]
we form
\[
    Q_X=XW_Q,
    \qquad
    K_X=XW_K,
    \qquad
    V_X=XW_V.
\]
Instead of computing all pairwise dot products $q_i^\top k_j$, GMA maps query and key projections to responsibility vectors over $K$ learned Gaussian components:
\[
    \Gamma^Q=\Gamma(Q_X),
    \qquad
    \Gamma^K=\Gamma(K_X),
\]
where
\[
    \Gamma^Q,\Gamma^K\in\mathbb{R}^{N\times K}.
\]
For a generic routing vector $x$, the responsibility of component $k$ is
\[
    \gamma_k(x)
    =
    \frac{
    \pi_k\mathcal{N}(x\mid\mu_k,\Sigma_k)
    }{
    \sum_{\ell=1}^{K}
    \pi_\ell\mathcal{N}(x\mid\mu_\ell,\Sigma_\ell)
    }.
\]
The key-side responsibilities write values into a latent memory:
\[
    \tilde V=(\Gamma^K)^\top V_X,
    \qquad
    Z=(\Gamma^K)^\top \mathbf{1}_N,
\]
where
\[
    \tilde V\in\mathbb{R}^{K\times d_v},
    \qquad
    Z\in\mathbb{R}^{K}.
\]
The query-side responsibilities then read from this latent memory:
\[
    O
    =
    \frac{\Gamma^Q\tilde V}{\Gamma^Q Z+\epsilon},
\]
where the denominator is broadcast row-wise across the value dimension.
Equivalently, for token $i$,
\[
    O_i
    =
    \frac{
    \sum_{k=1}^{K}\gamma^Q_{i,k}\tilde V_k
    }{
    \sum_{k=1}^{K}\gamma^Q_{i,k}Z_k+\epsilon
    }.
\]
The induced GMA self-attention weight from token $i$ to token $j$ is
\[
    A^{\mathrm{GMA}}_{ij}
    =
    \frac{
    \sum_{k=1}^{K}
    \gamma^Q_{i,k}\gamma^K_{j,k}
    }{
    \sum_{\ell=1}^{N}
    \sum_{k=1}^{K}
    \gamma^Q_{i,k}\gamma^K_{\ell,k}
    +
    \epsilon
    }.
\]
Thus, GMA replaces direct pairwise dot-product comparison with responsibility-space affinity followed by latent-memory routing.

For causal GMA self-attention, the key-side latent memory is accumulated only over prefix positions:
\[
    \tilde V^{(i)}_k
    =
    \sum_{j\leq i}\gamma^K_{j,k}V_{X,j},
    \qquad
    Z^{(i)}_k
    =
    \sum_{j\leq i}\gamma^K_{j,k}.
\]
The causal output is
\[
    O_i
    =
    \frac{
    \sum_{k=1}^{K}\gamma^Q_{i,k}\tilde V^{(i)}_k
    }{
    \sum_{k=1}^{K}\gamma^Q_{i,k}Z^{(i)}_k+\epsilon
    }.
\]
The corresponding induced causal GMA weight is
\[
    A^{\mathrm{cGMA}}_{ij}
    =
    \mathbf{1}\{j\leq i\}
    \frac{
    \sum_{k=1}^{K}
    \gamma^Q_{i,k}\gamma^K_{j,k}
    }{
    \sum_{\ell\leq i}
    \sum_{k=1}^{K}
    \gamma^Q_{i,k}\gamma^K_{\ell,k}
    +
    \epsilon
    }.
\]
This is the form used in the WikiText-103 language-modelling experiment.

\subsection{GMA Cross-Attention}

GMA can also be applied to \textit{cross-attention}. Let
\[
    X\in\mathbb{R}^{N_q\times d_{\mathrm{model}}}
\]
be the query-side sequence and
\[
    Y\in\mathbb{R}^{N_k\times d_{\mathrm{model}}}
\]
be the context-side sequence. We form
\[
    Q_X=XW_Q,
    \qquad
    K_Y=YW_K,
    \qquad
    V_Y=YW_V.
\]
The query-side and context-side responsibilities are
\[
    \Gamma_X^Q=\Gamma(Q_X)\in\mathbb{R}^{N_q\times K},
    \qquad
    \Gamma_Y^K=\Gamma(K_Y)\in\mathbb{R}^{N_k\times K}.
\]
The context sequence writes values into a latent memory:
\[
    \tilde V_Y=(\Gamma_Y^K)^\top V_Y,
    \qquad
    Z_Y=(\Gamma_Y^K)^\top \mathbf{1}_{N_k}.
\]
The query sequence reads from this context-side memory:
\[
    O_X
    =
    \frac{
    \Gamma_X^Q\tilde V_Y
    }{
    \Gamma_X^Q Z_Y+\epsilon
    }
    \in\mathbb{R}^{N_q\times d_v}.
\]
Elementwise,
\[
    O_{X,i}
    =
    \frac{
    \sum_{k=1}^{K}
    \gamma^Q_{X,i,k}\tilde V_{Y,k}
    }{
    \sum_{k=1}^{K}
    \gamma^Q_{X,i,k}Z_{Y,k}
    +
    \epsilon
    }.
\]
The induced GMA cross-attention weight from query token $i$ in $X$ to context token $j$ in $Y$ is
\[
    A^{\mathrm{GMA\text{-}cross}}_{ij}
    =
    \frac{
    \sum_{k=1}^{K}
    \gamma^Q_{X,i,k}\gamma^K_{Y,j,k}
    }{
    \sum_{\ell=1}^{N_k}
    \sum_{k=1}^{K}
    \gamma^Q_{X,i,k}\gamma^K_{Y,\ell,k}
    +
    \epsilon
    }.
\]
Thus, GMA cross-attention has the same structure as GMA self-attention, except that query responsibilities are computed from the query-side sequence and key responsibilities/value memories are computed from the context-side sequence.

The experiments in this paper focus on self-attention-style sequence mixing: bidirectional self-attention-style mixing for LRA and causal self-attention-style mixing for WikiText-103. A systematic empirical study of GMA cross-attention, for example in encoder--decoder or multimodal architectures, is left to future work.

\section{Additional Interpretability Diagnostics}
\label{app:gma_interpretability}

This appendix reports additional diagnostics for the interpretability analysis in Section~\ref{sec:gma_interpretability}. All results use the trained WikiText-103 causal GMA model with $K=128$. We extract the final-layer query responsibility tensor, average it over heads for each token occurrence, and analyze $S=64$ validation sequences of length $L=1024$, giving $T=SL=65{,}536$ token-level responsibility vectors. Following Section~\ref{sec:gma_interpretability}, we write the resulting head-averaged responsibility vector as $\bar{\gamma}^Q_t\in\Delta^{K-1}$ for $t=1,\ldots,T$, and define the hard assignment $z_t=\arg\max_k \bar{\gamma}^Q_{t,k}$.

\subsection{Token Category Distribution}

The token categories used in the purity and mutual-information analysis are surface-form categories derived from GPT-2 BPE tokens. Table~\ref{tab:category_distribution} shows the category frequencies in the analyzed validation subset.

\begin{table}[H]
\centering
\caption{Surface-form token category distribution in the analyzed WikiText-103 validation subset.}
\label{tab:category_distribution}
\footnotesize
\setlength{\tabcolsep}{5pt}
\renewcommand{\arraystretch}{0.95}
\begin{tabular}{lcc}
\toprule
Category & Count & Fraction \\
\midrule
lower\_alpha       & 18,757 & 0.286 \\
function\_word     & 18,359 & 0.280 \\
punctuation        & 10,306 & 0.157 \\
capitalized        & 8,714  & 0.133 \\
subword\_alpha     & 4,208  & 0.064 \\
numeric            & 2,451  & 0.037 \\
newline\_or\_space & 1,481  & 0.023 \\
eos                & 738    & 0.011 \\
all\_caps          & 334    & 0.005 \\
other              & 188    & 0.003 \\
\bottomrule
\end{tabular}
\end{table}

The largest single category is lower-case alphabetic tokens, with fraction $18{,}757/65{,}536\approx 0.286$. This gives a simple majority-category baseline for interpreting the weighted category purity. The observed weighted purity of $0.483$ is substantially above this baseline.

\subsection{Representative Component Specialization}

Table~\ref{tab:component_examples} shows representative routing channels with their hard-assignment frequency, dominant category, purity, and representative tokens. These examples illustrate that some channels are broad mixed routers, whereas others show clearer surface-form specialization.

\begin{table}[H]
\centering
\caption{Representative GMA routing channels under hard assignments $z_t=\arg\max_k \bar{\gamma}^Q_{t,k}$. Purity is the fraction of tokens assigned to the channel that belong to its dominant surface-form category.}
\label{tab:component_examples}
\scriptsize
\setlength{\tabcolsep}{3pt}
\renewcommand{\arraystretch}{0.95}
\begin{tabular}{ccclp{0.42\textwidth}}
\toprule
Comp. & Hard frac. & Purity & Dominant category & Representative high-count tokens \\
\midrule
22  & 0.104 & 0.320 & function\_word
& \texttt{.}, \texttt{,}, \texttt{and}, \texttt{the}, \texttt{a}, \texttt{The}, \texttt{his}, \texttt{He} \\

50  & 0.089 & 0.788 & lower\_alpha
& \texttt{city}, \texttt{yards}, \texttt{well}, \texttt{built}, \texttt{song}, \texttt{role}, \texttt{season} \\

55  & 0.077 & 0.414 & function\_word
& \texttt{.}, \texttt{,}, \texttt{of}, \texttt{in}, \texttt{to}, \texttt{the}, \texttt{by}, \texttt{from}, \texttt{on} \\

5   & 0.035 & 0.696 & function\_word
& \texttt{the}, \texttt{a}, \texttt{The}, \texttt{it}, \texttt{he}, \texttt{his}, \texttt{who}, \texttt{which} \\

17  & 0.014 & 0.764 & punctuation
& \texttt{=}, \texttt{@}, \texttt{(}, \texttt{-}, with occasional lexical tokens \\

73  & 0.032 & 0.516 & capitalized
& \texttt{M}, \texttt{Sh}, \texttt{McC}, \texttt{H}, \texttt{L}, \texttt{B}, \texttt{U}, \texttt{R} \\

116 & 0.023 & 0.621 & function\_word
& \texttt{to}, \texttt{not}, \texttt{he}, \texttt{were}, \texttt{are}, \texttt{I}, \texttt{be}, \texttt{was} \\

2   & 0.024 & 0.413 & subword\_alpha
& \texttt{ik}, \texttt{gam}, \texttt{mar}, \texttt{our}, \texttt{rom}, \texttt{in}, \texttt{ed} \\
\bottomrule
\end{tabular}
\end{table}

The component examples show a mixture of broad and specialized behavior. High-usage channels such as component 22 have low purity but absorb common mixed tokens. Other channels are more interpretable: component 50 is strongly associated with lower-case alphabetic words, components 5 and 116 are function-word-heavy, component 17 is punctuation-like, component 73 captures capitalized fragments, and component 2 captures subword alphabetic fragments. These examples should be interpreted as surface-form tendencies of head-averaged routing channels, not as clean syntactic or semantic categories.

\subsection{Permutation Baseline}

To check whether the observed component-category association could arise simply from imbalanced category frequencies, we compare against a permutation baseline. The category labels $\{c_t\}_{t=1}^T$ are randomly shuffled while keeping the learned hard component assignments $\{z_t\}_{t=1}^T$ fixed. This preserves both the marginal component distribution and the marginal category distribution, but destroys any systematic alignment between them.

\begin{table}[H]
\centering
\caption{Permutation sanity check for component-category alignment. The permutation baseline shuffles token categories while preserving hard component assignments and category frequencies.}
\label{tab:interpretability_permutation}
\footnotesize
\setlength{\tabcolsep}{5pt}
\renewcommand{\arraystretch}{0.95}
\begin{tabular}{lccc}
\toprule
Metric & Observed & Permutation mean $\pm$ std. & Interpretation \\
\midrule
Weighted purity
& 0.483 & $0.293 \pm 0.001$
& above shuffled baseline \\

$I(Z;C)$, nats
& 0.427 & $0.0067 \pm 0.0003$
& above shuffled baseline \\

Normalized MI
& 0.244 & $0.0038 \pm 0.0002$
& above shuffled baseline \\
\bottomrule
\end{tabular}
\end{table}

The permutation baseline confirms that the observed purity and mutual information are not explained by the marginal token-category frequencies alone. The learned component assignments contain substantially more surface-form category information than randomized assignments with the same marginals.

\subsection{Additional Visualizations}

Figure~\ref{fig:appendix_usage_purity} shows the marginal component usage and per-component category purity. The component usage distribution has normalized entropy $0.933$, indicating broad usage of the latent routing channels. The purity plot shows that several high-usage channels have a clear dominant surface-form category, although many remain mixed.

\begin{figure}[H]
    \centering
    \includegraphics[width=0.8\textwidth]{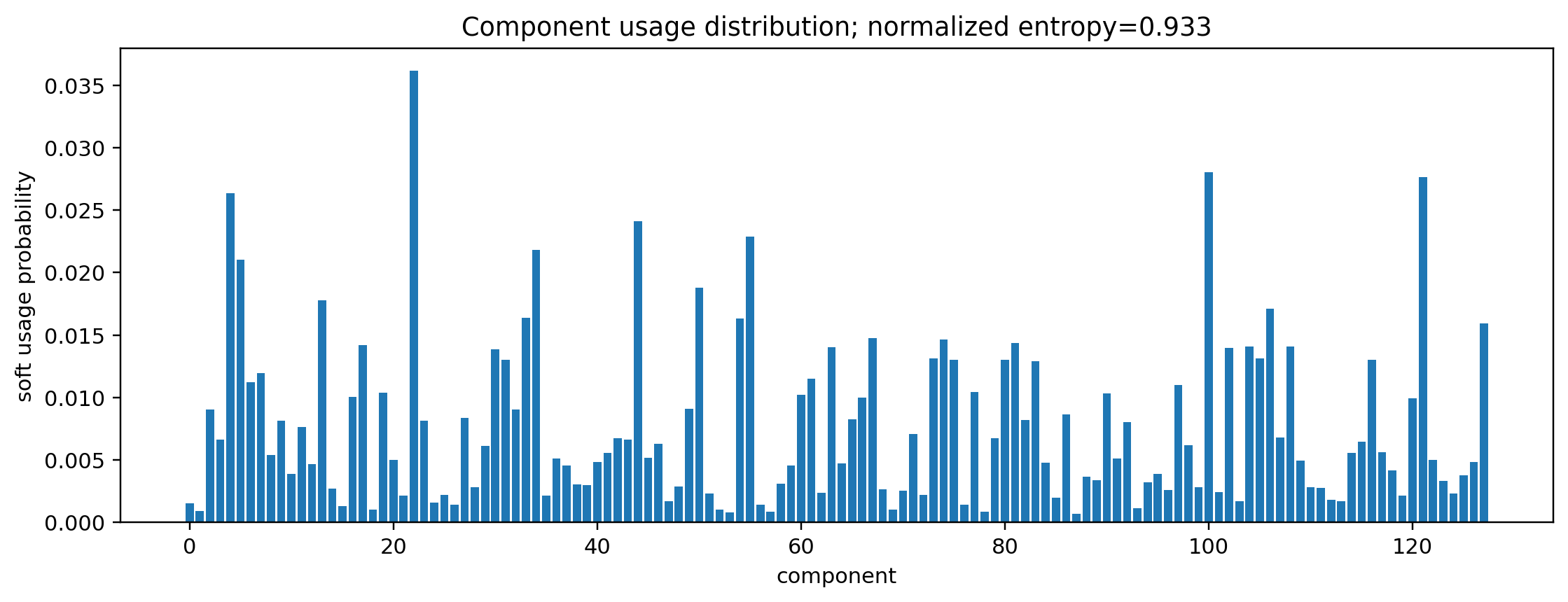}

    \vspace{0.5em}

    \includegraphics[width=0.8\textwidth]{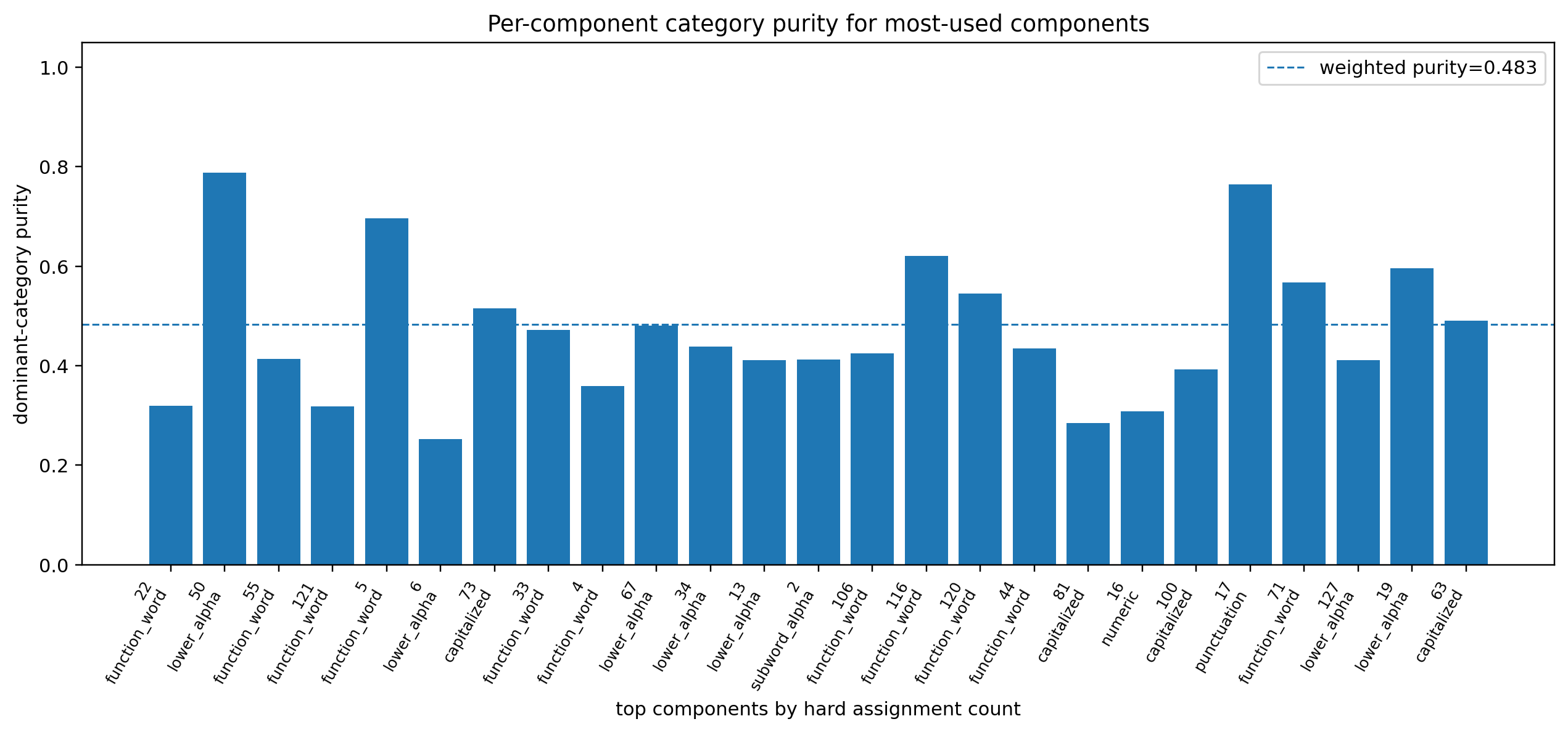}
    \caption{Additional component-level diagnostics. Top: marginal component usage distribution $p_k=T^{-1}\sum_{t=1}^T\bar{\gamma}^Q_{t,k}$, showing broad component usage and no severe collapse. Bottom: dominant-category purity for the most-used components.}
    \label{fig:appendix_usage_purity}
\end{figure}

Figure~\ref{fig:appendix_token_strip} gives a qualitative visualization of hard component assignments on one WikiText-103 validation segment. Tokens are colored by their dominant component. Repeated colors over related local token roles provide a visual indication of latent routing specialization.

\begin{figure}[H]
    \centering
    \includegraphics[width=1.0\textwidth]{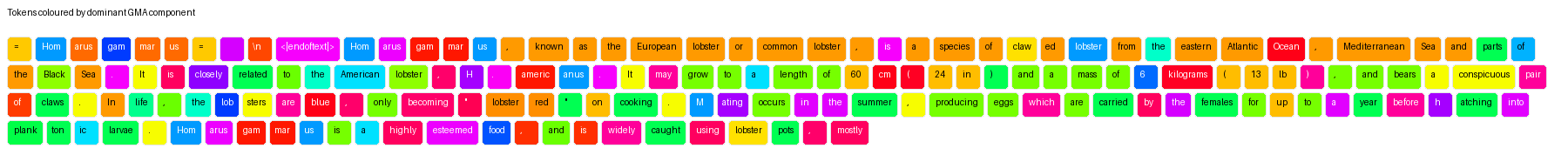}
    \caption{Qualitative token strip in which each token is colored by its dominant GMA component. This visualization illustrates local routing patterns across a validation sequence.}
    \label{fig:appendix_token_strip}
\end{figure}

Figure~\ref{fig:appendix_pca} projects token responsibility vectors to two dimensions using PCA. The first two principal components explain approximately $12.0\%$ and $7.9\%$ of the variance, respectively. The category-colored projection shows partial but overlapping separation between surface-form categories, while the component-colored projection shows that hard component assignments occupy structured regions of the responsibility space.

\begin{figure}[H]
    \centering
    \includegraphics[width=0.65\textwidth]{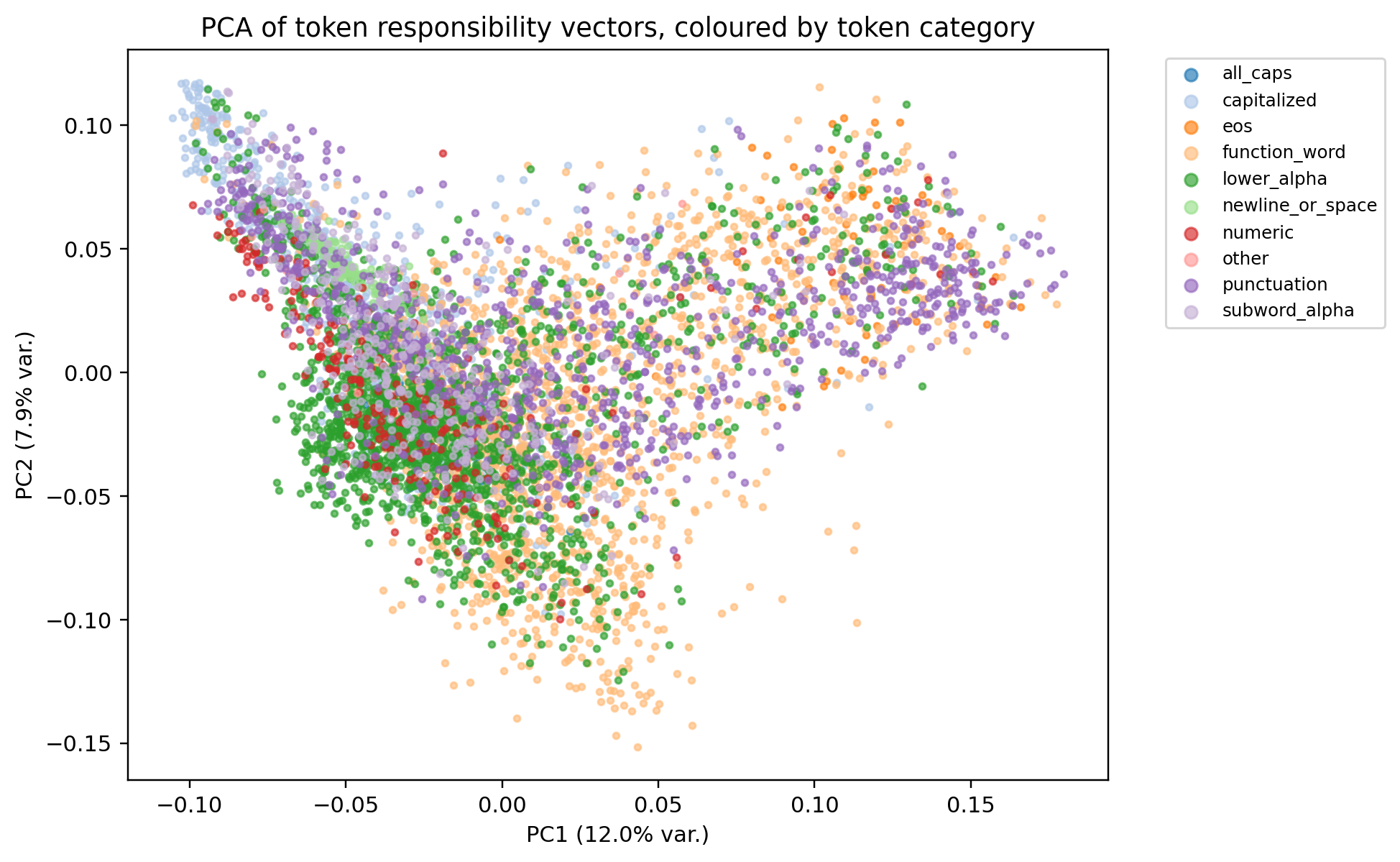}

    \vspace{0.5em}

    \includegraphics[width=0.55\textwidth]{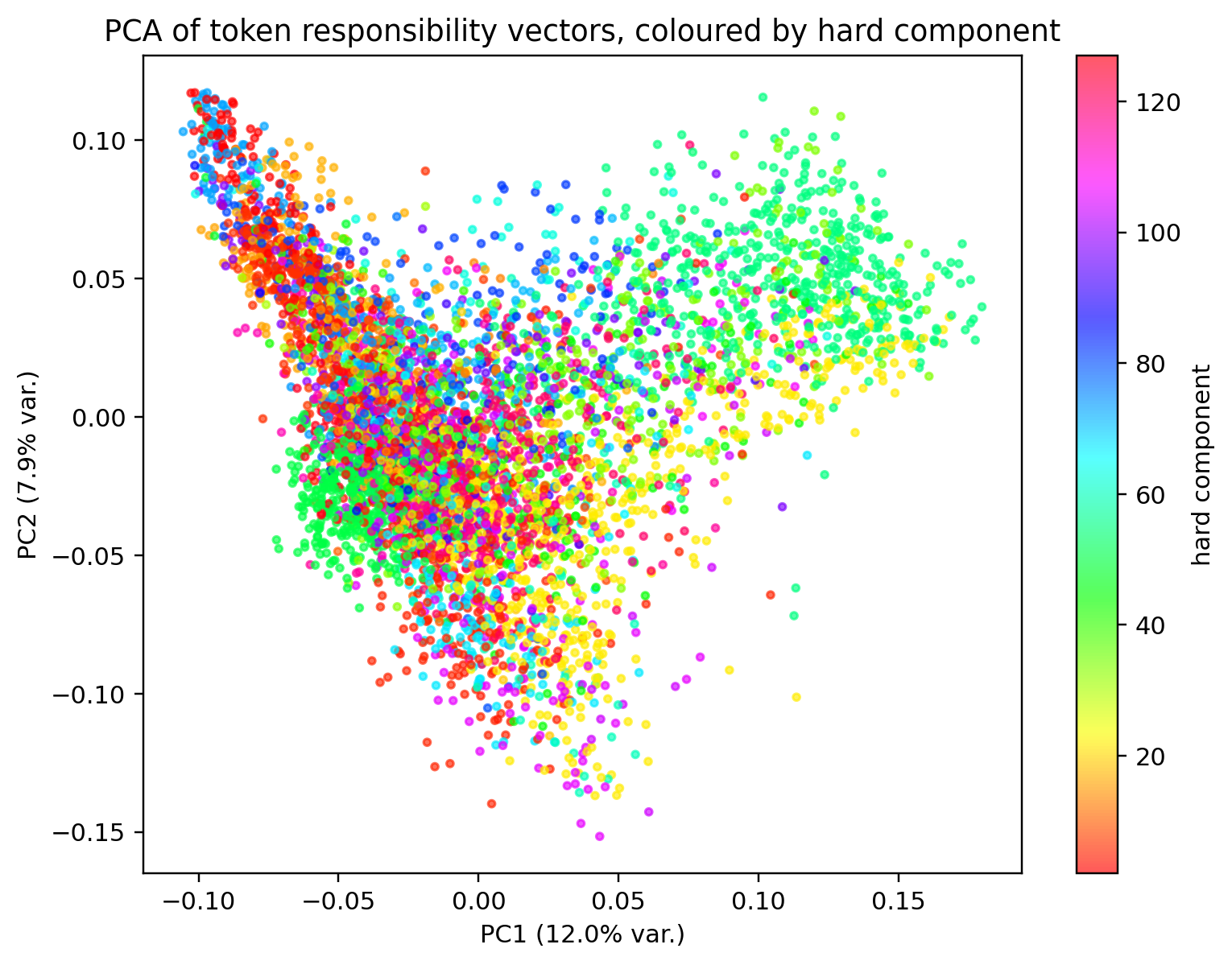}
    \caption{PCA projections of token responsibility vectors. Top: points colored by surface-form token category. Bottom: points colored by hard GMA component assignment. The projections show structured but overlapping responsibility geometry.}
    \label{fig:appendix_pca}
\end{figure}

\section{Attention, Similarity, Distance, Divergence, and Correlation}
\label{app:attention_similarity_distance}

This appendix provides a broader \textit{conceptual} perspective on attention mechanisms. At an abstract level, an attention layer assigns weights to value vectors by first computing a compatibility score between a query representation and a set of key representations. Different attention mechanisms can therefore be viewed as different choices of compatibility function. Dot-product attention uses an inner product \citep{vaswani2017attention}; additive attention uses a learned neural alignment/scoring function \citep{bahdanau2015neural}; kernelized attention uses feature-map similarities and associative matrix products \citep{katharopoulos2020transformers, choromanski2021rethinking}; sparse attention modifies which query-key pairs are compared \citep{child2019generating,beltagy2020longformer,zaheer2020bigbird}; and GMA replaces direct pairwise comparison with routing through latent Gaussian mixture responsibilities, as introduced in this work.

\paragraph{Similarity and distance.}
A similarity function assigns large values to objects that are considered compatible or related. A distance function assigns small values to objects that are close. These two notions are related, but not identical. A distance can be converted into a similarity through a monotone decreasing transformation, for example
\[
    \kappa(x,y)=\exp(-\beta d(x,y)),
    \qquad
    \kappa(x,y)=\frac{1}{1+d(x,y)},
    \qquad
    s(x,y)=-d(x,y),
\]
where $\beta>0$. However, a similarity need not be the reciprocal of a distance, and not every similarity corresponds to a valid metric.

A metric distance $d(x,y)$ satisfies non-negativity, identity of indiscernibles, symmetry, and the triangle inequality:
\[
    d(x,y)\geq 0,
    \qquad
    d(x,y)=0 \Leftrightarrow x=y,
    \qquad
    d(x,y)=d(y,x),
    \qquad
    d(x,z)\leq d(x,y)+d(y,z).
\]
Common pointwise distances include the $\ell_1$ distance, Euclidean distance, squared Euclidean distance, Mahalanobis distance, cosine distance, Hamming distance, and edit distance. In representation learning, these distances can be used directly as negative attention scores or indirectly through kernels such as the radial basis function (RBF) kernel
\[
    \kappa(x,y)
    =
    \exp\left(
    -\frac{\|x-y\|_2^2}{2\sigma^2}
    \right).
\]
The RBF kernel is a similarity derived from Euclidean distance: nearby points receive high similarity, while distant points receive exponentially small similarity.

\paragraph{Dot products, cosine similarity, and bilinear scores.}
Standard scaled dot-product attention uses
\[
    s(q_i,k_j)
    =
    \frac{q_i^\top k_j}{\sqrt{d_k}},
\]
where $q_i$ is a query vector and $k_j$ is a key vector \citep{vaswani2017attention}. This is a bilinear compatibility score: it is linear in $q_i$ when $k_j$ is fixed, and linear in $k_j$ when $q_i$ is fixed.
Cosine similarity normalizes the dot product by the vector norms,
\[
    \operatorname{cos}(q_i,k_j)
    =
    \frac{q_i^\top k_j}{\|q_i\|_2\|k_j\|_2},
\]
so that compatibility depends on direction rather than magnitude. More general bilinear attention scores take the form
\[
    s(q_i,k_j)
    =
    q_i^\top W k_j,
\]
where $W$ is a learned matrix. Dot-product attention is the special case $W=I$, up to projection matrices and scaling.

From this viewpoint, standard attention can be interpreted as a learned similarity-based retrieval mechanism. Given a query, the model asks which keys are most compatible with it, normalizes the resulting scores, and uses the normalized weights to aggregate values:
\[
    A_{ij}
    =
    \frac{\exp(s(q_i,k_j))}
    {\sum_{\ell=1}^N \exp(s(q_i,k_\ell))},
    \qquad
    O_i
    =
    \sum_{j=1}^N A_{ij}V_j.
\]
Thus, attention is not merely a similarity function; it is a normalized similarity-weighted aggregation mechanism.

\paragraph{Correlation and kernel similarity.}
Correlation is another form of relatedness. The Pearson correlation coefficient between scalar random variables $X$ and $Y$ is
\[
    \rho(X,Y)
    =
    \frac{\operatorname{Cov}(X,Y)}
    {\sqrt{\operatorname{Var}(X)\operatorname{Var}(Y)}}.
\]
It measures normalized \textit{linear} dependence. Cosine similarity can be viewed as a closely related geometric quantity: it is the normalized dot product between two vectors. If vectors are centered, cosine similarity and Pearson correlation are closely aligned \footnote{See e.g. Eq.(3.7) in \cite{stephen2018VMLS}.}.

Kernel functions generalize inner-product similarity by implicitly mapping inputs into a feature space:
\[
    \kappa(x,y)
    =
    \langle \phi(x),\phi(y)\rangle_{\mathcal{H}}.
\]
Examples include the linear kernel, polynomial kernel, RBF kernel, Laplacian kernel, Matérn kernels and string kernels. Kernelized attention mechanisms build on this idea by replacing the softmax attention kernel with feature-map approximations or positive kernel decompositions. Linear Transformer \cite{katharopoulos2020transformers} and Performer \cite{choromanski2021rethinking} are examples: they approximate or reformulate attention through feature maps so that the attention computation can be reassociated and evaluated in linear time \citep{katharopoulos2020transformers,choromanski2021rethinking}. In this sense, designing efficient attention can be viewed as designing a similarity function whose algebraic structure permits efficient aggregation.

\paragraph{Divergences between distributions.}
When the objects being compared are probability distributions rather than vectors, one often uses divergences or distributional distances. The Kullback-Leibler divergence is
\[
    D_{\mathrm{KL}}(p\|q)
    =
    \int p(x)\log\frac{p(x)}{q(x)}\,dx.
\]
It is not a metric because it is generally asymmetric and does not satisfy the triangle inequality. The Jensen-Shannon divergence symmetrizes and smooths KL:
\[
    D_{\mathrm{JS}}(p\|q)
    =
    \frac{1}{2}D_{\mathrm{KL}}(p\|m)
    +
    \frac{1}{2}D_{\mathrm{KL}}(q\|m),
    \qquad
    m=\frac{1}{2}(p+q).
\]
Other distributional comparisons include total variation distance, Hellinger distance, Wasserstein distance, Fisher-Rao distance, maximum mean discrepancy (MMD), and Bregman divergences. These quantities suggest possible distribution-aware attention mechanisms in which queries and keys are not treated as points, but as distributions with uncertainty. In such a setting, compatibility could be based on negative divergence, transport cost, kernel mean embedding distance, or overlap between probabilistic representations.

\paragraph{Attention as normalized compatibility.}
The common structure behind many attention mechanisms is
\[
    A_{ij}
    =
    \operatorname{Normalize}_j
    \left[
    \operatorname{Compat}(q_i,k_j)
    \right],
    \qquad
    O_i
    =
    \sum_j A_{ij}V_j.
\]
The compatibility function may be a dot product, a bilinear score, a neural network, a negative distance, a kernel value, a negative divergence, or a probabilistic responsibility. The normalization may be a softmax, sparsemax \cite{martins2016sparsemax}, entmax \cite{entmax}, top-$k$ normalization, kernel normalization, or another stochastic normalizer. This perspective separates two design choices: first, how compatibility is measured; second, how compatibility is converted into aggregation weights.

This also clarifies why the relationship between similarity and attention is subtle. A similarity score is usually symmetric if the same representation space and symmetric function are used. Attention, however, is generally asymmetric: queries and keys are produced by different projections, and the normalization is row-wise. Even if $s(q_i,k_j)=s(k_j,q_i)$ were symmetric, the normalized attention matrix need not be symmetric because
\[
    A_{ij}
    =
    \frac{\exp(s(q_i,k_j))}
    {\sum_{\ell}\exp(s(q_i,k_\ell))}
\]
normalizes separately for each query position $i$.

\paragraph{GMA as responsibility-based compatibility.}
GMA follows a different route from direct pairwise similarity. Instead of first forming all $N\times N$ query-key scores, it maps queries and keys into posterior responsibility vectors over $K$ learned Gaussian mixture components.
For a query vector $q_i$ and key vector $k_j$, the induced unnormalized affinity is (Eq.\ref{eq:gma_induced_attention_unnormalised})
\[
    \widetilde A^{\mathrm{GMA}}_{ij}
    =
    \sum_{k=1}^K
    \gamma^Q_{i,k}\gamma^K_{j,k}.
\]
Thus, \textit{GMA compares tokens through their distributions over latent routing components, rather than through a dot product in the projected routing representation space}. Equivalently, the affinity is an inner product in the probability simplex of component responsibilities. This makes GMA a responsibility-space similarity mechanism.

The computational philosophy is also different. Standard attention first forms pairwise query-key compatibilities and then aggregates values. GMA first uses key responsibilities to write values into a compact latent memory (Eq.\ref{eq:gma_latent_aggregation}),
\[
    \tilde V
    =
    (\Gamma^K)^\top V_X,
\]
and then uses query responsibilities to read from this memory (Eq.\ref{eq:gma_normalized_broadcasting}):
\[
    O
    =
    \frac{\Gamma^Q\tilde V}{\Gamma^QZ+\epsilon}.
\]
In this sense, GMA replaces direct pairwise similarity with a latent mixture routing mechanism. The learned Gaussian components define a multi-modal routing space, the key responsibilities construct component-indexed memory slots, and the query responsibilities determine how each token reads from those slots.

\paragraph{Implications for attention design.}
This perspective suggests that new attention mechanisms can be designed by choosing different notions of compatibility or routing geometry. Possible choices include:
\[
\begin{array}{lll}
\text{Dot-product attention} & q^\top k
& \text{inner-product similarity},\\[2mm]
\text{Cosine attention}
& \frac{q^\top k}{\|q\|\|k\|}
& \text{angular similarity},\\[3mm]
\text{Mahalanobis attention}
& -(q-k)^\top\Sigma^{-1}(q-k)
& \text{anisotropic distance},\\[2mm]
\text{RBF-kernel attention}
& \exp(-\|q-k\|^2/2\sigma^2)
& \text{distance-induced similarity},\\[2mm]
\text{Polynomial-kernel attention}
& (q^\top k+c)^p
& \text{higher-order feature similarity},\\[2mm]
\text{Divergence-based attention}
& -D(p_q\|p_k)
& \text{distributional compatibility},\\[2mm]
\text{Transport-based attention}
& -W_p(p_q,p_k)
& \text{geometry-aware distributional distance},\\[2mm]
\text{GMA}
& \sum_k\gamma^Q_{i,k}\gamma^K_{j,k}
& \text{responsibility-space affinity}.
\end{array}
\]
The main challenge is not only to define a meaningful compatibility score, but also to make it computationally tractable. Pairwise similarities usually require an $N \times N$ score matrix. Efficient attention mechanisms therefore need additional structure: sparsity, low rank, kernel feature maps, recurrence, state-space dynamics, or latent routing. GMA belongs to the latent-routing family: it replaces direct token-to-token comparison with interaction through a fixed number of learned probabilistic components.

Overall, attention can be understood as normalized compatibility-based aggregation. Dot-product attention is one important instance of this principle, but not the only one. Distances, similarities, correlations, kernels, divergences, and probabilistic responsibilities all provide possible foundations for attention design. GMA contributes to this broader view by showing how Gaussian-mixture responsibilities can define a linear-time, interpretable attention-style routing mechanism.

\end{document}